\documentclass[11pt]{article}

\usepackage{microtype}
\usepackage{graphicx}
\usepackage{subfigure}
\usepackage{booktabs}
\usepackage{csquotes}
\usepackage{textcomp}
\usepackage[numbers]{natbib}

\usepackage{amsthm,amsmath,amssymb,amsfonts,graphicx,url}
\usepackage{latexsym,nicefrac,bbm}
\usepackage{color}
\usepackage{hyperref}
\usepackage{algorithm}
\usepackage[noend]{algorithmic} 
\usepackage{caption}
\usepackage{pdfsync}
\usepackage{fullpage}
\usepackage{hyperref}

\def\showauthornotes{0}

\ifnum\showauthornotes=1

\else

\fi

\ifnum\showauthornotes=1
\newcommand{\Authornote}[2]{{\sf\small\color{red}{[#1: #2]}}}
\else
\newcommand{\Authornote}[2]{}
\fi

\newcommand{\eps}{\varepsilon}
\renewcommand{\epsilon}{\varepsilon}
\newcommand{\calL}{\mathcal{L}}

\newcommand{\ai}{i}
\newcommand{\aj}{j}

\usepackage{mhequ}
\def \be{\begin{equs}}
\def \ee{\end{equs}}

\newtheorem{theorem}{Theorem}[section]
\newtheorem{lemma}[theorem]{Lemma}
\newtheorem{definition}[theorem]{Definition}
\newtheorem{assumption}[theorem]{Assumption}

\newtheorem{corollary}[theorem]{Corollary}

\newtheorem{proposition}[theorem]{Proposition}

\newtheorem*{theorem*}{Theorem}

\begin{document}

\title{\bf A Convergent and Dimension-Independent Min-Max Optimization Algorithm}

\author{
  Vijay Keswani \\
  Yale University
  \and
  Oren Mangoubi \\
  WPI
  \and
  Sushant Sachdeva \\
  University of Toronto
  \and
  Nisheeth K. Vishnoi \\
  Yale University
}

\maketitle

\begin{abstract}
We study a variant of a recently introduced min-max optimization framework 
where the max-player is constrained to update its parameters in a greedy manner until it reaches a first-order stationary point. 
Our equilibrium definition for this framework depends on a proposal distribution which the min-player uses to choose directions in which to update its parameters.
We show that, given a smooth and bounded nonconvex-nonconcave objective function, access to any proposal distribution for the min-player’s updates, and stochastic gradient oracle for the max-player, our algorithm converges to the aforementioned approximate local equilibrium in a number of iterations that does not depend on the dimension. 
The equilibrium point found by our algorithm depends on the proposal distribution, and when 
applying our algorithm to train GANs we choose the proposal distribution to be a distribution of stochastic gradients.
We empirically evaluate our algorithm on challenging nonconvex-nonconcave test-functions and loss functions arising in GAN training. 
Our algorithm converges on these test functions and, when used to train GANs, trains stably on synthetic and real-world datasets and avoids mode collapse.

\end{abstract}

\clearpage
\tableofcontents
\clearpage

\newcommand{\gen}{G}
\newcommand{\disc}{D} \newcommand{\loss}{\ensuremath{\calL}}

\section{Introduction} \label{sec_intro}

For a loss function $f: \mathcal{X} \times \mathcal{Y} \rightarrow \mathbb{R}$ on some (convex) domain $\mathcal{X} \times \mathcal{Y} \subseteq \mathbb{R}^d \times \mathbb{R}^d$, we consider:
 \begin{equation} 
 \label{eq_global_minmax_intro}
 { \min_{x \in \mathcal{X}} \max_{y \in \mathcal{Y}} f(x,y)}. 
 \end{equation}
This min-max optimization problem has several applications to machine learning, including GANs \citep{Goodfellow2014generative} and adversarial training \citep{madry2017towards}.
In many of these applications, only first-order access to $f$ is available efficiently, and gradient-based algorithms are widely used.

Unfortunately, there is a lack of gradient-based algorithms with convergence guarantees for this min-max framework if one allows for loss functions $f(x,y)$ which are nonconvex and nonconcave in $x$ and $y$ respectively.
This lack of convergence guarantees can be a serious problem in practice, since popular algorithms such as gradient descent-ascent (GDA) oftentimes fail to converge, and GANs trained with these algorithms can suffer from issues such as cycling~\citep{Arjovsky2017principled} and ``mode collapse''~\citep{Dumoulin2017inference, Che2017mode, santurkar2018classification}.

Since min-max optimization includes minimization (and maximization) as special cases, it is intractable for general nonconvex-nonconcave functions.
Motivated by the success of a long line of results which show efficient convergence of minimization algorithms to various (approximate) local minimum notions (e.g., \cite{nesterov2006cubic, ge2015escaping, agarwal2017finding}), 
 previous works have sought to extend these ideas of local minima to various (approximate) notions of local min-max point--that is, a point $(x^\star, y^\star)$ where $x^\star$ is a local minimum of $f(\cdot, y^\star)$ and $y^\star$ is a local maximum of $f(x^\star, \cdot)$-- in the hope that this will allow for algorithms with convergence guarantees to such points.
 Unfortunately, to prove convergence, these works (e.g., \citet{nemirovski2004prox, lu2019hybrid}) make strong assumptions on $f$, e.g. assume $f(x,y)$ is concave in $y$, or that their algorithm is given a starting point such that its underlying dynamical system converges (\citet{heusel2017gans, mescheder2017numerics}). 
 It is a challenge to develop gradient-based algorithms which converge efficiently to an equilibrium for even a local variant of the min-max framework under less restrictive  assumptions comparable to those required for
 convergence of  algorithms to local minima.

\paragraph{Our contributions.}
We study a variant of the min-max framework which allows the max-player to update $y$ in a ``greedy" manner (Section \ref{sec_framework}).
This greedy restriction
models first-order maximization algorithms such as gradient ascent, popular in machine learning applications, which can make updates far from the current value of $y$ when run for multiple steps.
Roughly, from the current point $(x,y)$, our framework allows the max-player to update $y$ along {\em any} continuous path, along which the loss $f(x,\cdot)$ is nondecreasing.

         Our main contribution is a new gradient-based algorithm (Algorithm \ref{alg:Informal}) that provably converges from any initial point to an approximate local equilibrium for this framework (Definition \ref{def_greedy_minmax}).
         Our approximate equilibrium definition depends on the choice of proposal distribution $Q_{x,y}$, parametrized by $(x,y)$, which the min-player uses to update its parameters at any given point $(x,y)$.
         In particular, for a $b$-bounded function $f$ with ${L}$-Lipschitz gradient,
         and $\epsilon \geq  0$, our algorithm requires $\mathrm{poly}(b,L, 1/\epsilon)$ gradient and function oracle calls to converge to an $(\epsilon, Q)$-approximate local equilibrium (Theorem \ref{thm:GreedyMinimax-main}).
         The number of oracle calls required by our algorithm is independent of the dimension $d$.
         Gradient-based algorithms that converge in (almost) dimension independent number of iterations are required for machine-learning applications where the dimension $d,$ equal to the number of trainable parameters, can be very large.

   The equilibrium point our algorithm converges to depends on the choice of proposal distribution $Q$.
   In the special case when $f(x,y)$ is strongly convex in $x$ and strongly concave in $y$, there is a choice of proposal distribution $Q$-- the (deterministic or stochastic) gradients with mean $-\nabla_x f$--for which our $(\eps, Q)$-equilibrium corresponds to  an (approximate) global min-max point with duality gap roughly $O(\eps)$ (Theorem \ref{thm_strongly_convex} in the appendix).
   Our algorithm can find such a point in time that is polynomial in $\frac{1}{\eps}$ and independent of dimension (Corollary \ref{cor_runtime_strongly_convex}).
   This motivates using (stochastic) gradients for proposal distributions in more general settings as well, and, when training GANs we choose the proposal distribution to be the distribution of stochastic gradients.

           Empirically, we show that our algorithm converges on   test functions \citep{wang2019solving} on which other popular gradient-based min-max optimization algorithms such as  GDA and optimistic mirror descent (OMD) \citep{Daskalakis2018optimism} are known to either diverge or cycle (Figure \ref{fig_intro}, see also Figure \ref{figure_toy_function} in Section \ref{sec_test_functions}).
         We also show that a practical variant of our algorithm can be employed 
          for training GANs, with a per-step complexity and memory requirement similar to GDA.
        We observe that our algorithm consistently learns a greater number of modes than GDA, OMD, and unrolled GANs (Table \ref{tbl:gaussian_results}),
        when applied to training GANs on a  Gaussian mixture dataset.
        While not the focus of this paper, we also provide results for our algorithm on the real-world datasets in the Supplementary Material.

         \begin{figure}
         \centering
        \includegraphics[width=0.5\linewidth]{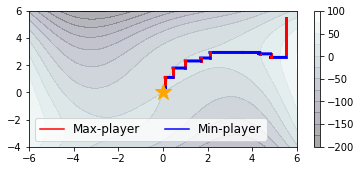}
         \caption{\small Our algorithm applied to the function $f(x,y) = (4x^2 -(y-3x + 0.05x^3)^2-0.1y^4)e^{-0.01(x^2+y^2)}$ with global min-max point $(x,y) = (0,0)$ (yellow star). 
         Our algorithm's max-player (red segments) first finds a point where $f(x,\cdot)$ is maximized. 
         The min-player then proposes random updates to $x$, and only accepts those updates which lead to a decrease in the value of $f(x,y)$ after the max-player's response is taken into account (blue segments). 
         This allows our algorithm to converge to $(0,0)$.
         This function is considered as a challenging test function in \cite{wang2019solving}, who show several first-order algorithms, namely GDA, OMD, and extra-gradient method \citep{extra_gradient}, fail to converge on this function and instead cycle forever (see Figure 2).
         }
         \label{fig_intro}
        \end{figure}

\paragraph{Discussion of equilibrium.}
The equilibrium points $(x^\star, y^\star)$  our algorithm converges to can be viewed as local equilibria for a game where the maximizing player is restricted to making greedy updates to the value of $y$.
    Namely, the point $x^\star$ is an approximate local minimum of an alternative to the function $\max_y f(\cdot, y)$ where, rather than maximizing over all $y \in \mathbb{R}^d$,  the maximum is instead taken over all ``greedy" paths--i.e., paths along which $f(x^\star, \cdot)$ is increasing--initialized at the value $y^\star$.
    Additionally, $y^\star$ is an approximate local maximum of $f(x^\star, \cdot)$.
In particular, we show that any point $(x^\star, y^\star)$ which is a local min-max point 
is also an equilibrium point for our algorithm (see Section \ref{sec_framework}).

\paragraph{Discussion of assumptions.}
For our main result, 
we assume $f$ is bounded above and below. 
The assumption that the loss function is bounded below is standard in the minimization literature (see e.g., \cite{nesterov2006cubic}), as an unbounded function need not achieve its minimum and a minimization algorithm could diverge in a manner such that the loss function value tends to $-\infty$.
Thus, in min-max optimization, both the upper and lower bound assumptions are necessary to ensure the existence of even an 
(approximate) global min-max point.
 If we drop the lower bound assumption, and only assume $f(x,\cdot)$ is bounded above for $x \in \mathbb{R}^d,$ 
our algorithm still does not cycle: instead it either converges to a local equilibrium $(x^\star,y^\star)$, or the value of $f$  diverges monotonically to $-\infty$.
 Such functions include  popular losses, e.g. cross entropy \citep{Goodfellow2014generative} which is bounded above by zero, making our algorithm applicable to training GANs.

  \section{Related Work}
    \paragraph{Local frameworks.} 
   In addition to the local min-max framework (e.g., \cite{nemirovski2004prox, lu2019hybrid, heusel2017gans, mescheder2017numerics}), previous works propose local frameworks where the max-player is able to choose its move after the min-player. 
   These include the local stackleberg equilibrium
   \citep{fiez2020implicit} and the closely related local minimax point 
   \citep{Jin2019minmax}.
     In the local min-max, local stackleberg and local minimax frameworks, the max-player is restricted to move in a small ball around the current point $y$. 
     In contrast, in our framework, the max-player can move  much farther, as long as it follows a  path along which $f$ is continuously increasing.
  
  \paragraph{Convergence guarantees.}
 Several works have studied the convergence properties of GDA dynamics~\citep{nagarajan2017gradient, mescheder2017numerics,  Balduzzi2018mechanics, Daskalakis2018limit, Jin2019minmax, li2018limitations}, 
 and established that GDA suffers from severe limitations: GDA can exhibit rotation around some points, or otherwise fail to converge.
To address convergence issues for GDA, multiple works analyze algorithms based on Optimistic Mirror Descent (OMD), Extra-gradient (EG) methods, or similar
approaches \citep{Gidel2019momentum, Daskalakis2018limit, liang2018interaction, Daskalakis2019ITCS,  mokhtari2019proximal}. 
 For instance, \cite{Daskalakis2018optimism} guarantee convergence of OMD to a global min-max point on bilinear losses $f(x,y) = x^\top A y,$ and 
\cite{Mokhtari2019extra} also show convergence of OMD and EG methods on strongly convex-strongly concave $f$.
However, 
as observed in \cite{wang2019solving}, GDA, OMD, and EG fail to converge on some simple nonconvex-nonconcave test functions; in comparison,
we observe that 
our algorithm converges for these functions (Figure~\ref{figure_toy_function}).
Many works make additional assumptions to prove convergence--\cite{mertikopoulos2019optimistic} show asymptotic convergence of OMD under a ``coherence''  assumption,
and  \cite{Balduzzi2018mechanics} show convergence of a second-order algorithm if $f$ corresponds to a Hamiltonian game.
 Some works assume there exists a ``variational inequality" solution $(x^\star, y^\star)$ such that, roughly, the component of the gradient field $(-\nabla_x f, \nabla_y f)$ in the direction away from $(x^\star,y^\star)$ is very small \cite{dang2015convergence, liu2019towards, song2020optimistic, diakonikolas2021efficient, liu2021first}.
Other works \cite{yang2020global} assume a ``PL" condition which 
says that magnitude of $\nabla_x f(x,y)$ is at least $f(x,y) - \min_x f(x,y)$.
 However, many simple functions, e.g. $f(x,y) = \sin(x)\sin(y)$, do not satisfy Hamiltonian game, coherence, variational, and PL
 assumptions.

For this reason, multiple works show convergence to an (approximate) local min-max point.
For instance \cite{heusel2017gans} prove convergence of finite-step GDA to a local min-max point, under the assumption that their algorithm is initialized such that the underlying continuous dynamics converge to a local min-max point.
And \cite{mescheder2017numerics} show convergence if their algorithm is initialized in a small neighborhood of a local min-max point.
In addition, many works provide convergence guarantees to a local {\em stackleberg} or local {\em minimax} point, if their algorithm is provided with a starting point in the region of attraction \citep{fiez2019convergence, fiez2020implicit}, or a small enough neighborhood \citep{wang2019solving}, of such an equilibrium.
And other works 
\citep{nemirovski1978cesari, kinderlehrer1980introduction, nemirovski2004prox,rafique2018non, lu2019hybrid, lin2019gradient, nouiehed2019nonconvex,thekumparampil2019efficient,kong2019accelerated}
 show convergence to an approximate local min-max point when $f$ may be nonconvex in $x$, but is concave in $y$.
In contrast to the above works, our algorithm is guaranteed to converge for any nonconvex-nonconcave $f$, from any starting point, in a number of gradient evaluations that is independent of the dimension $d$ and polynomial in $L$ and $b$ if $f$ is $b$-bounded with ${L}$-Lipschitz gradient.
 Such smoothness/Lipschitz bounds are standard in convergence guarantees for optimization algorithms~\citep{bubeck2014convex,ge2015escaping,vishnoi2021convex}.
Similar to our approach, some algorithms make multiple $y$-player updates each iteration (e.g., \cite{nouiehed2019nonconvex}). However, 
these algorithms are not guaranteed to converge from any initial point on any nonconvex-noncocave smooth bounded $f$;
to overcome this, our algorithm introduces a randomized accept-reject procedure.

\paragraph{Greedy paths.}
\cite{our_theory_paper} also consider a framework where the max-player makes updates in a greedy manner.
The ``greedy paths" considered in their work are defined such that at every point along these paths, $f$ is non-decreasing, and the first derivative of $f$ is at least $\eps$ or the 2nd derivative is at least  $\sqrt{\eps}.$ 
In contrast, we just require a condition on the first derivative of $f$ along the path.
This distinction gives rise to a different framework and equilibrium than the one presented in their work. 
Secondly, 
\cite{our_theory_paper} is a second-order method 
that converges to an $\epsilon$-approximate local equilibrium in $\mathrm{poly}(d, b,L, 1 / \epsilon)$ Hessian evaluations. On the other hand, the convergence of our algorithm is independent of $d$;
it requires $\mathrm{poly}(b,L, 1 / \epsilon)$ gradient evaluations for convergence.

 \paragraph{Training GANs.}
An important line of work focuses on designing min-max optimization algorithms that mitigate non-convergence behavior such as cycling when training GANs using GDA \citep{Goodfellow2014generative}.
\cite{Daskalakis2018optimism} show OMD can mitigate cycling when training GANs with Wasserstein loss. 
In contrast to both GDA and OMD, where at each iteration the min- and max-players are allowed only to make small updates roughly proportional to their respective gradients, our algorithm empowers the max-player to make large updates at each iteration.
\cite{Metz2017unrolled} introduced Unrolled GANs, where the min-player optimizes an ``unrolled'' loss that allows the min-player to simulate a fixed number of max-player updates.
While this has some similarity to our algorithm
the main distinction is that the min-player in Unrolled GANs may not reach an (approximate) local minimum, and hence their algorithm does not
have any convergence guarantees.
We observe that our algorithm, applied to training GANs, trains stably and avoids mode collapse.

\section{Theoretical Results} \label{sec_theoretical_results}

As a first step towards obtaining a computationally tractable variant of the min-max framework, we consider the local min-max point studied in prior work— that is, any point $(x,y)$ such that $x$ is local minimum of $f(\cdot, y)$ and $y$ is a local maximum of $f(x,\cdot)$.
 Unfortunately, local min-max points may not exist even on smooth and bounded functions.
For instance, the function $f(x,y) = \sin(x+y)$ has no local min-max points.
This is because, while $f(x,y) = \sin(x+y)$ has local maximum in $y$ at all points along the collection of lines $S = \{(x,y): x+y = \frac{\pi}{2} + 2\pi k, k \in \mathbb{N}\}$, $\sin(x+y)$ does not have  a local minimum in $x$ at any of these points.
However, the points $S$ are all global min-max points of $f(x,y) = \sin(x + y)$, since for every $(x,y) \in S$, $x$ is a global minimum of $\max_{y\in \mathbb{R}^d} f(\cdot,y)$, and $y$ is a global maximum of $f(x,\cdot)$.
This is true even though $x$ is neither a global nor a local minimum of $f(\cdot, y)$.

On the other hand, an (approximate) {\em global} min-max point is always guaranteed to exist for smooth and bounded $f$.
This is because, in the global min-max framework, before the min-player considers whether to choose a value $x$, it is able to ``look ahead” and anticipate the response $\arg\max_{y\in \mathbb{R}^d} f(x,y)$ of the max-player.
Thus, for any smooth and bounded function $f$, one can always find an (approximate) global min-max point by first finding an (approximate) global minimum $x$ of the function $\max_{y\in \mathbb{R}^d} f(\cdot,y)$, and then finding a value of $y$ which maximizes $f(x,\cdot)$.
In order to guarantee existence for our framework, we would therefore ideally like 
to allow the min-player to anticipate the max-player’s response $\max_{y\in \mathbb{R}^d} f(\cdot,y)$ to any value of $x$ proposed by the min-player.
Unfortunately, computing the global maximum $\max_{y\in \mathbb{R}^d} f(\cdot,y)$ is intractable.

\subsection{Framework and Equilibrium} \label{sec_framework}
To get around this problem, we consider a variant of the min-max framework, which empowers the max-player to update $y$ in a ``greedy" manner.
More specifically, we restrict the max-player to update the current point $(x,y)$ to any point in a set $P(x,y)$ consisting of the endpoints of  paths in $\mathcal{Y}$ initialized at $y$ along which $f(x, \cdot)$ is nondecreasing.
These paths model the paths taken by a class of first-order algorithms, which includes popular algorithms such as gradient ascent.
 Our framework therefore allows the min-player to learn from max-players which are computationally tractable and yet (in contrast to the local min-max framework) are still empowered to make updates to the value of 
 $y$ which may lead to large increases in $f(x,y)$.
Given a bounded loss $f: \mathcal{X} \times \mathcal{Y} \rightarrow \mathbb{R}$, where $\mathcal{X}, \mathcal{Y} \subseteq \mathbb{R}^d$ are convex, an equilibrium for our framework is a point $(x^\star, y^\star) \in \mathcal{X} \times \mathcal{Y}$ such that
        \begin{align}
         x^\star &\in \mathrm{argmin}_{x \in \mathcal{X}} (\mathrm{max}_{y \in P(x,y^\ast)} f(x,y) ), \label{eq_our_frameowrk_x}\\
         y^\star &\in \mathrm{argmax}_{y \in P(x^\star, y^\star)} f(x^\star,y). \label{eq_our_frameowrk_y}
        \end{align}
 This is in contrast to the (global) min-max framework of \eqref{eq_global_minmax_intro} where the maximum is taken over all $y \in \mathcal{Y}$.
However, solutions to  \eqref{eq_our_frameowrk_x} and  \eqref{eq_our_frameowrk_y} may not exist, and even when they do exist, finding such a solution is  intractable since  \eqref{eq_our_frameowrk_x} generalizes nonconvex minimization. 

\noindent {\em Local equilibrium.}
Replacing the global minimum in \eqref{eq_our_frameowrk_x} with a local minimum leads to the following local version of our framework's equilibrium.
A point $(x^\star, y^\star) \in \mathcal{X} \times \mathcal{Y}$ is a local equilibrium if, for some $\nu > 0$ (and denoting the ball of radius $\nu$ at $x^\star$ by  $B(x^\star,\nu)$),
        \begin{align}
         x^\star &\in \mathrm{argmin}_{x\in B(x^\star, \nu) \cap \mathcal{X}}  (\mathrm{max}_{y \in P(x,y^\ast)} f(x,y) ), \label{eq_our_local_frameowrk_x}\\
        y^\star &\in \mathrm{argmax}_{y \in P(x^\star, y^\star)} f(x^\star,y) \label{eq_our_local_frameowrk_y},
                  \end{align}

\noindent {\em Approximate local equilibrium.}
Similar to previous work on local minimization for smooth  non-convex objectives, we would like to solve \eqref{eq_our_local_frameowrk_x} and \eqref{eq_our_local_frameowrk_y} to converge to approximate 
 stationary points \citep{nesterov2006cubic}. 
      Towards this end, 
      we can replace $P(x, y^\star)$ in \eqref{eq_our_local_frameowrk_x} and  \eqref{eq_our_local_frameowrk_y} with the set $P_\epsilon(x, y^\star)$ of endpoints of paths starting at $y^\star$ along which $f(x, \cdot)$ increases at some ``rate"  $\epsilon > 0$.

\begin{definition} \label{def_path}
  For any $x \in \mathcal{X}$, $y \in \mathcal{Y},$ and $\epsilon \geq  0,$ define
  $P_{\epsilon}(x, y){\subseteq}\mathcal{Y}$ to be points $w \in \mathcal{Y}$
  s.t. there is a continuous and (except at finitely many points)
  differentiable path $\gamma(t)$ starting at $y,$ ending at $w,$ and unit-speed, i.e.,
  $ \left\| \frac{\mathrm{d}}{\mathrm{d}t} \gamma(t)\right\| \leq  1$, s.t. at any point on $\gamma$,   $$\frac{\mathrm{d}}{\mathrm{d}t} f(x, \gamma(t))\geq\epsilon.$$
\end{definition}
\noindent
The above definition 
restricts the max-player to updating $y$ via any ``greedy" algorithm, e.g. gradient ascent.
Note that, compared to Definition \ref{def_path}, the notion of greedy paths in  \cite{our_theory_paper} additionally requires $$\frac{\mathrm{d}^2}{\mathrm{d}t^2} f(x, \gamma(t)) \geq \sqrt{\epsilon}$$ so as to achieve the goal of converging to an \textit{approximate second-order local equilibrium}.
Our goal, on the other hand, is for the max-player's updates to approximate paths taken by first-order greedy algorithms, hence, the condition on first derivative in Definition~\ref{def_path} suffices.

While we would also like to replace the local minimum in \eqref{eq_our_local_frameowrk_x} with an approximate stationary point, the {\em min-player's objective}, $\mathcal{L}_{\eps}(x,y) := \max_{z \in P_\eps(x,y)} f(x,z)$, may not be
continuous\footnote{Consider the example $f(x,y)=\min(x^2y^2,1).$ The
  min-player's objective for $\eps>0$ is
  $\mathcal{L}_\eps(x,y) = f(x,y)$ if $2x^2y < \eps,$ and 1
  otherwise. Thus $\mathcal{L}_{1/2}$ is discontinuous at
  $(1/2,1).$} in $x,$ and thus, gradient-based notions of
approximate local minimum do not apply.
To bypass this difficulty and to define a notion of approximate local minimum which applies to discontinuous functions, we sample updates to $x,$ and test whether $\mathcal{L}_{\eps}(\cdot, y)$ has decreased. 
Formally, given a choice of sampling distribution ${Q}_x$ (which may depend on $x$), and $\delta, \omega >0$,  $x^\star$ is said to be an approximate local minimum of a (possibly discontinuous) function $g: \mathcal{X} \rightarrow \mathbb{R} \,\,$ if 
        $\,\, \Pr_{\Delta \sim {Q}_{x^\star}} \left[ g(x^\star +   \Delta) < g(x^\star) - \delta  \right] < \omega$.

Thus, replacing the set $P$ with $P_\eps$  in \eqref{eq_our_local_frameowrk_x} and \eqref{eq_our_local_frameowrk_y}, and the ``exact" local minimum in \eqref{eq_our_local_frameowrk_x} with an approximate local minimum,  we arrive at our equilibrium definition:
\begin{definition} \label{def_greedy_minmax}
 Given $\eps, \delta, \omega>0$ and a distribution ${Q}_{x,y}$, we say that a point $(x^\star, y^\star)\in \mathcal{X} \times \mathcal{Y}$ is an $(\eps, \delta, \omega, Q)$-approximate local equilibrium for our framework if
  \begin{gather} \label{eq_approx_local_equilibrium_x}
{     
\begin{split}
\Pr\limits_{\substack{\Delta \sim {Q}_{x^\star, y^\star}}} 
    \bigg[\max_{y \in P_\eps(x^\star + \Delta,y^\star)} f(x^\star + \Delta,y) \quad < \max\limits_{\substack{y \in P_{\eps}(x^\star,y^\star)}} f(x^\star,y) - \delta \bigg] \leq \omega,
    \end{split}    
    } 
    \\
   \label{eq_approx_local_equilibrium_y}
      {   y^\star \in \mathrm{argmax}_{y \in P_{\eps}(x^\star, y^\star)} f(x^\star,y).
        }
                    \end{gather} 
\end{definition}

\noindent

\paragraph{Proposal distribution.}  
In the special case when $f$ is, e.g., $O(1)$-strongly convex in $x$ and $O(1)$-strongly concave in $y$ with  $O(1)$-Lipschitz gradients, if one chooses the updates $Q$ to be the (deterministic or stochastic) gradients $-\nabla_x f$, then the $(\epsilon, \delta, \omega, Q)$-equilibrium corresponds to an ``approximate" global min-max point for $f$ with duality gap $O(\epsilon + \delta)$ (Theorem \ref{thm_strongly_convex}).
This duality gap does not depend on the dimension.
 This motivates choosing the proposal distribution to be the (stochastic) gradients $-\nabla_x f$ in the more general setting when $f$ is nonconvex-nonconcave. 
Another motivation for this choice of $Q$ is that adding stochastic gradient noise to steps taken by deep learning algorithms is known empirically to lead to better 
 outcomes than, e.g., standard Gaussian noise (see \cite{zhu2019anisotropic}).
Empirically, 
this choice of  ${Q}$ leads to GANs that are able to successfully learn the dataset's distribution (Section \ref{sec:experiments}).

\setlength{\intextsep}{7pt}%
\begin{algorithm}
\small
\caption{Our algorithm for min-max optimization}
\textbf{input:} Stochastic zeroth-order oracle $F$ for bounded  
loss function $f: \mathbb{R}^d \times \mathbb{R}^d \rightarrow \mathbb{R}$ with $L$-Lipschitz gradient, 
stochastic gradient oracle $G_{y}$ with mean
$\nabla_{y} f$,  Initial point $(x_0, y_0)$

\textbf{input:} A distribution ${Q}_{x,y}$, and an oracle for sampling from this distribution. Error parameters $\epsilon, \delta >0$
\textbf{hyperparameters:} $\eta>0$ (learning rate), $r_{\mathrm{max}}$ (maximum number of rejections); $\tau_1$ (for annealing);
 
 \begin{algorithmic}[1]
 \STATE  Set $i \leftarrow 0$, $r \leftarrow 0$, $\epsilon_0 = \frac{\epsilon}{2}$, $f_{\mathrm{old}} \leftarrow \infty$
\WHILE{$r \leq r_{\mathrm{max}}$} \label{InnerWhileStart}
\STATE Sample $\Delta_{\ai}$ from the distribution $Q_{x_i, y_i}$ 
\STATE Set $X_{\ai +1} \leftarrow x_{\ai} + \Delta_i$ \COMMENT{{\itshape \footnotesize min-player's proposed update}}
\STATE Run Algorithm \ref{alg:InnerMaxLoop} with inputs $\mathsf{x} \leftarrow X_{\ai+1}$,  $\mathsf{y}_0 \leftarrow y_{\ai}$, and $\epsilon' \leftarrow \epsilon_{\ai}\times (1- 2\eta L)^{-1}$ \COMMENT{\emph{\footnotesize max-player's update}} 
\STATE Set $\mathcal{Y}_{\ai+1} \leftarrow \mathsf{y}_{\mathrm{stationary}}$ to be the output of Algo \ref{alg:InnerMaxLoop}.           
\STATE Set $f_{\mathrm{new}} \leftarrow F(X_{\ai+1}, \mathcal{Y}_{\ai+1})$ \COMMENT{\emph{\footnotesize Compute the new loss}}
\STATE Set $\mathsf{Accept}_i \leftarrow \mathsf{True}.$ 
\IF{$f_{\mathrm{new}} > f_{\mathrm{old}} - \frac{\delta}{4}$,}
\STATE Set $\mathsf{Accept}_i \leftarrow \mathsf{False}$ with probability $\max(0,1-e^{-\frac{i}{\tau_1}})$ \label{accept_reject_step}
\COMMENT{\emph{\footnotesize Decide to accept or reject}} \label{AcceptRejectStep}
\ENDIF
\IF{$\mathsf{Accept}_{\ai} = \mathsf{True}$}
\STATE Set $x_{\ai+1} \leftarrow X_{\ai+1}$, $y_{\ai+1} \leftarrow \mathcal{Y}_{\ai+1}$  \COMMENT{\emph{\footnotesize accept the proposed $x$ and $y$ updates}}
\STATE Set $f_{\mathrm{old}} \leftarrow f_{\mathrm{new}}$, $r \leftarrow 0$, \,\,\, $\epsilon_{\ai+1} \leftarrow \epsilon_{\ai}\times (1- 2\eta L)^{-2}$
\ELSE
\STATE Set $x_{\ai+1} \leftarrow x_{\ai}$, $y_{\ai+1} \leftarrow y_{\ai}$, \, \, $r \leftarrow r + 1$,\,
  $\epsilon_{\ai+1} \leftarrow \epsilon_{\ai}$  \COMMENT{\emph{\footnotesize Reject  the proposed updates}}%
\ENDIF
\STATE Set $\ai \leftarrow \ai+1$
\ENDWHILE
\STATE {\bf return}  $(x^\star, y^\star) \leftarrow (x_{\ai}, y_{\ai})$
\end{algorithmic}
\label{alg:Informal}
\end{algorithm}

\paragraph{Comparison to local min-max points. } 
Note that any local min-max point, if it exists,
satisfies Definition~\ref{def_greedy_minmax} for a proposal distribution $Q$ with small enough mean and variance.
This is because if $(x^\star, y^\star)$ is a local min-max point of $f$, $x^\star$ is a local minimum of  $f(\cdot, y^\star)$ and hence there is a ball $B$ containing  $x^\star$   on which $x^\star$ minimizes $f(\cdot, y^\star)$.
Moreover, $y^\star$ is a first-order stationary point of $f(x^\star, \cdot)$, which means that  $P_{\eps}(x^\star,y^\star) = \{y^\star\}$ and hence $$\max_{y \in P_{\eps}(x^\star,y^\star)} f(x^\star,y) = f(x^\star, y^\star),$$ satisfying \eqref{eq_approx_local_equilibrium_y}.
Therefore, if $Q$ has mean and variance small enough that the min-player's proposed updates $\Delta \sim Q_{x^\star, y^\star}$ fall inside $B$ 
w.h.p.,
we will have that  $$\max_{y \in P_\eps(x^\star + \Delta,y^\star)} f(x^\star + \Delta,y)>f(x^\star, y^\star)$$
since $y^\star \in P_\eps(x^\star + \Delta,y^\star)$, implying that $(x^\star, y^\star)$ satisfies \eqref{eq_approx_local_equilibrium_x}
(proof provided in Appendix~\ref{sec_local}). 

However, the converse is not true.
This is a necessary feature of our definition 
as there are simple smooth bounded functions which do not have any local min-max points and yet an equilibrium from Definition~\ref{def_greedy_minmax} is guaranteed to exist.
For instance, as mentioned earlier, $\sin(x+y)$ does not have any local min-max points; however, the 
global min-max points 
 $S = \{(x,y): x + y = \frac{\pi}{2} + 2\pi k, k{\in}\mathbb{N}\}$ 
 of $\sin(x+y)$ 
satisfy Definition~\ref{def_greedy_minmax} for any $\epsilon > 0$, $\delta = \Omega(\sqrt{\epsilon})$, $\omega = 0$, and, e.g., any proposal distribution $Q$ with support on a ball of radius $\frac{1}{2}$ centered at $0$.
 (See Appendix \ref{sec:examples} for examples)

\subsection{Algorithm}
We present an algorithm for our framework (Algorithm \ref{alg:Informal}), along with the gradient ascent subroutine it uses to compute max-player updates (Algorithm \ref{alg:InnerMaxLoop}).
In 
Theorem~\ref{thm:GreedyMinimax-main},
we show it 
efficiently finds an approximate local equilibrium (Definition~\ref{def_greedy_minmax}).
We consider bounded loss functions
$f: \mathbb{R}^d \times \mathbb{R}^d \rightarrow \mathbb{R}$, where $f$ is an empirical risk loss over $m$ training
examples, i.e., $$f := \frac{1}{m} \sum_{i \in [m]} f_i.$$
We assume we are given access to $f$ via a randomized oracle $F$ where
$\mathbb{E}[F] = f.$
We call such an oracle a stochastic zeroth-order oracle for $f$.
We are also given randomized oracles $G_{x}, G_{y}$ for $\nabla_x f,$
$\nabla_y f,$ where $\mathbb{E}[G_{x}] = \nabla_x f,$ and $\mathbb{E}[G_{y}] =\nabla_{y} f,$ and call such oracles stochastic
gradient oracles for $f$. 
These oracles are computed by
randomly sampling ``batches'' $B, B_x, B_y \subseteq [m]$ (iid, with replacement) and returning
$$F = \frac{1}{|B|} \sum_{i \in B} f_i, G_{x} = \frac{1}{|B_x|} \sum_{i
  \in B_x} \nabla_{x} f_i, \text{ and }
G_{y} = \frac{1}{|B_y|} \sum_{i \in B_y} \nabla_{y} f_i.$$ 
For our convergence guarantees, we require the following bounds on
standard
smoothness parameters for each $f_i:$
$b, L>0$  such that  $|f_i(x,y)| \le b$
and
$$\|\nabla f_i(x,y) - \nabla f_i(x',y')\| \le {L}\|x-x'\| +
{L}\|y-y'\|$$ for all $x,y$ and all $i$.
These bounds imply $f$ is also $b$-bounded,
and ${L}$-gradient-Lipschitz.

\begin{algorithm}[H]
\small
\caption{ Stochastic gradient ascent (for max-player updates) \label{alg:InnerMaxLoop}}
\textbf{input:}   Stochastic gradient oracle $G_{y}$ for $\nabla_{y} f$; initial points $\mathsf{x}, \mathsf{y}_0$; error parameter $\epsilon'$\\
  \textbf{hyperparameters:} $\eta>0$ 

 \begin{algorithmic}[1]
 \STATE  Set $\aj \leftarrow 0$ \, \,\,, $\mathsf{Stop} \leftarrow \textrm{False}$
\WHILE{$\mathsf{Stop} = \mathsf{False}$}%
\STATE Set $g_{\mathrm{y}, \aj} \leftarrow G_{y}(\mathsf{x}, \mathsf{y}_{\aj})$ 
\IF{$\| g_{\mathrm{y}, \aj}\| > \epsilon'$} \label{DiscADAMStart}
\STATE Set $\mathsf{y}_{\aj+1} \leftarrow \mathsf{y}_{\aj} + \eta g_{\mathrm{y}, \aj}$  , \label{step_SGD_update} 
 \STATE  Set $\aj \leftarrow \aj + 1$
\ELSE
\STATE Set $\mathsf{Stop} \leftarrow\textrm{True}$
\ENDIF
\ENDWHILE 
\STATE  {\bf return}  $\mathsf{y}_{\mathrm{stationary}} \leftarrow \mathsf{y}_{\aj}$
\end{algorithmic}
\end{algorithm}

\paragraph{Overview and intuition of our algorithm.} %
From the current point $(x,y)$, Algorithm~\ref{alg:Informal} first proposes a random update $\Delta$ from the given distribution $Q_{x,y}$ to update the min-player's parameters to $x + \Delta$.
In practice, we oftentimes choose ${Q}_{x,y}$ to be the distribution of the (scaled) stochastic gradient $-G_x$, although one may implement Algorithm~\ref{alg:Informal} with any choice of distribution ${Q}_{x,y}$.
Then, it  updates the max-player's parameters greedily by running gradient ascent using the stochastic gradients $G_y$ until it reaches a first-order $\epsilon$-stationary point $y'$, that is, a point where  $$\|\nabla_y f(x+\Delta,y')\| \leq \epsilon.$$
Thus, the point $y'$ satisfies \eqref{eq_approx_local_equilibrium_y}.
However, Algorithm~\ref{alg:Informal} still needs to eventually find a pair of points $(x^\star,y^\star)$ where $x^\star$ is an approximate local minimum of the min-player's objective $\mathcal{L}_\eps(\cdot, y^\star)$ in order to satisfy \eqref{eq_approx_local_equilibrium_x}.
Moreover, it must also ensure that 
$y^\star$ satisfies \eqref{eq_approx_local_equilibrium_y}.
Towards this, Algorithm~\ref{alg:Informal} does the following:

1) The algorithm re-uses this same point $y'$ to compute an approximation  $f(x+\Delta, y')$ for $\mathcal{L}_{\eps}(x+\Delta,y)$
in order to have access to the value of the min-player's objective $\mathcal{L}_{\eps}$ to be able to minimize it.

2) If $f(x+\Delta,y')$ is less than $f(x,y)$ the algorithm concludes that $\mathcal{L}_{\eps}(x + \Delta,y)$ has decreased and, consequently, accepts the updates $x+\Delta$ and $y'$; otherwise it rejects both updates.
We show that after accepting both $x + \Delta$ and $y'$, $\mathcal{L}_\eps(x + \Delta, y') < \mathcal{L}_\eps(x, y)$, implying that the algorithm does not cycle.

3) It then starts the next iteration proposing a random update which again depends on its current position.

4) While Algorithm~\ref{alg:Informal} does not cycle, to avoid getting stuck, if it is unable to decrease $\mathcal{L}_\eps$ after roughly $1/\omega$ attempts,  it concludes w.h.p. that the current $x$ is an approximate local minimum for $\mathcal{L}_\eps(\cdot, y)$ with respect to the given distribution.
This is because, by definition, at an approximate local minimum, a random update from the given distribution has probability at most $\omega$ of decreasing 
$\mathcal{L}_\eps$.
We also show that the current $y$ is an $\eps$-stationary point for $f(x,\cdot)$.%

We conclude this section with a few  remarks:
1) In practice our algorithm can be implemented just as easily with ADAM instead of SGD, as in some of our  experiments
{(alternately, one may also be able to substitute other optimization algorithms such as Momentum SGD \cite{polyak1964some}, ADAGrad \cite{duchi2011adaptive}, or Adabelief \cite{zhuang2020adabelief} for gradient updates).}
 2)  Algorithm \ref{alg:Informal} uses a randomized accept-reject rule (similar to simulated annealing)-- if the resulting loss has decreased, the updates for $x$ and $y$ are accepted; otherwise they are only accepted with a small probability $e^{-i/\tau_1}$ at each iteration $i$, where $\tau_1$ is a ``temperature'' parameter.
3) While our main result still holds if one replaces simulated annealing with a deterministic acceptance rule, the  annealing step seems to be beneficial in practice in the early period of training when our algorithm is implemented with ADAM gradients.
4) Finally, in simulations,  we find that Algorithm \ref{alg:Informal}'s implementation can be simplified by taking a small fixed number of max-player updates at each iteration.

\subsection{Convergence Guarantee} 

\begin{theorem}[{\bf Main result}]\label{thm:GreedyMinimax-main}
%
  Algorithm \ref{alg:Informal}, with
   hyperparameters $\eta>0$, $\tau_1 >0$, given access to stochastic zeroth-order and gradient oracles for a
  function $f = \sum_{i \in [m]} f_i$
  where each $f_i$ is $b$-bounded with ${L}$-Lipschitz gradient for some $b,{L} > 0$, and
  $\epsilon, \delta, \omega > 0$, and an oracle for sampling from a distribution ${Q}_{x,y}$, with probability at least $9/10$ returns
  $(x^\star,y^\star) \in \mathbb{R}^{d} \times \mathbb{R}^{d}$
  such that, for some
  $\epsilon^\star \in\left[ 0.5\epsilon, \epsilon\right]$,  $(x^\star, y^\star)$ is an $(\epsilon^\star, \delta, \omega,Q)$-approximate local equilibrium. 
  The number of stochastic gradient, function, and sampling oracle calls required by the algorithm is $\mathrm{poly}\left(b,L, 1/\eps, 1/\delta, 1/\omega \right)$ and does not depend on the dimension $d$.
\end{theorem}
\noindent
Theorem \ref{thm:GreedyMinimax-main} says that our algorithm is guaranteed to converge to an approximate local equilibrium for our framework from any starting point, for any $f$ which is bounded with Lipschitz gradients including nonconvex-nonconcave $f$.
As discussed in related work this is in contrast to  prior works which assume e.g., that $f(x,y)$ is concave in $y$ or that the algorithm is provided with an initial point such that the underlying continuous dynamics converge to a local min-max point.
 The exact number of stochastic gradient, function, and sampling oracle calls required by the algorithm is $\tilde{O}(b^3 L^3/ (\delta^3 \omega^3 \varepsilon^4))$.
We present a proof overview for Theorem \ref{thm:GreedyMinimax-main} next,
and the full proof in Section \ref{sec:proofs}.

In the setting where $f(x,y)$ is $\alpha$-strongly convex in $x$ and $\alpha$-strongly concave in $y$ and has $L$-Lipschitz gradients, Algorithm \ref{thm:GreedyMinimax-main} with $x$-player updates $Q_{x,y}$ chosen to be the $x$-gradients $-\nabla_x f(x,y)$, outputs a point $(x^\star,y^\star)$ which is an $(\eps, \eps, 0.25, Q)$-equilibrium point in  $\mathrm{poly}(1/\epsilon, L, 1/\alpha, D)$ gradient evaluations, where $D:= \|(x_0, y_0) - (x^\dagger, y^\dagger)\|$ is the distance from the initial point to the global min-max point $(x^\dagger, y^\dagger)$ .
 As mentioned in Section \ref{sec_intro}, since $f$ is strongly convex-strongly concave with Lipschitz gradients, this point is also an approximate global min-max point with duality gap $O(\eps)$, and the number of gradient evaluations for our algorithm to achieve a duality gap $O(\eps)$ is independent of the dimension (Corollary \ref{cor_runtime_strongly_convex}).

\color{black}

\begin{figure*}[t]
    \centering
       \includegraphics[height=3.3cm]{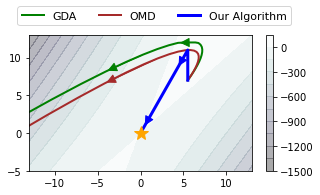}
         \includegraphics[height=3.3cm]{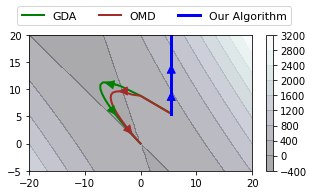}
         \includegraphics[height=3.3cm]{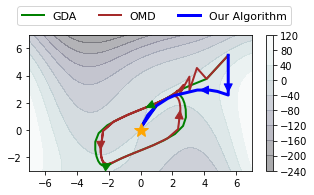}
         \caption{\small Our algorithm (blue), GDA (green), and OMD (red) on test functions $F_1$ (left), $F_2$ (center), and $F_3$ (right).
         $F_1$ and $F_3$ have global min-max points at (0,0) (yellow star), and $F_2$ has no min-max points since $\min_{x \in \mathbb{R}} \max_{y \in \mathbb{R}} F_2(x,y) = + \infty$. 
         }
         \label{figure_toy_function}
\end{figure*}

\section{Proof Overview for Theorem \ref{thm:GreedyMinimax-main}}
For simplicity, assume $b = L = \tau_1 = 1$ and $\eps = \delta = \omega$.  
There are two key pieces to proving Theorem~\ref{thm:GreedyMinimax-main}.
The first is to show that our algorithm converges to some point $(x^\star, y^\star)$ in a number of gradient, function, and sampling oracle calls that is $\mathrm{poly}(1/\epsilon)$ and independent of the dimension $d$ (Lemma \ref{lemma_polytime_informal}).
 Secondly, we show that, $y^\star$ is a first-order $\epsilon$-stationary point for $f(x^\star, \cdot)$, and $x^\star$ is an approximate local minimum of $\mathcal{L}_\epsilon(\cdot, y^\star)$ (Lemma \ref{lemma_greedyMinimax_informal}).

\paragraph{Step 1: Bounding the number of oracle evaluations.}
\begin{lemma}[\textbf{Informal, see Lemma \ref{lemma_polytime}}]\label{lemma_polytime_informal}
Algorithm \ref{alg:Informal} terminates after at most \\$\mathrm{poly}(b,L,1/\eps, 1/\delta, 1/\omega)$ gradient, function, and sampling oracle evaluations.
\end{lemma}
\begin{proof}[Proof outline of Lemma \ref{lemma_polytime_informal}]
After  $\Theta(\log(1/\epsilon))$ iterations of Algorithm \ref{alg:Informal}, the decaying acceptance rate (Line \ref{accept_reject_step} of Algorithm \ref{alg:Informal}) ensures that, with probability at least $1 - O(\eps)$, at any iteration $i$ for which Algorithm \ref{alg:Informal} accepts a proposed update to $(x_i, y_i)$, we have that 
\be \label{eq_proof_overview1}
f(x_{i+1}, y_{i+1}) \leq f(x_i, y_i) - \eps.
\ee
Next, we note that the stopping condition in Line \ref{InnerWhileStart} of Algorithm \ref{alg:Informal} implies 
 our algorithm stops whenever $r_{\mathrm{max}} = \Theta(1/\epsilon)$ proposed steps are rejected in a row.
Thus, \eqref{eq_proof_overview1} implies that for every $ \Theta(r_\mathrm{max})$ iterations where the algorithm does not terminate, with probability at least $1 - O(\epsilon)$ the value of the loss decreases by at least $\Omega(\epsilon)$.
 Since $f$ is 1-bounded, this implies our algorithm terminates after roughly  $O(r_\mathrm{max}/\epsilon)$ iterations of the minimization routine w.h.p. (Prop.~\ref{thm_OuterRuntime}).

Next, we use the fact that $G_{y}(x,y)$ is a batch gradient, $$G_{y}(x,y) = \frac{1}{|B_y|} \sum_{i \in B_y} \nabla_{y} f_i(x,y),$$  of batch size $|B_y| = O(1/\eps^{2}
\log(1/\eps))$,
together with the Azuma–Hoeffding concentration inequality, 
 to show w.h.p. that
 \begin{equation} \label{eq_proof_overview2}
 \|G_y(x,y) - \nabla_y f(x, y)\| \leq O(\eps),
 \end{equation}
(Proposition \ref{Prop_Azuma}).  We then use \eqref{eq_proof_overview2} together with the fact $f$ is 1-bounded with 1-Lipschitz gradient, to show that, w.h.p., the maximization subroutine (Algorithm \ref{alg:InnerMaxLoop}) requires at most $\mathrm{poly}(1/\epsilon)$ stochastic gradient ascent steps to reach an $\eps$-stationary point (Proposition \ref{thm_InnerRuntime}). 
As each step of the max-subroutine requires one gradient evaluation, and each iteration of the min-routine calls the max-subroutine once (and makes O(1) oracle calls), the total number of oracle calls is $\mathrm{poly}(1/\epsilon)$.
\end{proof}

\paragraph{Step 2: Show  $x^\star$ is approximate local minimum for $\mathcal{L}_\epsilon(\cdot, y^\star)$, and $y^\star$ is $\epsilon$-stationary point.} 

\begin{lemma}[\textbf{Informal, see Lemma \ref{lemma_greedyMinimax}}]\label{lemma_greedyMinimax_informal}
W.h.p., the output $(x^\star, y^\star)$ of Algorithm \ref{alg:Informal} is an approximate local equilibrium for our framework, for parameters $(\eps, \delta, \omega)$ and proposal distribution $Q$.
\end{lemma}
\begin{proof}[Proof outline of Lemma \ref{lemma_greedyMinimax_informal}.]
Since we have already shown that Algorithm \ref{alg:InnerMaxLoop} runs stochastic gradient ascent until it reaches a $\eps$-stationary point, $\|\nabla_{\mathrm{y}} f(x^\star, y^\star)\| \leq \epsilon$. 
The accept/reject rule (Line \ref{accept_reject_step} of Algorithm \ref{alg:Informal}) says that the proposed update $x^\star + \Delta$ is rejected with probability at least $1-O(\eps)$ whenever
\be \label{eq_proof_overview3}
    f(x^\star +
   \Delta,y') \geq f(x^\star ,y^\star) - \epsilon,
   \ee
    where the maximization subroutine computes $y'$ by gradient ascent on $f(x^\star+\Delta, \cdot)$ initialized at $y^\star$. 
And the stopping condition in Line \ref{InnerWhileStart} of Algorithm \ref{alg:Informal} implies that the last $r_{\mathrm{max}}$ updates $x^\star + \Delta$ proposed by the  min-player were all rejected,
  and hence were sampled from the distribution ${Q}_{x^\star,y^\star}$. 
     Roughly, 
    this fact together with \eqref{eq_proof_overview3}  implies that, with high probability, the proposal distribution  ${Q}_{x^\star,y^\star}$ at the point $(x^\star,y^\star)$ satisfies
 \be 
{\Pr_{\Delta \sim {Q}_{x^\star,y^\star}}[ f(x^\star + \Delta,y') \geq  f(x^\star ,y^\star)}& - \epsilon ] \geq 1 - O(r_{\mathrm{max}}^{-1})
   \\ &= 1-O(\epsilon).\label{eq:f_bound}
 \ee
  To show \eqref{eq_approx_local_equilibrium_x} holds, we need to replace $f$ in the above equation with the  min-player's objective $\mathcal{L}_\epsilon$.  
  Towards this end, we first use the fact that $f$ has $O(1)$-Lipschitz gradient,
    together with \eqref{eq_proof_overview2}, to show that, w.h.p., the stochastic gradient ascent steps of Algorithm \ref{alg:InnerMaxLoop} form an  ``$\epsilon$-increasing"  path, starting at $y^{\star}$ with endpoint $y'$, along which $f$ increases at rate at least $\epsilon$ (Prop. \ref{Lemma_Greedypath}).
    Since $\mathcal{L}_\epsilon$ is the supremum of $f$ at the endpoints of {\em all} such $\epsilon$-increasing paths starting at $y^\star,$ 
\begin{equation} \label{eq:g_bound}
 f(x^\star + \Delta, y') \leq \mathcal{L}_\epsilon(x^\star + \Delta, y^\star).
\end{equation} 
 Finally,  recall (Section~\ref{sec_theoretical_results}) that  $\|\nabla_{\mathrm{y}} f(x^\star, y^\star)\| \leq \epsilon^\star$ implies that 
     $\mathcal{L}_\epsilon(x^\star, y^\star) = f(x^\star, y^\star), $
 and hence  \eqref{eq_approx_local_equilibrium_y} holds. 
    Plugging this and \eqref{eq:g_bound} 
    into \eqref{eq:f_bound} implies that
     \be 
{\Pr_{\Delta \sim {Q}_{x^\star,y^\star}}\left[ \mathcal{L}_\epsilon(x^\star + \Delta,y')  \geq   \mathcal{L}_\epsilon(x^\star ,y^\star)  -  \epsilon \right]  \geq   1 - O(\epsilon),}
 \ee
    and hence that \eqref{eq_approx_local_equilibrium_x} holds.
    \end{proof}

\section{Empirical Results}
\label{sec:experiments}

\subsection{Performance on Test Functions} \label{sec_test_functions}
We apply our algorithm to three test loss functions previously considered in \cite{wang2019solving}:
 $$F_1(x,y)  = -3x^2 - y^2 + 4xy,$$
    $$F_2(x,y)  = 3x^2 + y^2 + 4xy,$$
    $$F_3(x,y)  = (4x^2 - (y - 3x + 0.05x^3)^2 - 0.1y^4)e^{-0.01(x^2 + y^2)}.$$
We choose these functions because they are known to be challenging for gradient-based algorithms.

Both $F_1$ and $F_3$ have global min-max at $(0,0)$, yet popular gradient-based algorithms including GDA, OMD, and extra-gradient (EG) algorithm were shown in \cite{wang2019solving} not to converge on these functions.
In contrast, we observe that our algorithm finds the global min-max points of both $F_1$ and $F_3$.
To see why our algorithm converges, note that it uses the maximization subroutine (Algorithm~\ref{alg:InnerMaxLoop}) to first find the ``ridge" along which $f(x,y)$ is a local maximum in the $y$ variable, and to then return to a point on the ridge every time the min-player proposes an update.
Since the min-player in our algorithm only accepts updates which lead to a net  decrease in $f$, our algorithm eventually finds the point $(0,0)$ on this ridge where $f$ is minimized.
In comparison, for GDA and OMD, the max-player's gradient $\nabla_y f$ is zero along the ridge where $f(x,y)$ is a local maximum in the $y$ variable, while the min-player's gradient $-\nabla_x f$ can be large; on $F_1$ and $F_3$ $-\nabla_x f$ points away from this ridge, and this can prevent GDA and OMD from converging to the point $(0,0)$.

In case of $F_2$, 
$\min_{x\in \mathbb{R}} \max_{y \in \mathbb{R}} f(x,y) =  + \infty$.
On $F_2$, 
GDA, OMD, and EG all {\em converge} to $(0,0)$ which is neither a global min-max 
nor a local min-max point.
In contrast, our algorithm
diverges to infinity.

When applying our algorithm 
we use 
$\eta = 0.05$ and 
$Q_{x,y} \sim N(0, 0.25)$. For GDA and OMD we use learning rate $0.05$    (see Appendix \ref{toy_simulation_setup}).

\begin{figure}[t]
\centering
\includegraphics[width=0.8\linewidth]{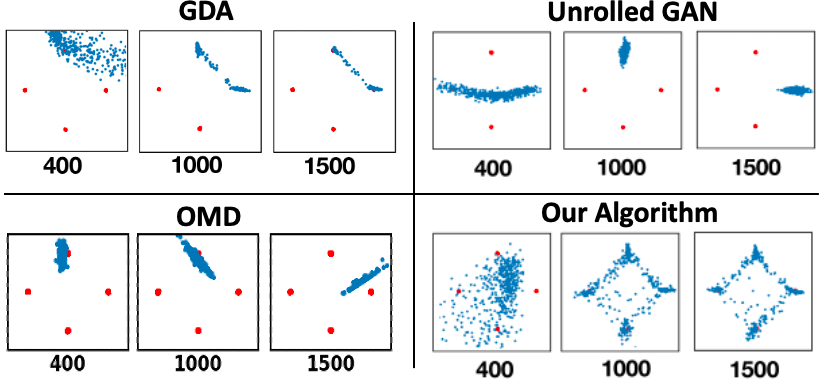}
\caption{\small Our algorithm, unrolled GANs with $k = 6$ unrolling steps, OMD, and  GDA with 
$k = 6$ max-player steps 
 trained on a 4-Gaussian mixture for 1500 iterations.  Our algorithm used $k = 6$ max-player steps and acceptance rate 
  $e^{-1/\tau} = 0.25$. 
 Plots show the points generated by each algorithm after the specified 
 iterations.}
\label{fig:4_Gaussians}
\end{figure}

\subsection{Performance when Training GANs}
We apply our algorithm to train GANs to learn from both synthetic and real-world datasets.
 When training on both datasets, we choose the proposal distribution $Q$ in our algorithm to be the (ADAM) stochastic gradients for $-\nabla_x f$.
We formulate GAN using our framework with cross entropy loss,
$$f(x,y) = - \left( \log(\mathcal{D}_y(\zeta)) + \log(1-
\mathcal{D}_y(\mathcal{G}_x(\xi))\right),$$ where $x, y$ are the parameters of generator $\mathcal{G}$ and discriminator $\mathcal{D}$
respectively,
$\zeta$ is sampled from data, and $\xi \sim N(0,I_d)$.

To adapt Algorithm \ref{alg:Informal} to training GANs, we make certain
simplifications: 
1) we use a fixed temperature $\tau$ at all iterations $i$, making
it simpler to choose a good temperature value, rather than a temperature schedule;
2) 
we replace the randomized acceptance rule with a deterministic rule:  If  $f_{\mathrm{new}} \leq  f_{\mathrm{old}}$ we accept, and if $f_{\mathrm{new}} > f_{\mathrm{old}}$ we only accept if $i$ is a multiple of $e^{\frac{1}{\tau}}$ 
  (i.e., average acceptance rate of $e^{-\frac{1}{\tau}}$);
3) 
  we take a fixed number of max-player steps at each iteration,
  instead of taking as many steps as needed to achieve a small gradient.
These simplifications do not 
significantly affect our algorithm's performance (see Appendix
 \ref{sec:More_simulations_annealing}).

\begin{table}[t]
\center
\caption{\small Gaussian mixture dataset. The fraction of times (out of 20 runs) each method generates $m$ modes, for $m \in [4]$.
$k$ is the number of max-player steps per iteration.
Our algorithm learns 4 modes in more runs than other algorithms.}
\small
\begin{tabular}{lcccc}
\toprule
 & \multicolumn{4}{c}{Number of modes learnt}\\
Method &  1 & 2 & 3 & 4 \\
\midrule
This paper & 0 & 0.15 & 0.15 & \textbf{0.70}  \\
GDA ($k=1$) & 0.95 & 0.05 & 0 & 0 \\
GDA ($k=6$) & 0.05 & 0.75 & 0 & 0.20 \\
OMD & 0.80 & 0.20 & 0 & 0 \\
Unrolled-GAN & 0.75 & 0.15 & 0.10 & 0 \\
\bottomrule
\end{tabular}
\label{tbl:gaussian_results}
\end{table}

\noindent
\textit{Gaussian mixture dataset.} This synthetic dataset consists of 512 points sampled from a mixture of four equally weighted Gaussians in two dimensions with standard deviation 0.01 and means at $(0,1)$, $(1,0)$, $(-1,0)$, $(0,-1).$
Since modes in this dataset are well-separated, mode collapse 
can be clearly detected. 
We report the number of modes learnt by the GAN from each training algorithm across iterations.

\paragraph{Baselines.} We compare
 our algorithm's performance
 to GDA, OMD \citep{Daskalakis2018optimism},
 and unrolled GAN \citep{Metz2017unrolled}. %
  For the 
 networks and hyperparameter details,
 see Appendix
 \ref{appendix_hyperparameters}.

\paragraph{Results.}
We trained GANs on the Gaussian mixture dataset for 1500
 iterations using our
 algorithm, unrolled GANs with 6 unrolling steps,  GDA
 with $k=1$ and $k=6$ max-player steps (using Adam updates), and OMD with $k=6$  max-player steps.
 We repeated each simulation 20
 times. 
The performance of the output GAN learned by all algorithms is presented in Table~\ref{tbl:gaussian_results}, while
Figure~\ref{fig:4_Gaussians} shows the samples from generators of the different training algorithms at various iterations (see Appendix
 \ref{sec:More_simulation_results}  for images
 from all runs).
The GAN returned by our algorithm learns all four modes in 70\% of the runs, significantly more than the other training algorithms (Table~\ref{tbl:gaussian_results}).  
{Thus, for this synthetic dataset, our algorithm is the most  effective in avoiding mode collapse and cycling in comparison to baselines.}

\begin{figure*}[t]
    \begin{minipage}{0.495\textwidth}
    \begin{center}
      \textbf{\small  GDA} \\
      \vspace{0.4mm}
    \fbox{\includegraphics[width=0.96\textwidth]{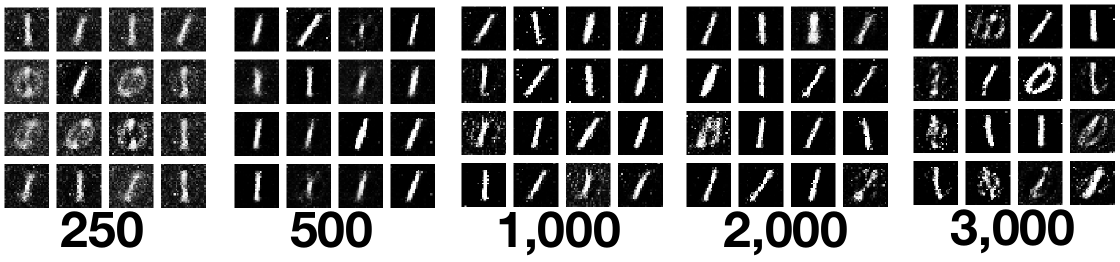}}
  \end{center}
  \end{minipage}
    \begin{minipage}{0.495\textwidth}
    \begin{center}
    \textbf{\small Our algorithm} \\
    \fbox{\includegraphics[trim={0, 0, 0, 0}, clip, width=0.96\textwidth]{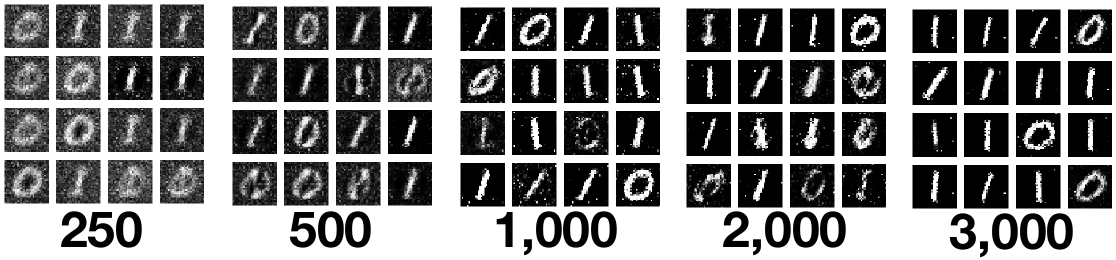}}
  \end{center}
  \end{minipage}
  \caption{Images generated at various iterations by the GAN trained using GDA vs our algorithm (for 3000 iterations) on the 01-MNIST dataset. See Appendix~\ref{sec:mnist_results} for more results and details.}
  \label{fig:mnist_examples}
\end{figure*}

\paragraph{Results on real-world datasets.}
{While we focus on  2-D  and Gaussian mixture GAN simulations in this section to illustrate the convergence properties of our algorithm, we also ran our algorithm on two real-world datasets, 01-MNIST and CIFAR-10.
For the 01-MNIST dataset, samples generated from GANs trained using GDA and our algorithm are presented in Figure~\ref{fig:mnist_examples}.
We observed that GANs trained on the  01-MNIST dataset with the gradient descent-ascent algorithm (GDA) exhibit mode collapse in 77\% of the trial runs (Figure~\ref{fig_01MNISTGDA},
in Appendix~\ref{sec:mnist_results}), while GANs trained with our algorithm do not exhibit mode collapse in any of the training runs (Figure~\ref{fig_01MNISTOur} in Appendix~\ref{sec:mnist_results}).
For the CIFAR-10 dataset, samples generated from GANs trained using our algorithm are presented in Figure~\ref{fig:cifar_examples}.
On CIFAR-10, our algorithm achieved a mean Inception score of 4.68 after 50k iterations (across 20 repetitions); in comparison, GDA achieved a mean Inception score of 4.51 and OMD achieved a mean Inception score of 1.96 (Table~\ref{tbl:cifar_results_main}).
Detailed results and methodologies used for 01-MNIST and CIFAR-10 datasets are presented in Appendix \ref{sec:mnist_results} and \ref{CIFAR_results_appendix} respectively.}

{
Experiments with GANs also demonstrate that our algorithm scales to high-dimensional parameter spaces; the dimension $d$ of the space of trainable parameters used in the GAN experiments was around $3.5\times 10^4$ for the GANs trained on the Gaussian mixture dataset, $3\times 10^6$ for 01-MNIST and $2\times 10^6$ for CIFAR-10.
\footnote{The code for the above simulations is available at \url{https://github.com/vijaykeswani/Min-Max-Optimization-Algorithm}.}
}

\begin{table}

\begin{center}
\caption{CIFAR-10 dataset: The mean (and standard error) of Inception Scores of models from different training algorithms.
Note that, GDA and our algorithm return generators with similar mean performance; however, the standard error of the Inception Score in case of GDA is relatively larger.
}
\label{tbl:cifar_results_main}
\small
\begin{tabular}{lccc}
\toprule
& \multicolumn{3}{c}{Iteration} \\
Method & 5000 & 25000 & 50000 \\
\midrule
Ours & 2.71 (0.28) & 4.10 (0.35)  & \textbf{4.68} (0.39) \\
GDA & 2.80 (0.52) & 4.28 (0.77) & 4.51 (0.86) \\
OMD & 1.60 (0.18) & 1.73 (0.25) & 1.96 (0.26) \\
\bottomrule
\end{tabular}
\end{center}
\end{table}

\begin{figure}
\centering
\includegraphics[width=0.5\linewidth, height=4.3cm]{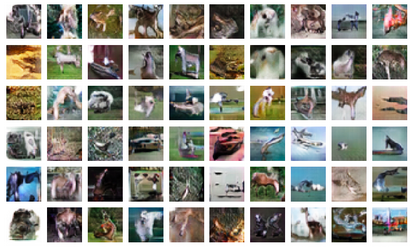}
\caption{Images generated by the GAN trained using our algorithm on the CIFAR-10 dataset. See Appendix~\ref{CIFAR_results_appendix} for more results and samples from GANs trained using other baselines.}
\label{fig:cifar_examples}
\end{figure}

 \section{Proof of Theorem \ref{thm:GreedyMinimax-main}} \label{sec:proofs}

In this section we give the  proof of Theorem \ref{thm:GreedyMinimax-main}.

\paragraph{Setting parameters:}
We start by setting parameters which will be used in the proof.
 Let $\mathfrak{b}_0 = |B|$,  $\mathfrak{b}_y = |B_y|$ denote the batch sizes. 
And note that the fact that each $f_i$ has $L$-Lipschitz gradient for all $i\in [m]$, implies that each $f_i$ is also $L_1$-Lipschitz, where $L_1= \sqrt{2L b}$.

For the theoretical analysis, we assume $0< \epsilon \leq 1$, and set the following parameters:
\begin{enumerate}

\item $\nu = \frac{1}{20}\bigg[\frac{320 b (L+1)}{\eps^2}\bigg(\tau_1 \log\left(\frac{128}{\omega^2}\right) + \frac{2048 b}{\omega \delta} \log^2\left(\frac{100}{\omega}\left(\tau_1 + 1\right)\left(8 \frac{b}{\delta} +1\right)\right)+ 1          \bigg)            \bigg]^{-2}$

\item $r_{\mathrm{max}} = \frac{128}{\omega} \log^2\bigg(\frac{100}{\omega}(\tau_1+1)\left(8 \frac{b}{\delta} +1\right) +\log\left(\frac{1}{\nu}\right)\bigg)$

\item Define $\mathcal{I}:= \tau_1 \log(\frac{r_{\mathrm{max}}}{\nu}) + 8r_{\mathrm{max}} \frac{b}{\delta} + 1$

\item $\eta = \min\left( \frac{1}{10L}, \frac{1}{8L \mathcal{I}}\right)$

\item Define $\mathcal{J} := \frac{16b}{\eta \epsilon^2}$

\item $\hat{\epsilon}_1 = \min\left(\eps \eta L, \frac{\delta}{8}\right)$

\item  $\mathfrak{b}_0 = \hat{\epsilon}_1^{-2}300^2 b^2 \log(\frac{1}{\nu})$

\item $\mathfrak{b}_{y} =  \epsilon^{-2}  \hat{\epsilon}_1^{-2}300^2 L_1^2 \log(\frac{1}{\nu})$

\end{enumerate}

\noindent
In particular, note that  $$\nu \leq  \frac{1}{20}\left(2 \mathcal{J} \mathcal{I} + 2 \times \left(r_{\mathrm{max}}
\frac{8b}{\delta} +
1\right)\right)^{-1} \text{ and } r_{\mathrm{max}} \geq \frac{4}{\omega} \log\left(\frac{100 \mathcal{I}}{\omega}\right).$$   At every iteration $i\leq \mathcal{I}$, where we set $\epsilon' = \epsilon_i$.
We also have 
\begin{equation} \label{eq_hyperparameter_1}
\epsilon' \leq \epsilon_0 \left(\frac{1}{1-2\eta L}\right)^{2i}  \leq \epsilon.
\end{equation}
To see why \eqref{eq_hyperparameter_1} holds, note that since we set the hyperparameter $\eta$ to be $\eta = \min \left(\frac{1}{10 L}, \frac{1}{8L \mathcal{I}}\right)$, we have $$1- 2 \eta L \leq 1-\frac{1}{4 \mathcal{I}}.$$
 Since we also set $\eps_0 = \frac{\eps}{2}$, we therefore have that for all $i \leq \mathcal{I}$,
 \begin{align*}
 \epsilon_0(1- 2 \eta L)^{-2 i } \leq \frac{\epsilon}{2} \left(1-\frac{1}{4 \mathcal{I}}\right)^{-2 \mathcal{I} }
 \leq \epsilon,
  \end{align*}
where the second inequality holds because $$\left(1-\frac{1}{2t}\right)^{-t} \leq 2$$ for all $t \geq 1.$

\subsection{Step 1: Bounding the Number of Gradient, Function, and Sampling Oracle Evaluations}\label{sec_step1}
 
 The first step in our proof is to bound the number of gradient, function, and sampling oracle evaluations required by our algorithm.
 Towards this end, we begin by showing a concentration bound (Proposition \ref{Prop_Azuma}) for the value of the stochastic gradient and function oracles used by our algorithm.
 Next, we bound the number of iterations of its discriminator update subroutine Algorithm \ref{alg:InnerMaxLoop} (Proposition \eqref{eq_stepsizey}), and the number of iterations in Algorithm \ref{alg:Informal} (Proposition \ref{thm_OuterRuntime}); together, these two bounds imply a $\mathrm{poly}(b,L, 1/\eps, 1/\delta, 1/\omega)$ bound on the number of gradient, function, and sampling oracle evaluations (Lemma \ref{lemma_polytime}).
 
\begin{proposition} \label{Prop_Azuma}
For any $\hat{\epsilon}_1, \nu>0$, if we use batch sizes $\mathfrak{b}_{y} = \epsilon^{-2} \hat{\epsilon}_1^{-2}300^2 L_1^2 \log(1/\nu)$ and  $\mathfrak{b}_0 = \hat{\epsilon}_1^{-2}300^2 b^2 \log(1/\nu)$, we have that
\be \label{eq_AzumaG}
\mathbb{P}\left(\|G_{y}(x,y) - \nabla_{y} f(x,y)\| \geq \frac{\hat{\epsilon}_1}{10}\right) < \nu,
\ee
and
\be \label{eq_AzumaF}
\mathbb{P} \left (|F(x,y) - f(x,y)| \geq \frac{\hat{\epsilon}_1}{10} \right) < \nu.
\ee
\end{proposition}

\begin{proof}

From Section \ref{sec_theoretical_results} we have that
\be
G_{y}(x,y) - \nabla_{y} f(x,y) = \frac{1}{\mathfrak{b}_{y}} \sum_{i \in B_y} [\nabla_{y} f_i(x,y)  - \nabla_{y} f(x,y)],
\ee
where the batch $B_y \subseteq [m]$ is sampled iid with replacement from $[m]$.
But since each $f_i$ has $L$-Lipschitz gradient, we have (with probability 1) that
\be
&\|\nabla_{y} f_i(x,y)  - \nabla_{y} f(x,y)\| \leq \|\nabla_{y} f_i(x,y)\| + \|\nabla_{y} f(x,y)\| 
 \leq 2 L_1.
\ee

\noindent
Now, 
\be
&\mathbb{E}[\nabla_{y} f_i(x,y)  - \nabla_{y} f(x,y)] = \mathbb{E}[\nabla_{y} f_i(x,y)  - \mathbb{E}[\nabla_{y} f_i(x,y)]] = 0.
\ee

\noindent
Therefore, by the Azuma–Hoeffding inequality for mean-zero bounded vectors, 
 we have
\be
&\mathbb{P}\bigg( \bigg\| \frac{1}{\mathfrak{b}_{y}} \sum_{i \in B_y} [\nabla_{y} f_i(x,y)  - \nabla_{y} f(x,y)]\bigg\|  \geq \frac{s \sqrt{\mathfrak{b}_{y}} +1}{\mathfrak{b}_{y}} 2 L_1\bigg) < 2e^{1- \frac{1}{2}s^2} \qquad \qquad \forall s >0.
\ee

\noindent
Hence, if we set $s  = 6 \log^{\frac{1}{2}}\left(\frac{2}{\nu}\right)$, we have that $$7 \log^{\frac{1}{2}}\left(\frac{2}{\nu}\right)  \sqrt{\mathfrak{b}_{y}} +1 \geq s  \sqrt{\mathfrak{b}_{y}} +1$$ and hence
\be
&\mathbb{P}\bigg( \bigg\| \frac{1}{\mathfrak{b}_{y}} \sum_{i \in B_y} [\nabla_{y} f_i(x,y)  - \nabla_{y} f(x,y)]\bigg\| \geq \frac{7 \log^{\frac{1}{2}}(\frac{2}{\nu})  \sqrt{\mathfrak{b}_{y}}}{\mathfrak{b}_{y}} 2 L_1\bigg) < \nu.
\ee

\noindent
Therefore,
\be
\mathbb{P}\bigg( \bigg\| \frac{1}{\mathfrak{b}_{y}} \sum_{i \in B_y} [\nabla_{y} f_i(x,y)  - \nabla_{y} f(x,y)]\bigg\|  \geq \frac{\hat{\epsilon}_1}{10} \bigg) < \nu
\ee
which completes the proof of Inequality \eqref{eq_AzumaG}.
Inequality \eqref{eq_AzumaF} follows from the exact same steps as the proof of Inequality \eqref{eq_AzumaG}, if we replace the bound $L_1$ for $\|\nabla_{y} f_i(x,y)\|$ with the bound  $b$ on $|f_i(x,y)|$.

\end{proof}

\begin{proposition} \label{thm_stepsize}
For every $j$, with probability at least $1-\nu$ we have that either $\|G_y(\mathsf{x}, \mathsf{y}_j)\|< \epsilon$, or that
\be\label{eq_stepsizeyb}
&\| \nabla_y f(\mathsf{x}, \mathsf{y}_j) - G_y(\mathsf{x}, \mathsf{y}_j)\| \leq \frac{1}{10} \eta L \times \min\left(\|G_y(\mathsf{x}, \mathsf{y}_j)\|, \,\, \|\nabla_y f(\mathsf{x}, \mathsf{y}_j)\| +  \frac{\hat{\epsilon}_1}{10}\right)
\ee
and
\be \label{eq_stepsizey}
\| \mathsf{y}_{\aj+1} - \mathsf{y}_{\aj}\| &= \eta \| G_y(\mathsf{x}, \mathsf{y}_j)\| \leq 2 \eta \| \nabla_y f(\mathsf{x}, \mathsf{y}_j)\|
\ee
\end{proposition}

\begin{proof}
By Proposition \ref{Prop_Azuma}, we have that, with probability at least $1-\nu$, and whenever $\|G_y(\mathsf{x}, \mathsf{y}_j)\| \geq \epsilon$,
\be
    &\| \nabla_y f(\mathsf{x}, \mathsf{y}_j) - G_y(\mathsf{x}, \mathsf{y}_j)\| < \frac{\hat{\epsilon}_1}{10} \leq \frac{1}{10} \epsilon \eta L \leq \frac{1}{10} \eta L \times \min\left(\|G_y(\mathsf{x}, \mathsf{y}_j)\|, \,\, \|\nabla_y f(\mathsf{x}, \mathsf{y}_j)\| +  \frac{\hat{\epsilon}_1}{10}\right),
\ee
where the first inequality holds by Proposition \ref{Prop_Azuma},  the second inequality holds since $\hat{\epsilon}_1 \leq \epsilon \eta L$, and the third inequality holds since $\|G_y(\mathsf{x}, \mathsf{y}_j)\| \geq \epsilon$ and since (again by Proposition \ref{Prop_Azuma}) $$\|\nabla_y f(\mathsf{x}, \mathsf{y}_j) - G_y(\mathsf{x}, \mathsf{y}_j)\| < \frac{\hat{\epsilon}_1}{10}.$$
This proves Inequality \eqref{eq_stepsizeyb}.
Moreover, we have that, whenever $\|G_y(\mathsf{x}, \mathsf{y}_j)\| \geq \epsilon$ and in the same probability $1-\nu$ event where \eqref{eq_stepsizeyb} holds,
\be\label{eq_stepsizey_c}
\|\eta G_y(\mathsf{x}, \mathsf{y}_j)\| \leq \eta \left(\| \nabla_y f(\mathsf{x}, \mathsf{y}_j)\| + \frac{\hat{\epsilon}_1}{10}\right).
\ee
 Thus,
 \begin{equation*}
     2 \eta \| \nabla_y f(\mathsf{x}, \mathsf{y}_j)\| \geq 2\eta\left(\|G_y(\mathsf{x}, \mathsf{y}_j)\|  - \frac{\hat{\epsilon}_1}{10}\right)
     \geq \eta \|G_y(\mathsf{x}, \mathsf{y}_j)\| = \| \mathsf{y}_{\aj+1} - \mathsf{y}_{\aj}\|,
 \end{equation*}
 where the first inequality holds by \eqref{eq_stepsizey_c}, 
  the second inequality holds since $\|G_y(\mathsf{x}, \mathsf{y}_j)\| \geq \epsilon \geq \hat{\epsilon}_1 $, and the equality holds by Step \ref{step_SGD_update} of Algorithm \ref{alg:InnerMaxLoop}.
This proves Inequality \eqref{eq_stepsizey}.
\end{proof}

\begin{proposition} \label{thm_InnerRuntime}
Algorithm \ref{alg:InnerMaxLoop} terminates in at most $\mathcal{J} := \frac{16b}{\eta \epsilon^2}$ iterations of its ``While" loop, with probability at least $1- \nu \times \mathcal{J}$.
\end{proposition}

\begin{proof}

Let $\aj_{\mathrm{max}} \in \mathbb{N}\cup\{\infty\}$ be the number of iterations of the ``While" loop in Algorithm \ref{alg:InnerMaxLoop}.
First, we note that the stopping condition for Algorithm \ref{alg:InnerMaxLoop} implies that 
\begin{equation} \label{eq_za1}
    \|G_y(\mathsf{x}, \mathsf{y}_j)\| \geq \frac{1}{2}\eps
\end{equation} for all $j \leq \aj_{\mathrm{max}} -1$.
Since $f$ has ${L}$-Lipschitz gradient, there exits a vector $u$, with $\|u\| \leq {L} \|\mathsf{y}_{\aj+1} - \mathsf{y}_{\aj}\|$, such that, for all $j \leq \aj_{\mathrm{max}} -1$

\be   \label{eq_c4}
f(\mathsf{y}_{\aj+1}) - f(\mathsf{y}_{\aj}) &= \langle \mathsf{y}_{\aj+1} - \mathsf{y}_{\aj}, \, \, \nabla_{y} f(\mathsf{x}, \mathsf{y}_{\aj}) + u  \rangle\\
& = \langle \mathsf{y}_{\aj+1} - \mathsf{y}_{\aj}, \, \, \nabla_{y} f(\mathsf{x}, \mathsf{y}_{\aj})  \rangle +  \langle \mathsf{y}_{\aj+1} - \mathsf{y}_{\aj}, \, \, u  \rangle\\
& = \langle \eta G_y(\mathsf{x}, \mathsf{y}_j), \, \, G_y(\mathsf{x}, \mathsf{y}_j)  \rangle - \langle   \eta G_y(\mathsf{x}, \mathsf{y}_j), \, \, G_y(\mathsf{x}, \mathsf{y}_j) - \nabla_{y} f(\mathsf{x}, \mathsf{y}_{\aj})  \rangle +  \langle \eta G_y(\mathsf{x}, \mathsf{y}_j), \, \, u  \rangle\\
& \geq \eta \|G_y(\mathsf{x}, \mathsf{y}_j)\|^2 
-  \eta\| G_y(\mathsf{x}, \mathsf{y}_j)\|\times \|G_y(\mathsf{x}, \mathsf{y}_j) - \nabla_{y} f(\mathsf{x}, \mathsf{y}_{\aj})\|
-  \eta\| G_y(\mathsf{x}, \mathsf{y}_j)\|\times \|u\|\\
& \stackrel{\textrm{Prop.}  \ref{thm_stepsize}} \geq   \eta \|G_y(\mathsf{x}, \mathsf{y}_j)\|^2
-  \eta\| G_y(\mathsf{x}, \mathsf{y}_j)\|\times \frac{\eta L}{10} \|G_y(\mathsf{x}, \mathsf{y}_j)\|
- \eta\| G_y(\mathsf{x}, \mathsf{y}_j)\|\times L \|\mathsf{y}_{\aj+1} - \mathsf{y}_{\aj}\|\\
 & = \eta \|G_y(\mathsf{x}, \mathsf{y}_j)\|^2 -  \frac{1}{10}  \eta^2 L \|G_y(\mathsf{x}, \mathsf{y}_j)\|^2
- \eta\| G_y(\mathsf{x}, \mathsf{y}_j)\|\times L \|\eta G_y(\mathsf{x}, \mathsf{y}_j)\|\\
 &\geq  \frac{1}{8} \eta \|G_y(\mathsf{x}, \mathsf{y}_j)\|^2 \\
 &\stackrel{\textrm{Eq.}  \ref{eq_za1}} \geq   \frac{1}{2} \eta \epsilon^2,
\ee
with probability at least $1-\nu$, where the second-to-last inequality holds since $\eta \leq \frac{1}{10 L}$.
 Existence of the vector $u$ in Equation  \eqref{eq_c4} is guaranteed by the fundamental theorem of calculus.
Namely, by the fundamental theorem of calculus we have
\begin{equation*}
    f(\mathsf{y}_{j+1})-f(\mathsf{y}_{j}) 
    =   \int_0^1 \langle \mathsf{y}_{j+1}- \mathsf{y}_{j},    \nabla_y f(\mathsf{x},   \mathsf{y}_j + t (\mathsf{y}_{j+1}- \mathsf{y}_{j}) )  \rangle \mathrm{d}t.
\end{equation*}
 Thus \eqref{eq_c4} holds for $u = \int_0^1 \nabla_y f(\mathsf{x}, \mathsf{y}_j + t (\mathsf{y}_{j+1}- \mathsf{y}_{j}) )  - \nabla_y f(\mathsf{x}, \mathsf{y}_j) \mathrm{d}t$.
Note that this choice of $u$ satisfies $\|u\| \leq L \| \mathsf{y}_{j+1}- \mathsf{y}_{j}\|$,  since
$\|\nabla_y f(\mathsf{x}, \mathsf{y}_j + t (\mathsf{y}_{j+1}- \mathsf{y}_{j}) )  - \nabla_y f(\mathsf{x}, \mathsf{y}_j)\| \leq L \|\mathsf{y}_{j+1}- \mathsf{y}_{j}\|$  for $t\in [0,1]$ because $f$ has $L$-Lipschitz gradient.
Since $f$ takes values in $[-b, b]$, Inequality \eqref{eq_c4} implies that Algorithm \ref{alg:InnerMaxLoop} terminates in at most  $\mathcal{J} = \frac{16b}{\eta \epsilon^2}$ iterations of its ``While" loop, with probability at least $1- \nu \times \mathcal{J}$.  

\end{proof}

\begin{proposition} \label{thm_OuterRuntime}
Algorithm \ref{alg:Informal} terminates in at most $\mathcal{I}:= \tau_1 \log(\frac{r_{\mathrm{max}}}{\nu}) + 8r_{\mathrm{max}} \frac{b}{\delta} + 1$ iterations of its ``While" loop, with probability at least $1-2\nu \times (r_{\mathrm{max}} \frac{8b}{\delta} + 1)$.
\end{proposition}

\begin{proof}

For any $i > 0$, let $E_i$ be the ``bad" event that both $f(x_{i+1}, y_{i+1}) - f(x_{i}, y_{i}) > -\frac{\delta}{4}$ and $\mathsf{Accept}_i = \mathsf{True}$.  
%
Then by Proposition \ref{Prop_Azuma}, since $\frac{\hat{\epsilon}_1}{10} \leq \frac{\delta}{8}$, we have that
\be \label{eq_c5}
\mathbb{P}(E_i) \leq e^{-\frac{i}{\tau_1}} + \nu.
\ee
\noindent
Define $\hat{\mathcal{I}} := \tau_1 \log(\frac{r_{\mathrm{max}}}{\nu})$.
Then for $i \geq \hat{\mathcal{I}}$, from Line \ref{AcceptRejectStep} of Algorithm \ref{alg:Informal} we have by Inequality \eqref{eq_c5} that
\be
\mathbb{P}(E_i) \leq 2\nu.
\ee

\noindent
Define $h := r_{\mathrm{max}} \frac{8b}{ \delta} + 1$.  Then
\be  \label{eq_c6}
\mathbb{P}\left(\bigcup_{i=\hat{\mathcal{I}}}^{\hat{\mathcal{I}} + h} E_i\right) \leq 2\nu \times h.
\ee

\noindent
  Since $f$ takes values in $[-b, b]$, if $\bigcup_{i=\hat{\mathcal{I}}}^{\hat{\mathcal{I}} + h} E_i$ does not occur, the number of accepted steps over the iterations $\hat{\mathcal{I}} \leq i \leq \hat{\mathcal{I}} + h$ (that is, the size of the set $\{i : \hat{\mathcal{I}} \leq i \leq \hat{\mathcal{I}} + h, \mathsf{Accept}_i = \mathsf{True}\}$) is at most $\frac{8b}{ \delta}$.

Therefore, since $h = r_{\mathrm{max}} \frac{8b}{ \delta} + 1$, we must have that there exists a number $\mathfrak{i}$, with  $\hat{\mathcal{I}} \leq \mathfrak{i} \leq \mathfrak{i} + r_{\mathrm{max}} \leq \hat{\mathcal{I}} + h$, such that $\mathsf{Accept}_i = \mathsf{False}$ for all $i \in [\mathfrak{i}, \mathfrak{i} + r_{\mathrm{max}}]$.
Therefore the condition in the While loop (Line \ref{InnerWhileStart}) of Algorithm \ref{alg:Informal} implies that Algorithm \ref{alg:Informal} terminates after at most $\mathfrak{i} + r_{\mathrm{max}} \leq \hat{\mathcal{I}} + h$ iterations of its While loop, as long as $\bigcup_{i=\hat{\mathcal{I}}}^{\hat{\mathcal{I}} + h} E_i$ does not occur.
Therefore, Inequality \eqref{eq_c6} implies that,  with probability at least $1-2\nu \times \left(r_{\mathrm{max}} \frac{8b}{\delta} + 1\right)$,  Algorithm \ref{alg:Informal} terminates after at most 
\be
\hat{\mathcal{I}} + h = \tau_1 \log\left(\frac{r_{\mathrm{max}}}{\nu}\right) + 8r_{\mathrm{max}} \frac{b}{\delta} + 1
\ee
 iterations of its ``While" loop.

\end{proof}

\begin{lemma} \label{lemma_polytime}
With probability at least $1- 3 \nu  \mathcal{J}  \mathcal{I}$, Algorithm \ref{alg:Informal} terminates after at most $(\tau_1 \log(\frac{r_{\mathrm{max}}}{\nu}) + 4r_{\mathrm{max}} \frac{b}{\delta} + 1) \times (\mathcal{J} \times \mathfrak{b}_{y} + \mathfrak{b}_{\mathrm{0}} + \mathfrak{b}_{x})$ gradient, function, and sampling oracle evaluations.
\end{lemma}

\begin{proof}

Each iteration of the While loop in Algorithm \ref{alg:Informal} computes one batch gradient with batch size $\mathfrak{b}_{x}$, one stochastic function evaluation of batch size $\mathfrak{b}_{\mathrm{0}}$, generates one sample from the proposal distribution $Q$, and calls Algorithm \ref{alg:InnerMaxLoop} exactly once. 
Each iteration of the While loop in Algorithm \ref{alg:InnerMaxLoop} computes one batch gradient with batch size $\mathfrak{b}_{y}$.
The result then follows directly from Propositions \ref{thm_OuterRuntime} and \ref{thm_InnerRuntime}.
\end{proof}

\subsection{Step 2: Proving the Output $(x^\star, y^\star)$ of Algorithm \ref{alg:Informal} is an Approximate Local Equilibrium} \label{sec_step2}

The second step in our proof is to show that the output of Algorithm \ref{alg:Informal} is an approximate local equilibrium (Definition \ref{def_greedy_minmax}) for our framework with respect to $\epsilon, \delta, \omega>0$ and the distribution $Q_{x,y}$ (Lemma \ref{lemma_greedyMinimax}).
Towards this end, we first show that the steps taken by the discriminator update subroutine (Algorithm \ref{alg:InnerMaxLoop}) form a path along which the loss $f$ is increasing (Proposition \ref{Lemma_Greedypath}).

Recall the paths $\gamma(t)$ from Definition \ref{def_path}.  From now on we will refer to such paths as ``$\epsilon$-increasing paths".  That is, for any $\epsilon>0$, we say that a path $\gamma(t)$ is an ``$\epsilon$-increasing path" if  at every point along this path we have that $\left\| \frac{\mathrm{d}}{\mathrm{d}t} \gamma(t)\right\| = 1$ and that  $\frac{\mathrm{d}}{\mathrm{d}t} f(x, \gamma(t)) \geq \epsilon$.

\begin{proposition} \label{Lemma_Greedypath}
Every time Algorithm \ref{alg:InnerMaxLoop} is called we have that, with probability at least $1- 2 \nu  \mathcal{J}$, the path consisting of the line segments  $[\mathsf{y}_{\aj}, \mathsf{y}_{\aj+1}]$ formed by the points $\mathsf{y}_{\aj}$ computed by Algorithm \ref{alg:InnerMaxLoop} has a parametrization $\gamma(t)$ which is an $ (1- 2\eta L)\epsilon'$-increasing path.
\end{proposition}

\begin{proof}
We consider the following continuous unit-speed parametrized path $\gamma(t)$:
\be
&\gamma(t) = \mathsf{y}_{\aj}+ \left(t - \sum_{k=1}^{j-1} \|v_{k} \|\right) \frac{v_{\aj}}{\|v_{\aj} \|},
\qquad \qquad \forall \,\, t\in \left[\sum_{k=1}^{j-1} \|v_{k} \|, \sum_{k=1}^{j} \|v_{k}\| \right], \quad j \in [j_{\mathrm{max}}],
\ee
where $v_{\aj} = \eta G_y(\mathsf{x}, \mathsf{y}_j)$  and $\aj_{\mathrm{max}}$ is the number of iterations of the While loop of Algorithm \ref{alg:InnerMaxLoop}.
Next, we show that $\frac{\mathrm{d}}{\mathrm{d}t} f(\mathsf{x}, \gamma(t)) \geq \epsilon'$.
For each $j \in [j_{\mathrm{max}}]$ we have that
\be 
\frac{\mathrm{d}}{\mathrm{d}t} f(\mathsf{x}, \gamma(t)) &\geq  [\nabla_{y}f(\mathsf{x}, \mathsf{y}_j) -{L} \|\mathsf{y}_{\aj+1} - \mathsf{y}_{\aj}\| u  ]^\top \frac{v_j}{\|v_j\|}\\
&= [\nabla_{y}f(\mathsf{x}, \mathsf{y}_j) - {L} \eta \|G_y(\mathsf{x}, \mathsf{y}_j)\| u  ]^\top \frac{v_j}{\|v_j\|}\\
&\stackrel{\textrm{Prop.}  \ref{Prop_Azuma}} \geq    \bigg[G_y(\mathsf{x}, \mathsf{y}_j) -  \frac{1}{10} \eta L \|G_y(\mathsf{x}, \mathsf{y}_j)\|  w - {L} \eta \|G_y(\mathsf{x}, \mathsf{y}_j)\| u \bigg]^\top \frac{G_y(\mathsf{x}, \mathsf{y}_j)}{\|G_y(\mathsf{x}, \mathsf{y}_j)\|}\\
&\geq \|G_y(\mathsf{x}, \mathsf{y}_j)\| -  \frac{1}{10} \eta L \|G_y(\mathsf{x}, \mathsf{y}_j)\| - {L} \eta \|G_y(\mathsf{x}, \mathsf{y}_j)\|\\
&\geq (1- 2\eta L) \|G_y(\mathsf{x}, \mathsf{y}_j)\|\\
&\label{eq_d1}\geq (1- 2\eta L) \epsilon' \qquad \qquad \qquad \forall t\in \left[\sum_{k=1}^{j-1} \|v_{k} \|, \sum_{k=1}^{j} \|v_{k}\| \right],
\ee
with probability at least $1-\nu$ for some unit vectors $u, w \in \mathbb{R}^d$.
But by Proposition \ref{thm_InnerRuntime} we have that $\aj_{\mathrm{max}} \leq  \mathcal{J}$  with probability at least  $1- \nu \times \mathcal{J}$.  Therefore inequality \eqref{eq_d1} implies that
\be
\frac{\mathrm{d}}{\mathrm{d}t} f(\mathsf{x}, \gamma(t)) &\geq (1- 2\eta L)  \epsilon'   \qquad \forall t\in \left[0, \sum_{k=1}^{j_{\mathrm{max}}} \|v_{k}\|\right],
\ee
with probability at least $1- 2 \nu  \mathcal{J}$.

\end{proof}

\begin{lemma} \label{lemma_greedyMinimax}
Let $i^{\star}$ be such that $i^{\star}-1$ is the last iteration $i$ of the ``While" loop in Algorithm \ref{alg:Informal} for which $\mathsf{Accept}_{i} = \mathsf{True}$.
Then with probability at least $1- 2\nu \mathcal{J} \mathcal{I} - 2\nu \times \left(r_{\mathrm{max}} \frac{8b}{\delta} + 1\right)$ we have that
\be \label{eq_d11}
\|\nabla_{y} f(x^\star, y^\star)\| \leq  (1- \eta L)  \epsilon_{i^\star}.
\ee

\noindent
Moreover, with probability at least $1-\frac{\omega}{100} -2\nu \times \left(r_{\mathrm{max}} \frac{8b}{ \delta} + 1\right)$ we have that
\be \label{eq_d12}
\mathbb{P}_{\Delta \sim Q_{x^\star,y^\star}}\bigg(\mathcal{L}_{\epsilon_{i^\star}}(x^\star + \Delta,y^\star) \leq \mathcal{L}_{\epsilon_{i^\star}}(x^\star, y^\star) - \frac{1}{2}\delta \bigg | x^\star, y^\star \bigg) \leq \frac{1}{2} \omega.
\ee
and that
\be \label{eq_d22}
\frac{\epsilon}{2} \leq \epsilon_{i^\star} \leq \epsilon.
\ee
\end{lemma}

\begin{proof}

First, we note that $(x^\star, y^\star) = (x_\ai, y_\ai)$ for all $\ai \in \{i^{\star}, \ldots, i^{\star} + r_{\mathrm{max}}\}$, and that Algorithm \ref{alg:Informal} stops after exactly $i^{\star} + r_{\mathrm{max}}$ iterations of the ``While" loop in Algorithm \ref{alg:Informal}.

Let $\mathsf{H}_i$ be the ``bad" event that, when Algorithm \ref{alg:InnerMaxLoop}  is called during the $i$th iteration of the ``While" loop in Algorithm \ref{alg:Informal}, the path traced by Algorithm \ref{alg:InnerMaxLoop} is not an $\epsilon_i$-increasing path.  Then, 
by Proposition \ref{Lemma_Greedypath} we have that
\be \label{eq_d7}
\mathbb{P}(\mathsf{H}_i) \leq 2 \nu  \mathcal{J}.
\ee

\noindent
Let $\mathsf{K}_i$ be the ``bad" event that  $\|G_{y}(x_i, y_i) - \nabla_{y} f(x_i, y_i)\| \geq \frac{\hat{\epsilon}_1}{10}$.  Then by Propositions \ref{Prop_Azuma} and \ref{thm_InnerRuntime} we have that
\be \label{eq_d8}
\mathbb{P}(\mathsf{K}_i) \leq  2\nu \mathcal{J}.
\ee

\noindent
Whenever $\mathsf{K}_i^c$ occurs we have that
\be \label{eq_d4}
\|\nabla_{y} f(x_i, y_i)\| &\leq \|G_{y}(x_i, y_i)\| + \|G_{y}(x_i, y_i) - \nabla_{y} f(x_i, y_i)\|\\
& \leq  (1- 2\eta L) \epsilon_{i} +  \|G_{y}(x_i, y_i) - \nabla_{y} f(x_i, y_i)\|\\
&\leq (1- 2\eta L) \epsilon_{i} + \frac{\hat{\epsilon}_1}{10}\\
&\leq  (1- \eta L) \epsilon_{i},
\ee
where the second Inequality holds by Line \ref{DiscADAMStart} of Algorithm \ref{alg:InnerMaxLoop}, and the last inequality holds since $\frac{\hat{\epsilon}_1}{10} \leq \eta L$.
Therefore, Inequalities \eqref{eq_d8} and \eqref{eq_d4} together with Proposition \ref{thm_OuterRuntime} imply that
\be
\|\nabla_{y} f(x^\star, y^\star)\| \leq  (1- \eta L) \epsilon_{i^\star}
\ee
with probability at least $1- 2\nu \mathcal{J} \mathcal{I} - 2\nu \times \left(r_{\mathrm{max}} \frac{8b}{\delta} + 1\right)$.  This proves Inequality \eqref{eq_d11}.

Inequality \eqref{eq_d4} also implies that, whenever  $\mathsf{K}_i^c$ occurs, the set $P_{\epsilon_{i}}(x_i, y_i)$ of endpoints of $\epsilon_{i}$-increasing paths with initial point $y_i$ (and $x$-value $x_i$) consists only of the single point $y_i$.  Therefore, we have that
\be \label{eq_d5}
\mathcal{L}_{\epsilon_{i}}(x_i ,y_i) = f(x_i ,y_i)
\ee
whenever $\mathsf{K}_i^c$ occurs.
Moreover, whenver $\mathsf{H}_i^c$ occurs we have that $\mathcal{Y}_{\ai+1}$ is the endpoint of an $\epsilon_{i}$-increasing path with starting point $(x_i + \Delta_i, y_i)$.  Now, $\mathcal{L}_{\epsilon_{i}}(x_i  + \Delta_i, y_i)$ is the supremum of the value of $f$ at the endpoints of all $\epsilon_{i}$-increasing paths with starting point $(x_i + \Delta_i, y_i)$.  Therefore, we must have that
\be \label{eq_d6}
\mathcal{L}_{\epsilon_{i}}(x_i  + \Delta_i,y_i) \geq f(x_i + \Delta_i, \mathcal{Y}_{\ai+1})
\ee
whenever  $\mathsf{H}_i^c$ occurs.
Therefore,
\be \label{eq_d9}
 &\mathbb{P}_{\Delta \sim Q_{x_i,y_i}} \bigg(\mathcal{L}_{\epsilon_{i}}(x_i + \Delta,y_i) > \mathcal{L}_{\epsilon_{i}}(x_i, y_i) - \frac{1}{2}\delta \bigg | x_i, y_i \bigg)\\
  &\stackrel{\textrm{Eq.}  \ref{eq_d5}, \ref{eq_d6}} \geq    \mathbb{P}_{\Delta \sim Q_{x_i,y_i}}\bigg(f(x_i + \Delta, \mathcal{Y}_{\ai+1}) > f(x_i ,y_i) - \frac{1}{2} \delta \bigg |  x_i, y_i \bigg) - \mathbb{P}(\mathsf{H}_i) - \mathbb{P}(\mathsf{K}_i)\\
&\stackrel{\textrm{Prop.}  \ref{Prop_Azuma}} \geq   \mathbb{P}_{\Delta \sim Q_{x_i,y_i}}\bigg(F(x_i + \Delta, \mathcal{Y}_{\ai+1}) > F(x_i ,y_i) - \frac{1}{4} \delta \bigg |  x_i, y_i \bigg) - 2 \nu - \mathbb{P}(\mathsf{H}_i) - \mathbb{P}(\mathsf{K}_i)\\
& \geq \mathbb{P}\bigg( \mathsf{Accept}_i = \mathsf{False}  \big | x_i, y_i \bigg) - 2 \nu - \mathbb{P}(\mathsf{H}_i) - \mathbb{P}(\mathsf{K}_i)\\
&\stackrel{\textrm{Eq.}  \ref{eq_d7}, \ref{eq_d8}} \geq   \mathbb{P}\bigg( \mathsf{Accept}_i = \mathsf{False}  \big | x_i, y_i \bigg) - 2 \nu -  2 \nu  \mathcal{J} -    2 \nu  \mathcal{J}, \,\,\, \forall i \leq \mathcal{I},
\ee
where the second inequality holds by Proposition \ref{Prop_Azuma}, since $\frac{\hat{\epsilon}_1}{10} \leq \frac{\delta}{8}$.
Define
$$p_i := \mathbb{P}_{\Delta \sim Q_{x_i,y_i}}\left(\mathcal{L}_{\epsilon_{i}}(x_i + \Delta,y_i) > \mathcal{L}_{\epsilon_{i}}(x_i, y_i) - \frac{1}{2}\delta \bigg | x_i, y_i \right)$$ for every $i\in \mathbb{N}$. Then Inequality \eqref{eq_d9} implies that

\be \label{eq_d10}
\mathbb{P}\left( \mathsf{Accept}_i = \mathsf{False}  \big | x_i, y_i \right) &\leq p_i + \nu(4 \mathcal{J} + 2) \leq p_i + \frac{1}{8}\omega  \qquad \forall i \leq \mathcal{I},
\ee
since $\nu \leq \omega/(32 \mathcal{J} + 16)$.
We now consider what happens for indices $i$ for which $p_i \leq 1- \omega/2$.  Since $(x_{i+s}, y_{i+s}) = (x_i, y_i)$ whenever $\mathsf{Accept}_{i+k} = \mathsf{False}$ for all $0\leq k \leq s$, we have by Inequality \eqref{eq_d10} that

\be
\mathbb{P}&\left( \cap_{s=0}^{r_{\mathrm{max}}} \{ \mathsf{Accept}_{i+s} = \mathsf{False}\}   \bigg | p_i \leq 1- \frac{1}{2}\omega \right) \leq \left(1- \frac{1}{4}\omega\right)^{r_{\mathrm{max}}} \leq \frac{\omega}{100 \mathcal{I}} \qquad \forall i \leq \mathcal{I}  - r_{\mathrm{max}}
\ee
since $$r_{\mathrm{max}} \geq \frac{4}{\omega} \log\left(\frac{100 \mathcal{I}}{\omega}\right).$$
Therefore, with probability at least $1-\frac{\omega}{100 \mathcal{I}} \times \mathcal{I} = 1-\frac{\omega}{100}$, we have that the event  $\cap_{s=0}^{r_{\mathrm{max}}} \{ \mathsf{Accept}_{i+s} = \mathsf{False} \}$ does not occur for any $i \leq \mathcal{I} - r_{\mathrm{max}}$ for which $p_i \leq 1- \frac{1}{2}\omega.$
Recall from Proposition \ref{thm_OuterRuntime} that Algorithm \ref{alg:Informal} terminates in at most $\mathcal{I}$ iterations of its ``While" loop,  with probability at least $1-2\nu \times \left(r_{\mathrm{max}} \frac{8b}{ \delta} + 1\right)$.
 Therefore,
 \be \label{eq_d13}
 &\mathbb{P}\left( p_{i^\star} > 1- \frac{1}{2}\omega \right) \geq 1-\frac{\omega}{100} -2\nu \times \left(r_{\mathrm{max}} \frac{8b}{ \delta} + 1\right).
\ee

\noindent
In other words, by the definition of  $p_{i^\star}$, Inequality \eqref{eq_d13} implies that with probability at least $1-\frac{\omega}{100} -2\nu \times \left(r_{\mathrm{max}} \frac{8b}{ \delta} + 1\right)$,  the point $(x^\star, y^\star)$ is such that
\be
\mathbb{P}_{\Delta \sim Q_{x^\star,y^\star}}&\bigg(\mathcal{L}_{\epsilon_{i^\star}}(x^\star + \Delta,y^\star) \leq \mathcal{L}_{\epsilon_{i^\star}}(x^\star, y^\star) - \frac{1}{2}\delta \bigg | x^\star, y^\star \bigg) \leq \frac{1}{2} \omega.
\ee
 This completes the proof of inequality \eqref{eq_d12}.
Finally we note that when Algorithm \ref{alg:Informal} terminates in at most $\mathcal{I}$ iterations of its ``While" loop, we have
\begin{equation}
    \epsilon_{i^\star} = \epsilon_0 \left(\frac{1}{1-2\eta L}\right)^{2i^\star}  \leq \epsilon_0 \left(\frac{1}{1-2\eta L}\right)^{2\mathcal{I}} \leq \epsilon,
\end{equation}
since $\eta \leq \frac{1}{8 L \mathcal{I}}$.
This completes the proof of Inequality \eqref{eq_d22}.
\end{proof}

\noindent
We can now complete the proof of the main theorem:

\begin{proof}[Proof of Theorem \ref{thm:GreedyMinimax-main}]
First, by Lemma \ref{lemma_polytime}, with probability at least $1- 3 \nu  \mathcal{J}  \mathcal{I} \geq 99/100$, our algorithm converges to some point $(x^\star, y^\star)$ after at most $$\left(\tau_1 \log\left(\frac{r_{\mathrm{max}}}{\nu}\right) + 4r_{\mathrm{max}} \frac{b}{\delta} + 1\right) \times (\mathcal{J} \times \mathfrak{b}_{y} + \mathfrak{b}_{\mathrm{0}} + \mathfrak{b}_{x})$$ gradient, function, and sampling oracle evaluations, which is polynomial in $b, L_1, {L}, 1/\epsilon,1/\delta, 1/\omega,$ and does not depend on the dimension $d$.   
  By Lemma \ref{lemma_greedyMinimax}, if we set $\epsilon^\star= \epsilon_{i^\star}$, we have that Inequalities \eqref{eq_approx_local_equilibrium_y} and \eqref{eq_approx_local_equilibrium_x} hold for parameters $\epsilon^\star \in\left[ \frac{1}{2}\epsilon, \epsilon\right]$, $\delta, \omega$ and distribution $Q$, with probability at least $$1- 2\nu \mathcal{J} \mathcal{I} - 2\nu \times \left(r_{\mathrm{max}} \frac{8b}{ \delta} + 1\right) \geq \frac{19}{20},$$  since $\nu \leq  \frac{1}{20}\left(2 \mathcal{J} \mathcal{I} + 2 \times \left(r_{\mathrm{max}} \frac{8b}{\delta} + 1\right)\right)^{-1}$.
\end{proof}

\section{Conclusion and Future Directions} \label{sec_conclusion}

We introduce a new variant of the min-max optimization framework, and provide a gradient-based algorithm with efficient convergence guarantees to an equilibrium for this framework,
for nonconvex-nonconcave losses and from any initial point.
Empirically, we observe our algorithm converges on many challenging test functions
and shows improved stability 
when training GANs.

 While we show our algorithm runs in time polynomial in $b,L$, and independent of dimension $d$, we do not believe our bounds are tight and it would be interesting to show the run-time is linear in $b, L$.
Moreover, while our guarantees hold for any distribution $Q$, it would be interesting to see if a specialized analysis 
for adaptively preconditioned distributions can lead to improved bounds.
Our framework can also be extended to more general settings;
e.g., to multi-agent minimization problems arising in meta-learning \citep{finn2017model}.

 \section*{Acknowledgments} This research was supported in part by NSF CCF-1908347, NSF CCF-2112665, and NSF CCF-2104528 grants and an AWS ML research award.

\bibliographystyle{plainnat}
\bibliography{GAN}

\newpage
\appendix
\onecolumn

\section{Equilibrium in the Strongly Convex-Strongly Concave Setting} \label{sec_strongly_convex}

In this section we show that, if $f$ is strongly-convex strongly-concave with Lipschitz gradient, if we choose $Q$ to be the distribution of (either deterministic or stochastic) gradients with mean $-\nabla_x f$, then an $(\epsilon, \delta, \omega, Q)$-equilibrium corresponds to an ``aproximate" global min-max point (Theorem \ref{assumption_x_stochastic_gradient}).
We then show that this fact, together with the proof of our main result, implies that when our algorithm is applied to $\alpha$-strongly-convex $\alpha$-strongly-concave objective functions $f$ with $L$-Lipschitz gradient, it finds an ``approximate" global min-max point with duality gap $O(\eps)$ in a number of gradient evaluations that is polynomial in $L, \frac{1}{\eps}, \frac{1}{\alpha}, L, D$ (Corollary \ref{cor_runtime_strongly_convex}).

For any $L>0$, we say that a function $\psi: \mathbb{R}^d \rightarrow \mathbb{R}$ has $L$-Lipschitz gradient (equivalently, ``$L$-smooth") if for any $x, \theta \in \mathbb{R}^d$, $$\|\nabla \psi(x) -  \nabla \psi( \theta ) \| \leq L \times || x- \theta  \|.$$

\noindent
And for any $\alpha >0$ we say that $\psi$ is $\alpha$-strongly convex if for any $x, \theta \in \mathbb{R}^d$, $$(\nabla \psi(x) - \nabla \psi(\theta))^\top (x-\theta) \geq \alpha \|x-\theta\|^2.$$

\noindent
Similarly, we say that $\psi$ is $\alpha$-strongly concave if $-\psi$ is $\alpha$ strongly-convex.

In the following, we assume that the proposal distribution is a stochastic gradient for $-\nabla_x f$ with some variance $\sigma^2\geq 0$; for simplicity we set $\sigma^2 = 0$ in Corollary \ref{cor_runtime_strongly_convex}, although this is not strictly necessary.

\begin{assumption} \label{assumption_x_stochastic_gradient}$(\sigma\geq 0)$
For every $(x,y) \in \mathbb{R}^d \times \mathbb{R}^d$, the distribution $Q_{x,y}$ satisfies $\mathbb{E}_{\Delta \sim Q_{x,y}}[\Delta] = -\frac{1}{2L}  \nabla_x f(x,y)$ and $\mathbb{E}_{\Delta \sim Q_{x,y}}[\|-\frac{1}{2L} \nabla_x f(x,y)- \Delta \|^2] \leq \sigma^2$.
\end{assumption}

\begin{theorem}\label{thm_strongly_convex}
Suppose that $f: \mathbb{R}^d \times \mathbb{R}^d \rightarrow \mathbb{R}$ is $\alpha$-strongly convex in $x$ and $\alpha$-strongly concave in $y$, with $L$-Lipschitz gradient in both variables for some $L \geq \alpha >0$.
Then, for any $\eps, \delta, \omega>0$ with $\omega \leq \frac{1}{2}$, and any proposal distribution $Q_{x,y}$ satisfying Assumption \eqref{assumption_x_stochastic_gradient} for some $\sigma\geq 0$, we have that any point $(x^\star, y^\star)$ which is an $(\eps, \delta, \omega, Q)$-approximate local equilibrium of $f$ also has duality gap satisfying
\begin{equation}
\max_{y \in \mathbb{R}^d} f(x^\star, y) -    \min_{x \in \mathbb{R}^d} f(x, y^\star) \leq \frac{L\eps^2}{2\alpha^2} + 
    \frac{L^3}{\alpha^2}
   \left(\sqrt{\delta + L \left(2\frac{\eps}{\alpha} +  \frac{1}{\alpha}(L\frac{\eps}{\alpha} + 2 \sigma)\right) + L \frac{\eps^2}{2 \alpha^2} +L \frac{\eps}{\alpha}} +  L \frac{\eps}{\alpha}\right)^2. \nonumber
\end{equation}
\end{theorem}

\noindent
Before proving Theorem \ref{thm_strongly_convex}, we first show a number of Lemmas.
In the following we set $\mu = \frac{1}{L}$.

\begin{lemma}\label{prop_strongly_convex_1}
For any $x, w \in \mathbb{R}^d$,
\begin{equation}
    \|\mathrm{argmax}_{z \in \mathbb{R}^d} f(x,z) -  \mathrm{argmax}_{z \in \mathbb{R}^d} f(\theta,z)\| \leq  \frac{L}{\alpha}\|\theta-x\| \nonumber
\end{equation}
\end{lemma}

\begin{proof}
Since $f(x,\cdot)$ is concave, we have that a point $z$ is a global maximum if and only if $\nabla_y f(x,z) = 0$. 
Since $f(x,\cdot)$ is $\alpha$-strongly concave for $\alpha>0$, this global maximum point is unique.

Let $z^\star$ be the global maximum of $f(x,\cdot)$, and let $\zeta^\star$ be the global maximum of $f(\theta, \cdot)$.
Then, since $f(x,\cdot)$ is $\alpha$-strongly concave, $\|\nabla_y f(x,z)\| \geq \alpha \|z - z^\star\|$ for all $z \in \mathbb{R}^d$.
Moreover, since $f$ is $L$-smooth, we also have that   $\|\nabla_y f(x,z) - \nabla_y f(\theta,z)\| \leq L\|\theta-x\|$ for every $x,\theta, z \in \mathbb{R}^d$.
Therefore, 

\begin{align} \label{eq_strongly_convex1}
\|\nabla_y f(\theta,z)\| &\geq \|\nabla_y f(x,z) \| - \|\nabla_y f(\theta,z) - \nabla_y f(x,z)\| \nonumber\\
&\geq  \|\nabla_y f(x,z) \|  - L\|\theta-x\| \nonumber\\
&\geq \alpha \|z - z^\star\| - L\|\theta-x\|.
\end{align}

\noindent
Since $\alpha \|z - z^\star\| - L\|\theta-x\|>0$ for any $z \in \mathbb{R}^d$ such that $\|z-z^\star\| > \frac{L}{\alpha}\|\theta-x\|$, \eqref{eq_strongly_convex1} implies that 
\begin{equation}
    \|\mathrm{argmax}_{z \in \mathbb{R}^d} f(x,z) -  \mathrm{argmax}_{z \in \mathbb{R}^d} f(\theta,z)\| \leq \frac{L}{\alpha}\|\theta-x\| \nonumber
\end{equation}

\end{proof}

\begin{lemma}\label{prop_strongly_convex_2}
If $(x,w)\in \mathbb{R}^d\times \mathbb{R}^d$  are such that $\|\nabla_y f(x,w)\| \leq \epsilon$ then $$\|w -\mathrm{argmax}_{z \in \mathbb{R}^d} f(x, z) \| \leq \frac{\eps}{\alpha}$$.
\end{lemma}

\begin{proof}
Let  $z^\star$ be the unique global maximum point of $f(x,\cdot)$.
Since $f(x, \cdot)$ is $\alpha$-strongly concave, we have that
 \begin{align} \label{eq_strongly_convex2}
 \alpha \|w - z^\star\| &\leq \|\nabla_y f(x,w)\|  \leq \epsilon
\end{align}

\noindent
Therefore, \eqref{eq_strongly_convex2} implies that

\begin{align}
\|w -\mathrm{argmax}_{z \in \mathbb{R}^d} f(x, z) \| = \|w - z^\star\| \leq \frac{\eps}{\alpha}. \nonumber
\end{align}

\end{proof}

\begin{lemma}\label{prop_strongly_convex_3}
For any $x, \theta \in \mathbb{R}^d$,  $$\|\mathcal{L}_\eps(x,y) - \mathcal{L}_\eps(\theta,y)| \leq L \left(2\frac{\eps}{\alpha} +  \frac{L}{\alpha}\|\theta- x\|\right)$$
\end{lemma}

\begin{proof}

Denote by $\hat{P}_\eps(x,y) \subseteq P_\eps(x,y)$ the collection of endpoints of $\eps$-greedy paths where the endpoint $z$ of the path satisfies $\nabla_y f(x,z) = \epsilon$.
Since $f$ is smooth, we have that $\sup_{z\in P_\eps(x,y)} f(x,z) = \sup_{z\in \hat{P}_\eps(x,y)} f(x,z)$ (this is true since, if the endpoint $z$ of an  $\eps$-greedy path does not satisfy $\|\nabla_y f(x,z)\| = \epsilon$,  then the $\eps$-greedy path can be extended to achieve a higher value of $f$).
Thus, we have that
\begin{align*}
    |\mathcal{L}_\eps(x,y) - \mathcal{L}_\eps(\theta,y)| &= |\sup_{z\in P_\eps(x,y)} f(x,z) - \sup_{w\in P_\eps(\theta,y)} f(\theta,w)| \\
     &= |\sup_{z\in \hat{P}_\eps(x,y)} f(x,z) - \sup_{w\in \hat{P}_\eps(\theta,y)} f(\theta,w)| \nonumber\\
    &\leq \sup_{z\in \hat{P}_\eps(x,y), w\in \hat{P}_\eps(\theta,y)} | f(x,z) - f(\theta,w)| \nonumber\\
    &\leq \sup_{z\in \hat{P}_\eps(x,y), w\in \hat{P}_\eps(\theta,y)} L \times \|w-z\| \nonumber\\
    &\leq \sup_{z\in \hat{P}_\eps(x,y), w\in \hat{P}_\eps(\theta,y)} L \times \left(\|w - \mathrm{argmax}_{\zeta \in \mathbb{R}^d} f(\theta,\zeta)\|  \right. \nonumber\\
    & \qquad  \left. +  \|\mathrm{argmax}_{\zeta \in \mathbb{R}^d} f(\theta,\zeta)- \mathrm{argmax}_{\zeta \in \mathbb{R}^d} f(x,\zeta)\| + \|\mathrm{argmax}_{\zeta \in \mathbb{R}^d} f(x,\zeta) -z\|\right) \nonumber\\
    & \stackrel{\textrm{Lemmas \ref{prop_strongly_convex_1}, \ref{prop_strongly_convex_2}}} \leq   L\times \left(\frac{\eps}{\alpha} + \frac{L}{\alpha}\|\theta-x\| +\frac{\eps}{\alpha}\right) \nonumber\\
    &= L\times \left(2\frac{\eps}{\alpha} + \frac{L}{\alpha}\|\theta-x\|\right), \nonumber
\end{align*}
where the last inequality holds by Lemma \ref{prop_strongly_convex_1}, and also by Lemma \ref{prop_strongly_convex_2} because $\|\nabla_y f(x, z)\| = \eps$ whenever  $z\in \hat{P}_\eps(x,y)$ and $\|\nabla_y f(\theta, w)\| = \eps$ whenever $w\in P_\eps(\theta,y)$.
\end{proof}

\begin{lemma}\label{prop_strongly_convex_4}
For any $x,y \in \mathbb{R}^d$ we have $$|\mathcal{L}_\eps(x,y) - \mathcal{L}_0(x,y)| \leq L \frac{\eps^2}{2 \alpha^2}$$
\end{lemma}

\begin{proof}

\begin{align*}
|\mathcal{L}_\eps(x,y) - \mathcal{L}_0(x,y)|  &= |\sup_{z \in \hat{P}(x,y)} f(x,z) - f(x, \mathrm{argmax}_{z\in \mathbb{R}^d} f(x,z))|\\
&\leq  \sup_{z \in \hat{P}(x,y)}|f(x,z) - f(x, \mathrm{argmax}_{z\in \mathbb{R}^d} f(x,z))|\nonumber\\
&\leq \sup_{z \in \hat{P}(x,y)} \frac{L}{2}\|z - \mathrm{argmax}_{z\in \mathbb{R}^d} f(x,z)\|^2 \nonumber\\
&\stackrel{\textrm{Lemmas  \ref{prop_strongly_convex_2}}} \leq   \frac{L}{2} \left(\frac{\eps}{\alpha}\right)^2, \nonumber
\end{align*}
where the last inequality holds by Lemma \ref{prop_strongly_convex_2} because $\|\nabla_y f(x, z)\| = \eps$ whenever  $z\in \hat{P}_\eps(x,y)$.

\end{proof}

\begin{lemma}\label{prop_strongly_convex_5}
$\mathcal{L}_0(x^\star, y^\star)$ is differentiable and
$$\| \nabla_x f(x^\star, y^\star) - \nabla_x \mathcal{L}_0(x^\star, y^\star) \| \leq L \frac{\eps}{\alpha}.$$
\end{lemma}

\begin{proof}

Since $f(x,\cdot)$ is concave, we have that $\mathcal{L}_0(x^\star, y^\star) = \max_{z\in \mathbb{R}^d} f(x^\star, z)$.
Moreover, by strong concavity $f(x^\star, \cdot)$ has a unique maximizer $\mathrm{argmax}_{z\in \mathbb{R}^d} f(x^\star, z)$.
Thus, by Danskin's Theorem \cite{danskin1966theory}, we have that $ \mathcal{L}_0(x^\star, \cdot)$ is differentiable and that 
\begin{equation}\label{eq_strongly_convex_3}
    \nabla_x  \mathcal{L}_0(x^\star, y^\star) = \nabla_x f(x^\star, \mathrm{argmax}_{z\in \mathbb{R}^d} f(x^\star, z)).
    \end{equation}

\noindent
Therefore, since $f$ is $L$-smooth,
\begin{align*}
    \| \nabla_x f(x^\star, y^\star) - \nabla_x \mathcal{L}_0(x^\star, y^\star) \| &\stackrel{\textrm{Eq.  \ref{eq_strongly_convex_3}}} =  \| \nabla_x f(x^\star, y^\star) -  \nabla_x f(x^\star, \mathrm{argmax}_{z\in \mathbb{R}^d} f(x^\star, z)) \| \\
&\leq    L \times \|y^\star - \mathrm{argmax}_{z\in \mathbb{R}^d} f(x^\star, z) \| \nonumber\\
&\leq  L \times \frac{\eps}{\alpha}, \nonumber
\end{align*}
where the last inequality holds by Lemma \ref{prop_strongly_convex_2} since $\|\nabla_y f(x^\star, y^\star)\| \leq \eps$.

\end{proof}

\begin{proof}[Proof of Theorem \ref{thm_strongly_convex}]

Since $(x^\star, y^\star)$ is an $(\eps, \delta, \omega, Q)$-approximate local equilibrium of $f$, we have that,
\begin{equation}
    \|\nabla_y f(x^\star, y^\star)\| \leq \eps, \nonumber
\end{equation}
and that, with probability at least $1-\omega$,
\begin{equation}
    \mathcal{L}_\eps(x^\star - \Delta, y^\star) \geq f(x^\star, y^\star) - \delta, \nonumber
\end{equation}
where $\Delta \sim Q_{x^\star, y^\star}$.
Thus, since by Assumption \ref{assumption_x_stochastic_gradient} $\mathbb{E}_{\Delta \sim Q_{x,y}}[\Delta] = \mu \nabla_x f(x,y)$ and $\mathbb{E}_{\Delta \sim Q_{x,y}}[\|\Delta -\mu \nabla_x\|^2] \leq \sigma^2$, by Chebyshev's inequality, we have that, with probability at least $1-\omega- \frac{1}{4}$,
\begin{equation}\label{eq_strongly_convex_4}
    \mathcal{L}_\eps(x^\star - \mu \nabla_x f(x^\star, y^\star) + \nu, y^\star) \geq f(x^\star, y^\star) - \delta,
\end{equation}
for some $\nu \in \mathbb{R}^d$ such that $\|\nu\|\leq 2\sigma$.

Since $\omega\leq \frac{1}{2}$, we have $1-\omega- \frac{1}{4}\geq \frac{1}{4}$.  Thus, \eqref{eq_strongly_convex_4} holds with probability at least $\frac{1}{4}$.
Since \eqref{eq_strongly_convex_4}  holds with probability at least $\frac{1}{4}$ yet contains no random variables, it must also hold with probability $1$.
Therefore, by plugging in Lemma \ref{prop_strongly_convex_4} to the LHS of  \eqref{eq_strongly_convex_4}  and applying Lemma \ref{prop_strongly_convex_2} together with the fact that $f$ is $L$-smooth to the RHS of  \eqref{eq_strongly_convex_4}, we have that
\begin{align}\label{eq_strongly_convex_5}
        \mathcal{L}_0(x^\star - \mu \nabla_x f(x^\star, y^\star) + \nu, y^\star) + L \frac{\eps^2}{2 \alpha^2} \geq \mathcal{L}_0(x^\star, y^\star) - L \frac{\eps}{\alpha}- \delta.
\end{align}

\noindent
Therefore, applying Lemma \ref{prop_strongly_convex_5} to  \eqref{eq_strongly_convex_5} we get that
\begin{align}\label{eq_strongly_convex_6}
        \mathcal{L}_0(x^\star - \mu \nabla_x \mathcal{L}_0(x^\star, y^\star) + v, y^\star) + L \frac{\eps^2}{2 \alpha^2} \geq \mathcal{L}_0(x^\star, y^\star) - L \frac{\eps}{\alpha}- \delta,
\end{align}
for some $v \in \mathbb{R}^d$ such that $\|v\|\leq L\frac{\eps}{\alpha} + 2 \sigma$.
Therefore, plugging in Lemma \ref{prop_strongly_convex_5} to the LHS of \eqref{eq_strongly_convex_6}, we get that
\begin{align}\label{eq_strongly_convex_6b}
        \mathcal{L}_0(x^\star - \mu \nabla_x \mathcal{L}_0(x^\star, y^\star), y^\star) + L \left(2\frac{\eps}{\alpha} +  \mu \frac{L}{\alpha}\left(L\frac{\eps}{\alpha} + 2 \sigma\right)\right) + L \frac{\eps^2}{2 \alpha^2} \geq \mathcal{L}_0(x^\star, y^\star) - L \frac{\eps}{\alpha}- \delta,
\end{align}
since $\|v\|\leq L\frac{\eps}{\alpha} + 2 \sigma$.
We also have that
\begin{align}\label{eq_strongly_convex_7}
\mathcal{L}_0(x^\star - \mu \nabla_x \mathcal{L}_0(x^\star, y^\star), y^\star) = \mathcal{L}_0(x^\star, y^\star) -(\nabla_x \mathcal{L}_0(x^\star, y^\star) + u)^\top \mu \nabla_x \mathcal{L}_0(x^\star, y^\star),
\end{align}
for some $u \in \mathbb{R}^d$ such that $\|u\| \leq L \| \mu \nabla_x \mathcal{L}_0(x^\star, y^\star)\|$, since $f$ is $L$-smooth.
Therefore \eqref{eq_strongly_convex_7} implies that
\begin{align}\label{eq_strongly_convex_8}
\mathcal{L}_0(x^\star - \mu \nabla_x \mathcal{L}_0(x^\star, y^\star), y^\star) \leq \mathcal{L}_0(x^\star, y^\star) - (\mu-\mu^2 L)\|\nabla_x \mathcal{L}_0(x^\star, y^\star)\|^2,
\end{align}

\noindent
Plugging \eqref{eq_strongly_convex_8} into \eqref{eq_strongly_convex_6b}, we get that (since $\mu= \frac{1}{2L}$ implies that $\mu - \mu^2 L >0$),
\begin{align}\label{eq_strongly_convex_9}
        \|\nabla_x \mathcal{L}_0(x^\star, y^\star)\| \leq \frac{1}{\mu-\mu^2 L}\sqrt{\delta + L \left(2\frac{\eps}{\alpha} + \mu \frac{L}{\alpha}\left(L\frac{\eps}{\alpha} + 2 \sigma\right)\right) + L \frac{\eps^2}{2 \alpha^2} +L \frac{\eps}{\alpha}}
\end{align}

\noindent
But from \eqref{eq_strongly_convex_3} we have that
\begin{equation}\label{eq_strongly_convex_10}
    \nabla_x  \mathcal{L}_0(x^\star, y^\star) = \nabla_x f(x^\star, \mathrm{argmax}_{z\in \mathbb{R}^d} f(x^\star, z)).
    \end{equation}

\noindent
Thus,
\begin{align}\label{eq_strongly_convex_11}
    \|\nabla_x  \mathcal{L}_0(x^\star, y^\star) - \nabla_x  f(x^\star, y^\star)\| &\stackrel{\textrm{Eq.  \ref{eq_strongly_convex_10}}} =   \|\nabla_x f(x^\star, \mathrm{argmax}_{z\in \mathbb{R}^d} f(x^\star, z))  - \nabla_x  f(x^\star, y^\star)\|\nonumber\\
    &\leq L \|y^\star - \mathrm{argmax}_{z\in \mathbb{R}^d} f(x^\star, z)\| \nonumber\\
 &\stackrel{\textrm{Lemma  \ref{prop_strongly_convex_2}}} \leq    L \frac{\eps}{\alpha},
    \end{align}
where the first inequality holds since $f$ is $L$-smooth, and the second inequality holds by Lemma  \ref{prop_strongly_convex_2} since $\|\nabla_y f(x^\star, y^\star)\|\leq \eps$.

Plugging in \eqref{eq_strongly_convex_11} into \eqref{eq_strongly_convex_9}, we get
\begin{align}\label{eq_strongly_convex_13}
        \|\nabla_x f(x^\star, y^\star)\| \leq \frac{1}{\mu-\mu^2}\sqrt{\delta + L \left(2\frac{\eps}{\alpha} + \mu \frac{L}{\alpha}\left(L\frac{\eps}{\alpha} + 2 \sigma\right)\right) + L \frac{\eps^2}{2 \alpha^2} +L \frac{\eps}{\alpha}} +  L \frac{\eps}{\alpha}.
\end{align}

\noindent
Now, since $f(\cdot, y^\star)$ is $\alpha$-strongly convex, by Lemma \ref{prop_strongly_convex_2} (applied to $-f$ instead of $f$), we have
\begin{align} \label{eq_strongly_convex_14}
    \|x^\star - \mathrm{argmin}_{\theta \in \mathbb{R}^d} f(\theta, y^\star)\|  \leq \frac{\|\nabla_x f(x^\star, y^\star)\|}{\alpha}.
\end{align}

\noindent
Since $f$ is $L$-smooth, \eqref{eq_strongly_convex_14} implies that
\begin{align} \label{eq_strongly_convex_15}
    f(x^\star,y^\star) - \min_{\theta \in \mathbb{R}^d} f(\theta, y^\star) &\leq  L \times  \frac{1}{2}\|x^\star - \mathrm{argmin}_{\theta \in \mathbb{R}^d} f(\theta, y^\star)\|^2 \\
    &\stackrel{\textrm{Eq. \ref{eq_strongly_convex_14}}} \leq  \frac{L}{2} \times \left(\frac{\|\nabla_x f(x^\star, y^\star)\|}{\alpha}\right)^2. \nonumber
\end{align}

\noindent
Moreover, since $f(x^\star, \cdot)$ is $\alpha$-strongly concave and $\|\nabla_y f(x^\star, y^\star)\|\leq \eps$ we have  by Lemma \ref{prop_strongly_convex_2} that 
 \begin{align} \label{eq_strongly_convex_16}
 \|y^\star -\mathrm{argmax}_{z \in \mathbb{R}^d} f(x^\star, z) \| \leq \frac{\eps}{\alpha}.
 \end{align}
 Since $f$ is $L$-smooth, \eqref{eq_strongly_convex_16} implies that
 \begin{align} \label{eq_strongly_convex_17}
    \max_{z \in \mathbb{R}^d} f(x^\star, z) - f(x^\star,y^\star) &\leq  L \times  \frac{1}{2}\|y^\star - \mathrm{argmax}_{z \in \mathbb{R}^d} f(x^\star, z)\|^2 \\
    &\stackrel{\textrm{Eq. \ref{eq_strongly_convex_16}}} \leq  \frac{L}{2} \times \left(\frac{\eps}{\alpha}\right)^2. \nonumber
\end{align}

\noindent
Thus, adding \eqref{eq_strongly_convex_15} to \eqref{eq_strongly_convex_17}, and plugging in \eqref{eq_strongly_convex_13}, we get that
\begin{align*} 
     \max_{z \in \mathbb{R}^d} &f(x^\star, z)  - \min_{\theta \in \mathbb{R}^d} f(\theta, y^\star) \leq \frac{L}{2} \times \left(\frac{\|\nabla_x f(x^\star, y^\star)\|}{\alpha}\right)^2 + \frac{L}{2} \times \left(\frac{\eps}{\alpha}\right)^2 \\
    & \stackrel{\textrm{Eq. \ref{eq_strongly_convex_13}}} \leq  
    \frac{L\eps^2}{2\alpha^2} + 
    \frac{L}{2\alpha^2}
   \left(\frac{1}{\mu-\mu^2}\sqrt{\delta + L \left(2\frac{\eps}{\alpha} + \mu \frac{L}{\alpha}\left(L\frac{\eps}{\alpha} + 2 \sigma\right)\right) + L \frac{\eps^2}{2 \alpha^2} +L \frac{\eps}{\alpha}} +  L \frac{\eps}{\alpha}\right)^2. \nonumber
\end{align*}

\noindent
Since $\mu= \frac{1}{2L}$, we get that

\begin{align*}
     \max_{z \in \mathbb{R}^d} f(x^\star, z)  &- \min_{\theta \in \mathbb{R}^d} f(\theta, y^\star)\\ &\leq
    \frac{L\eps^2}{2\alpha^2} + 
    \frac{L}{2\alpha^2}
   \left(\frac{1}{\mu-\mu^2}\sqrt{\delta + L \left(2\frac{\eps}{\alpha} + \mu \frac{L}{\alpha}\left(L\frac{\eps}{\alpha} + 2 \sigma\right)\right) + L \frac{\eps^2}{2 \alpha^2} +L \frac{\eps}{\alpha}} +  L \frac{\eps}{\alpha}\right)^2 \\
   &=\frac{L\eps^2}{2\alpha^2} + 
    \frac{L^3}{\alpha^2}
   \left(\sqrt{\delta + L \left(2\frac{\eps}{\alpha} +  \frac{1}{\alpha}\left(L\frac{\eps}{\alpha} + 2 \sigma\right)\right) + L \frac{\eps^2}{2 \alpha^2} +L \frac{\eps}{\alpha}} +  L \frac{\eps}{\alpha}\right)^2. \nonumber
\end{align*}

\end{proof}

\subsection{Runtime in Strongly Convex-Strongly Concave Setting}\label{sec_strongly_convex_runtime}

Suppose that $f(x,y)$ is $L$-smooth in  $(x,y)$, $\alpha$-strongly convex in $x$ and $\alpha$-strongly concave in $y$.

In this section we assume that hyper-parameter $\tau_1$ in Algorithm \ref{alg:Informal} is set to $\tau_1 = \infty$, and that  Algorithm \ref{alg:InnerMaxLoop} takes as input exact gradients for $\nabla_y f$ (i.e., the ``stochastic" gradients have variance set to $0$).

\begin{lemma}
The function $\max_{z\in \mathbb{R}^d}f(\cdot,z)$ is $\alpha$-strongly convex.
\end{lemma}

\begin{proof}
Define the ``global max" function $\psi(x): = \max_{z \in \mathbb{R}^d} f(x,z)$ for all $x \in \mathcal{X}$.  We start by showing that the function $\psi(x)$ is $\alpha$-strongly convex.  Indeed, for any $x_1, x_2 \in \mathcal{X}$ and any $\lambda \in [0,1]$ we have 
\be
\lambda \psi( \lambda x_1 + (1-\lambda)x_2) &=  \max_{y \in \mathbb{R}^d} f(\lambda x_1 + (1-\lambda)x_2,y)\\
& \leq  \max_{y \in \mathbb{R}^d} \left[ \lambda f( x_1, y) + (1-\lambda)f(x_2,y) - \frac{1}{2} \alpha \lambda(1-\lambda)\|x_1 -x_2\|^2 \right]\\
&\leq   \lambda\left[ \max_{y \in \mathbb{R}^d}  f( x_1, y)] + (1-\lambda) [\max_{y \in \mathbb{R}^d} f(x_2,y)\right] - \frac{1}{2} \alpha \lambda(1-\lambda)\|x_1 -x_2\|^2\\
& = \lambda \psi(x_1) + (1-\lambda)\psi(x_2) - \frac{1}{2} \alpha \lambda(1-\lambda)\|x_1 -x_2\|^2,
\ee
where the first inequality holds by the $\alpha$-strong convexity of $f(\cdot, y)$.
Thus $\psi$ is $\alpha$-strongly convex.
\end{proof}

\begin{lemma}\label{lemma_bounding_ball}
Denote by $\mathsf{y}_{i,j}$ the point $\mathsf{y}_{j}$ in Algorithm \ref{alg:InnerMaxLoop} when it is called at the $i$'th iteration of Algorithm \ref{alg:Informal}.
Let $(x^\dagger, y^\dagger)$ be the global min-max point of $f$, and define $D:= \|(x_0, y_0) - (x^\dagger, y^\dagger)\|$ and  $$ \mathfrak{D}: = 2\max\left(D + \frac{LD}{\alpha} + \frac{\eps}{\alpha},  \sqrt{ \frac{LD +  \frac{L^2D}{\alpha}  + L \frac{\eps^2}{\alpha^2}}{\alpha}}, \frac{L\sqrt{D}}{\alpha\sqrt{\alpha}}, \frac{\eps\sqrt{L}}{\alpha\sqrt{\alpha}}\right).$$
Then, as long as $\eta \leq \frac{1}{L}$, at every iteration $i$ of Algorithm \ref{alg:Informal} and at every iteration $j$ of its subroutine Algorithm \ref{alg:InnerMaxLoop}, we have that $\|(x_i, y_i)  - (x^\dagger, y^\dagger) \| \leq \mathfrak{D}$ and  $\|(x_i, \mathsf{y}_{i,j})  - (x^\dagger, y^\dagger) \| \leq \mathfrak{D}$.
\end{lemma}

\begin{proof}

{\bf Bounding the distance $\|(x_1, y_1) - (x^\dagger, y^\dagger)\|$}.\\

Since at the global min-max point $(x^\dagger, y^\dagger)$ we have $\nabla_y f(x^\dagger, y^\dagger) =0$, and since $f$ is L-Lipschitz, we have that $$\|\nabla_y f(x_0,y_0)\| \leq L \|(x_0,y_0) - (x^\dagger, y^\dagger)\| \leq L D.$$

\noindent
Thus, since $f(x_0, \cdot)$ is $\alpha$-strongly concave, we have by Lemma \ref{prop_strongly_convex_2} that 
\begin{align}\label{eq_Strong_Convex_RT_5}
\|y_0 - \mathrm{argmax}_{z\in \mathbb{R}^d} f(x_0, z)\| \leq \frac{LD}{\alpha}.
\end{align}

\noindent
And since (by definition) $x_0 = x_1$,  and $\nabla_y f(x_0, y_1) = \nabla_y f(x_1, y_1) \leq \eps$, we have that (again by Lemma \ref{prop_strongly_convex_2} ),
\begin{align}\label{eq_Strong_Convex_RT_6}
\|y_1 - \mathrm{argmax}_{z\in \mathbb{R}^d} f(x_0, z)\| \leq \frac{\eps}{\alpha}.
\end{align}

\noindent
Thus, combining \eqref{eq_Strong_Convex_RT_5} and \eqref{eq_Strong_Convex_RT_6}, we have that (since $x_0 = x_1$),
\begin{align*}
\|(x_1, y_1) - (x^\dagger, y^\dagger)\| &\leq \|(x_0, y_0) - (x^\dagger, y^\dagger)\| + \|(x_0, y_0) - (x_0, y_1)\| \\
&\leq D + \|y_0 - y_1\|\nonumber\\
& \leq D + \|y_1 - \mathrm{argmax}_{z\in \mathbb{R}^d} f(x_0, z)\|  + \|y_0 - \mathrm{argmax}_{z\in \mathbb{R}^d} f(x_0, z)\| \nonumber\\
& \leq D + \frac{LD}{\alpha} + \frac{\eps}{\alpha}\nonumber\\
&\leq \mathfrak{D}. \nonumber
\end{align*}

\noindent
{\bf Bounding the distance $\|x_i - x^\dagger\|$}.\\

\noindent
At each iteration $i>1$ of Algorithm \ref{alg:Informal} we have that $\|\nabla_y f(x_i, y_i) \| \leq \eps$.
Thus, by Lemma \ref{prop_strongly_convex_2} we have that 
\begin{equation}\label{eq_Strong_Convex_RT_1}
\|y_i -\mathrm{argmax}_{z \in \mathbb{R}^d} f(x_i, z)\| \leq \frac{\eps}{\alpha}.
\end{equation}
Thus, since $f$ is $L$-smooth,
\begin{align}\label{eq_Strong_Convex_RT_2}
    \max_{z \in \mathbb{R}^d} f(x_i, z) - f(x_i, y_i)   &\leq L \times \|y_i -\mathrm{argmax}_{z \in \mathbb{R}^d} f(x_i, z)\|^2 \\
    &\stackrel{\textrm{Eq. \ref{eq_Strong_Convex_RT_1}}} \leq   L \frac{\eps^2}{\alpha^2}. \nonumber
\end{align}

\noindent
But $f(x_{i+1}, y_{i+1}) \leq f(x_i, y_i)$ at each iteration $i$  (since, if the proposed update to $x_i$ does not lead to a decrease in the value of $f$ we have that the proposed update to $x_i$ would be rejected and $x_i= x_{i+1}$ and $y_i = y_{i+1}$).
Therefore,
\begin{align}\label{eq_Strong_Convex_RT_3}
f(x_i, y_i) \leq f(x_1, y_1) \qquad \forall i \geq 1.
\end{align}

\noindent
Therefore, \eqref{eq_Strong_Convex_RT_2} and \eqref{eq_Strong_Convex_RT_3} imply that

\begin{align}
\max_{z \in \mathbb{R}^d} f(x_i, z)  &\leq  f(x_i, y_i)+ L \frac{\eps^2}{\alpha^2} \nonumber\\
  &\stackrel{\textrm{Eq. \ref{eq_Strong_Convex_RT_2}}} \leq   f(x_1, y_1) + L \frac{\eps^2}{\alpha^2}\nonumber\\
  & \stackrel{\textrm{Eq. \ref{eq_Strong_Convex_RT_3}}} \leq   \max_{z \in \mathbb{R}^d} f(x_1, z)  + L \frac{\eps^2}{\alpha^2}\nonumber\\
  &  = f(x_0, \mathrm{argmax}_{z \in \mathbb{R}^d} f(x_0,z))  + L \frac{\eps^2}{\alpha^2}\nonumber\\
  &  \stackrel{\textrm{Eq. \ref{eq_Strong_Convex_RT_5}}} \leq   f(x_0, y_0) + L \times \frac{LD}{\alpha}  + L \frac{\eps^2}{\alpha^2}, \label{eq_Strong_Convex_RT_4}
  \end{align}
since $(x_0 = x_1)$, and where the last inequality holds by \eqref{eq_Strong_Convex_RT_5} since $f$ is $L$-smooth.
But since $\nabla_x f(x^\dagger, y^\dagger) = \nabla_y f(x^\dagger, y^\dagger) =0$, and $f$ is $L$-smooth, 

\begin{align}\label{eq_Strong_Convex_RT_9}
f(x^\dagger, y^\dagger) - f(x_0,y_0) \leq L \|(x^\dagger, y^\dagger) - (x_0,y_0) \| \leq LD.
\end{align}

\noindent
Thus, plugging in \eqref{eq_Strong_Convex_RT_9} into \eqref{eq_Strong_Convex_RT_4}, we get
\begin{align}
\psi(x_i) - \min_{\theta \in \mathbb{R}^d}\psi(\theta) \nonumber
&= \max_{z \in \mathbb{R}^d} f(x_i, z) - \min_{\theta \in \mathbb{R}^d} \max_{z \in \mathbb{R}^d} f(\theta, z)\\
&=\max_{z \in \mathbb{R}^d} f(x_i, z) - f(x^\dagger, y^\dagger)\nonumber\\
&\stackrel{\textrm{Eq. \ref{eq_Strong_Convex_RT_5}}} \leq   LD +  \frac{L^2D}{\alpha}  + L \frac{\eps^2}{\alpha^2}. \label{eq_Strong_Convex_RT_10}
  \end{align}

\noindent
But we have shown that $\psi$ is $\alpha$-strongly convex.
Therefore, \eqref{eq_Strong_Convex_RT_10} implies that
\begin{align}
\|x_i - x^\dagger\| &= \|x_i - \mathrm{argmin}_{\theta \in \mathbb{R}^d}\psi(\theta)\| \nonumber\\
&\leq \sqrt{ \frac{LD +  \frac{L^2D}{\alpha}  + L \frac{\eps^2}{\alpha^2}}{\alpha}}\nonumber\\
&\leq \mathfrak{D}. \label{eq_Strong_Convex_RT_20}
  \end{align}

\noindent
{\bf Bounding the distance $\|\mathsf{y}_{i,j} - y^\dagger\|$.}

\noindent
Now, since $\eta \leq \frac{1}{L}$, we have that $f(x_i,\mathsf{y}_{i,j})$ is nondecreasing at each iteration $j$ of Algorithm \ref{alg:InnerMaxLoop},
\begin{align}\label{eq_Strong_Convex_RT_13}
f(x_i,\mathsf{y}_{i, j+1}) \geq f(x_i,\mathsf{y}_{i, j}) \qquad \forall i\geq 0, j \geq 1.
\end{align}

\noindent
First we consider the case when $i=0$.  
Since $\mathsf{y}_{1, 0} = y_0$, by \eqref{eq_Strong_Convex_RT_5} we have that

\begin{align}\label{eq_Strong_Convex_RT_11}
\|\mathsf{y}_{1, 0} - \mathrm{argmax}_{z\in \mathbb{R}^d} f(x_0, z)\| \leq \frac{LD}{\alpha}.
\end{align}

\noindent
Thus, since $f(x_0, \cdot)$ is $L$-smooth, \eqref{eq_Strong_Convex_RT_11} implies that 
\begin{align*}
\max_{z\in \mathbb{R}^d} f(x_0, z) - f(x_0, \mathsf{y}_{1, 0}) \leq L\left(\frac{\sqrt{LD}}{\alpha}\right)^2.
\end{align*}

\noindent
Thus, by \eqref{eq_Strong_Convex_RT_13} we have that

\begin{align}\label{eq_Strong_Convex_RT_14}
\max_{z\in \mathbb{R}^d} f(x_0, z) - f(x_0,\mathsf{y}_{0, j}) \leq L\left(\frac{\sqrt{LD}}{\alpha}\right)^2  \qquad \forall  j \geq 1.
\end{align}

\noindent
Thus, since $\nabla_y f(x_0,  \mathrm{argmax}_{z\in \mathbb{R}^d} f(x_0, z)) = 0$, and since $f(x_0,\cdot)$ is $\alpha$-strongly concave, we have that
\begin{align}\label{eq_Strong_Convex_RT_15}
\|\mathsf{y}_{1, j} - \mathrm{argmax}_{z\in \mathbb{R}^d} f(x_0, z)\| \leq \sqrt{\frac{L}{\alpha}\left(\frac{\sqrt{LD}}{\alpha}\right)^2} = \frac{L\sqrt{D}}{\alpha\sqrt{\alpha}} \leq \mathfrak{D}  \qquad \forall  j \geq 1.
\end{align}

\noindent
Next, we consider the case when $i\geq1$.
At each $i \geq 1$, we have that $ \|\nabla_y f(x_i, y_i) \|  =\|\nabla_y f(x_i, \mathsf{y}_{i,0}) \| \leq \eps$.  Therefore, since $f(x_i,\cdot)$ is $\alpha$-strongly concave, we have by Lemma \ref{prop_strongly_convex_2} that
\begin{align}\label{eq_Strong_Convex_RT_16}
\|\mathsf{y}_{i,0} -\mathrm{argmax}_{z \in \mathbb{R}^d} f(x, z) \| \leq \frac{\eps}{\alpha}.
\end{align}

\noindent
Thus, since $f(x_i,\cdot)$ is $L$-smooth, \eqref{eq_Strong_Convex_RT_16} implies that
\begin{align*}
\max_{z\in \mathbb{R}^d} f(x_i, z) - f(x_i, \mathsf{y}_{i, 0}) \leq L\left(\frac{\eps}{\alpha}\right)^2.
\end{align*}

\noindent
Thus by \eqref{eq_Strong_Convex_RT_13} we have that
\begin{align*}
\max_{z\in \mathbb{R}^d} f(x_i, z) - f(x_i, \mathsf{y}_{i, j}) \leq L\left(\frac{\eps}{\alpha}\right)^2 \qquad \forall j \geq 0.
\end{align*}

\noindent
Thus, since $\nabla_y f(x_0,  \mathrm{argmax}_{z\in \mathbb{R}^d} f(x_0, z)) = 0$, and since $f(x_0,\cdot)$ is $\alpha$-strongly concave, we have that
\begin{align}\label{eq_Strong_Convex_RT_19}
\|\mathsf{y}_{i, j} - \mathrm{argmax}_{z\in \mathbb{R}^d} f(x_i, z)\| \leq \sqrt{\frac{L}{\alpha}\left(\frac{\eps}{\alpha}\right)^2} = \frac{\eps\sqrt{L}}{\alpha\sqrt{\alpha}} \leq \mathfrak{D}  \qquad \forall  j \geq 1.
\end{align}

\noindent
Therefore, from   \eqref{eq_Strong_Convex_RT_20}, \eqref{eq_Strong_Convex_RT_14}, and \eqref{eq_Strong_Convex_RT_19}, and since $y_i = \mathsf{y}_{i,0}$ for every $i$,  we have that  $\|(x_i, y_i)  - (x^\dagger, y^\dagger) \| \leq \mathfrak{D}$ and  $\|(x_i, \mathsf{y}_{i,j})  - (x^\dagger, y^\dagger) \| \leq \mathfrak{D}$ for every $i \geq 0$ and every $j \geq 0$

\end{proof}

\begin{corollary}\label{cor_runtime_strongly_convex}
Suppose that $f:\mathbb{R}^d \times \mathbb{R}^d \rightarrow \mathbb{R}$ is such that $f(\cdot,y)$ is $\alpha$-strongly convex for every $y \in \mathbb{R}^d$ and $f(x, \cdot)$ is $\alpha$-strongly concave for  every $x \in \mathbb{R}^d$, and that $f$ has $L$-Lipschitz gradients for some $L\geq \alpha>0$.
And suppose that the proposal distribution  $Q_{x,y}$ of Algorithm  \ref{alg:Informal} are the deterministic gradients $\Delta = -\frac{1}{2L}  \nabla_x f(x,y)$ for $\Delta \sim Q_{x,y}$, and that Algorithm \ref{alg:InnerMaxLoop} takes as input deterministic gradients  $\nabla_y f$.
Then, given any $\eps'>0$ and any initial point $(x_0, y_0) \in \mathbb{R}^d \times \mathbb{R}^d$, Algorithm  \ref{alg:Informal}, with appropriate parameters, outputs a point $(x^\star, y^\star)$ which is 
an approximate global min-max point of $f$ with duality gap $$\max_{y \in \mathbb{R}^d} f(x^\star, y) -    \min_{x \in \mathbb{R}^d} f(x, y^\star) \leq \eps'$$ in $\mathrm{poly}\left(L, \frac{1}{\alpha}, \frac{1}{\eps'}, D\right)$ gradient and function evaluations, where $D= \|(x_0, y_0) - (x^\dagger, y^\dagger)\|$ is the distance from the initial point to the (exact) global min-max point $(x^\dagger, y^\dagger)$ of $f$.
\end{corollary}

\begin{proof}

Set the parameters 
$$\eps = \frac{1}{10} \min\left(\frac{\alpha}{\sqrt{L}} \sqrt{\eps'}, \, \,  \, \frac{\alpha^4}{L^5} \eps'\right),  \delta = \frac{1}{10} \frac{\alpha^2}{L^3} \eps', \text{ and }  \omega = \frac{1}{4}.$$
Define $$\mathfrak{D}: = 2\max\left(D + \frac{LD}{\alpha} + \frac{\eps}{\alpha},\, \,  \,   \sqrt{ \frac{LD +  \frac{L^2D}{\alpha}  + L \frac{\eps^2}{\alpha^2}}{\alpha}},\, \,  \,  \frac{L\sqrt{D}}{\alpha\sqrt{\alpha}}, \, \,  \,  \frac{\eps\sqrt{L}}{\alpha\sqrt{\alpha}}\right).$$
Define $b := 4L \mathfrak{D}^2$, and $L_1:= 2 L \mathfrak{D}$.
Set hyperparameters $\tau_1 = \infty$ (so that the rejection probability in line \ref{AcceptRejectStep} of Algorithm \ref{alg:Informal} is $1$).
Set the remaining hyperparameters as in Items 1-8 in Section \ref{sec:proofs} with the parameter ``$L$" in Items 1-8 replaced by $\min(L, \frac{L_1^2}{2b})$. 
Since $\eta \leq \frac{1}{10L}$, by Lemma  \ref{lemma_bounding_ball} we have that  every step $i$ of  Algorithm \ref{alg:Informal} and every step $j$ of its subroutine  Algorithm \ref{alg:InnerMaxLoop} satisfy
$$\|(x_i, y_i)  - (x^\dagger, y^\dagger) \| \leq \mathfrak{D} \text{ and } \|(x_i, \mathsf{y}_{i,j})  - (x^\dagger, y^\dagger) \| \leq \mathfrak{D}$$ for all $i\geq 0$ and all $j \geq 0$.
Thus, Algorithm \ref{alg:Informal} and its subroutine  Algorithm \ref{alg:InnerMaxLoop} remain inside the ball $B((x^\dagger, y^\dagger), \mathfrak{D})$ of radius $\mathfrak{D}$ with center at the global min-max point $(x^\dagger, y^\dagger)$ of $f$.
Since  $(x^\dagger, y^\dagger)$ is the global min-max point of $f$, we have that $$\nabla_x f(x^\dagger, y^\dagger) =\nabla_y f(x^\dagger, y^\dagger)= 0.$$
Without loss of generality we may assume that $f(x^\dagger, y^\dagger)=0$ (we can assume this since each step of the algorithm remains the same if we add a constant to $f$).
Thus, since $f$ is $L$-smooth on all of $\mathbb{R}^d \times \mathbb{R}^d$, we have that 
\begin{align*} 
|f(x,y)| \leq L\times 4 \mathfrak{D}^2 \qquad  \forall (x,y) \in B((x^\dagger, y^\dagger),  2\mathfrak{D}), \text{ and }
\end{align*}
\begin{align*}
\|(\nabla_x f(x,y), \nabla_y f(x,y))\| \leq L\times 2\mathfrak{D} \qquad  \forall (x,y) \in B((x^\dagger, y^\dagger),  2\mathfrak{D}).
\end{align*}

\noindent
Since $f$ is $b$-bounded with $L_1$-Lipschitz gradient on the ball $B((x^\dagger, y^\dagger),  2\mathfrak{D})$, and since every step of the algorithm remains inside the ball $B((x^\dagger, y^\dagger),  \mathfrak{D}) \subseteq B((x^\dagger, y^\dagger),  2\mathfrak{D})$, each step of the proof of Theorem \ref{thm:GreedyMinimax-main} holds if we replace the parameter ``L" in that proof with $\min\left(L, \frac{L_1^2}{2b}\right)$ (since the parameter ``$L$" in the proof of Theorem \ref{thm:GreedyMinimax-main} is required to be $\leq \frac{L_1^2}{4b}$, and setting that parameter ``L" to be $\min\left(L, \frac{L_1^2}{2b}\right)$ ensures that this assumption holds).

Therefore, the conclusion of Theorem \ref{thm:GreedyMinimax-main} must also hold and we have that Algorithm \ref{alg:Informal} returns a point
  $(x^\star,y^\star) \in \mathbb{R}^{d} \times \mathbb{R}^{d}$
  such that, for some  $\epsilon^\star \in\left[ \frac{1}{2}\epsilon, \epsilon\right]$,  $(x^\star, y^\star)$ is an $(\epsilon^\star, \delta, \omega,Q)$-equilibrium.
  The number of gradient and function evaluations required by the algorithm is $$\mathrm{poly}\left(b,\min\left(L, \frac{L_1^2}{2b}\right), \frac{1}{\eps}, \frac{1}{\delta}, \frac{1}{4}\right)$$ and does not depend on the dimension $d$.
Note that, since we assume the gradients and proposal distribution are deterministic, each step of the algorithm is also deterministic, and the conclusion must hold with probability $1$.
But $$\frac{1}{\eps}, b, \frac{1}{\delta} = \mathrm{poly}\left(L, \frac{1}{\alpha}, \frac{1}{\eps}, D\right) \text{ and } \min\left(L, \frac{L_1^2}{2b}\right) = \mathrm{poly}\left(L, \frac{1}{\alpha}, \frac{1}{\eps}, D\right).$$
Therefore, the number of gradient and function evaluations is also  $\mathrm{poly}(L, \frac{1}{\alpha}, \frac{1}{\eps}, D)$.
We have now shown that Algorithm \ref{alg:Informal} returns a point $(x^\star, y^\star)$ which is an $(\epsilon^\star, \delta, \omega,Q)$-equilibrium for $f$ where $\epsilon^\star \in\left[ \frac{1}{2}\epsilon, \epsilon\right]$  (and in particular,  $\eps^\star, \delta = \mathrm{poly}\left(\eps', \alpha, \frac{1}{L}\right)$).
Therefore, by Theorem \ref{thm_strongly_convex}, we have that since $f: \mathbb{R}^d \times \mathbb{R}^d \rightarrow \mathbb{R}$ is $\alpha$-strongly convex in $x$ and $\alpha$-strongly concave in $y$, with $L$-Lipschitz gradient in both variables, the point $(x^\star, y^\star)$, which is an $(\epsilon^\star, \delta, \omega,Q)$-equilibrium, also satisfies the duality gap

\begin{align*}
\max_{y \in \mathbb{R}^d} f(x^\star, y) &-    \min_{x \in \mathbb{R}^d} f(x, y^\star)\\ &\leq \frac{L(\eps^\star)^2}{2\alpha^2} + 
    \frac{L^3}{\alpha^2}
   \left(\sqrt{\delta + L \left(2\frac{\eps^\star}{\alpha} +  \frac{1}{\alpha}\left(L\frac{\eps^\star}{\alpha} + 2 \sigma\right)\right) + L \frac{(\eps^\star)^2}{2 \alpha^2} +L \frac{\eps^\star}{\alpha}} +  L \frac{\eps^\star}{\alpha}\right)^2 \nonumber\\
   &\leq \eps'.
\end{align*}

\end{proof}

\section{Examples of Functions where Global Min-Max Satisfies Definition~\ref{def_greedy_minmax} but not Other Local Equilibrium Notions} \label{sec:examples}

In this section, we expand upon the examples mentioned in 
Section~\ref{sec_theoretical_results}.
In particular, we provides example functions for which there exists min-max points that satisfy Definition~\ref{def_greedy_minmax} but which do not satisfy other common notions of local equilibrium.

\paragraph{Functions for which global min-max points are not first-order stationary points.}

Consider the function
$$f(x,y) = \sin(x) \times \sin(y) - \sum_{m,n \in \mathbb{Z}}  \mathrm{Bump}(x+ m \pi,y + n \pi),$$  where $$\mathrm{Bump}(x,y) = e^{-1/(1-100(x^2+y^2))} \text{ for } x^2 +y^2 < \frac{1}{100} \text{ and } \mathrm{Bump}(x,y)  =0 \text{ everywhere else}.$$  
This function has a global min-max point at $(x,y) = (0,1)$ and this point also satisfies Definition~\ref{def_greedy_minmax} (and $f$ also has such a point at all points along the line $x=0$ except for the intervals $(-\frac{1}{10}+n \pi, \frac{1}{10}+n \pi)$ for intergers $n$), and yet $$\nabla_x f(0,1) = \cos(0)\times\sin(1) = 0.84$$ meaning that $(x,y) = (0,1)$ is not a first-order stationary point for $x$.  
In fact, every global min-max point of this function is not a first-order stationary point in $x$. 

\paragraph{Functions for which global min-max points are not second-order equilibrium points.}
For functions $$f(x,y)= \sin(x+y) \text{ and } $$ $$f(x,y)= 10^3\cdot\sum_{k\in \mathbb{Z}} e^{-(x+y + 2 +9k)^2} + 2e^{-(x+y + 2 +9k)^2}- e^{-(x+6k)^2},$$ there are no $\varepsilon$-approximate local min/max points for $\varepsilon<\frac{1}{2}$, and yet, an equilibrium point from Definition~\ref{def_greedy_minmax} is guaranteed to exist for such functions. 
Note that these functions are indeed smooth and bounded.

\section{Comparison of Local Equilibrium Point and Local Min-Max Point} \label{sec_local}

\begin{lemma} \label{lemma_local}
Suppose that $(x^\star, y^\star)$ is such that $y^\star$ is a local maximum point of $f(x^\star, \cdot)$ and $x^\star$ is a local minimum point of $f(\cdot, y^\star)$.
Then $(x^\star, y^\star)$ is also a local equilibrium of $f$.
\end{lemma}

\begin{proof}

Fix any $\eps \geq 0$ (the proof of this Lemma requires only $\eps = 0$, but we state the proof for any $\eps \geq 0$ since this will allow us to prove Corollary \ref{cor_local}).
Since $y^\star$ is a local maximum of $f(x^\star, \cdot)$, there is only one $\eps$-greedy path with initial point $y^\star$, namely, the path  $\{y^\star\}$ consisting of the single point $y^\star$ (since $f$ must increase at rate at least $\eps$ at every point on an $\eps$-greedy path).
Thus,
\begin{equation}\label{eq_local_1}
P_\eps(x^\star, y^\star) = \{y^\star\}
\end{equation}

\noindent
Hence, \eqref{eq_local_1} implies that
\begin{equation}\label{eq_local_2}
y^\star \in \mathrm{argmax}_{y \in P_\eps(x^\star, y^\star)} f(x^\star,y)
\end{equation}
which proves Equation \eqref{eq_our_local_frameowrk_y}.
Next, we will show that Equation \eqref{eq_our_local_frameowrk_x} holds.
Since $x^\star$ is a local minimum point of $f(\cdot, y^\star)$, there exists $\nu >0$ such that
\begin{equation} \label{eq_local_3}
    f(z,y^\star) \geq f(x^\star, y^\star) \qquad \qquad \forall z \in B(x^\star, \nu)
\end{equation}

\noindent
Since $y^\star \in P_\eps(x, y^\star)$ for all $x \in \mathcal{X}$, we have that
\begin{equation} \label{eq_local_4}
    \max_{y \in P_\eps(x, y^\star)} f(x,y) \geq f(x,y^\star) \qquad \forall x \in \mathcal{X},
\end{equation}
and hence that
\be
\min_{x \in B(x^\star, \nu) \cap \mathcal{X}} \max_{y \in P_\eps(x, y^\star)} f(x,y) &\stackrel{\textrm{Eq. } \ref{eq_local_4}} \geq   \min_{x \in B(x^\star, \nu)} f(x, y^\star)\\
& \stackrel{\textrm{Eq. } \ref{eq_local_3}} =   f(x^\star, y^\star)\\
& \stackrel{\textrm{Eq. } \ref{eq_local_2}} =   \max_{y \in P_\eps(x^\star, y^\star)} f(x^\star,y),
\ee
which proves Equation \eqref{eq_our_local_frameowrk_x}.

\end{proof}

\begin{corollary} \label{cor_local}
Suppose that $(x^\star, y^\star)$ is such that $y^\star$ is a local maximum point of $f(x^\star, \cdot)$ and $x^\star$ is a local minimum point of $f(\cdot, y^\star)$.
Then there exists $\nu>0$ such that, for any $\eps, \delta \geq 0$,
and any proposal distribution $Q$ with support on $\mathcal{X}$ which satisfies
\be\label{eq_local_6}
\Pr_{\Delta \sim Q_{x^\star,y^\star}}(\|\Delta\| \geq \nu) < \omega,
\ee
for some $\omega >0$,
$(x^\star, y^\star)$ is also an approximate local equilibrium of $f$ for parameters $(\eps, \delta, \omega)$ and proposal distribution $Q$.
\end{corollary}

\noindent
We note that many distributions satisfy \eqref{eq_local_6}, for instance the distribution $Q_{x,y} \sim N(0,\sigma^2 I_d)$ for $\sigma = O(\nu \log^{-1}(\frac{1}{\omega}))$.

\begin{proof}
By Inequality \eqref{eq_local_6} in the proof of Lemma \ref{lemma_local}, there exists $\nu > 0$ such that
\be \label{eq_local_7}
\min_{x \in B(x^\star, \nu)\cap \mathcal{X}} \max_{y \in P_\eps(x, y^\star)} f(x,y) \geq \max_{y \in P_\eps(x^\star, y^\star)} f(x^\star,y),
\ee

\noindent
Thus, for any proposal distribution $Q$ which satisfies Inequality \eqref{eq_local_6}, Inequality \eqref{eq_local_7} implies that, for any $\delta \geq 0$,
\be
\Pr_{\Delta \sim {Q}_{x^\star, y^\star}} 
    \bigg[\max_{y \in P_\eps(x^\star +
    \Delta,y^\ast)} &f(x^\star +
    \Delta,y)   < \max_{y \in P_{\eps}(x^\star,y^\ast)} f(x^\star,y) - \delta \bigg]\\
&\ \stackrel{\textrm{Eq. } \ref{eq_local_7}} \leq   \Pr_{\Delta \sim {Q}_{x^\star, y^\star}} 
   \left[ x^\star + \Delta  \notin B(x^\star, \nu) \cap \mathcal{X} \right]\\
&=\Pr_{\Delta \sim Q_{x^\star,y^\star}}(\|\Delta\| \geq \nu)\\
    & \stackrel{\textrm{Eq. } \ref{eq_local_6}} <   \omega, 
\ee

\noindent
This proves Inequality \eqref{eq_approx_local_equilibrium_x}.
Inequality \eqref{eq_approx_local_equilibrium_y} follows directly from Inequality \eqref{eq_local_2} in the proof of Lemma \ref{lemma_local}.

\end{proof}

\begin{figure}[t]
    \centering
    \includegraphics[width=\linewidth]{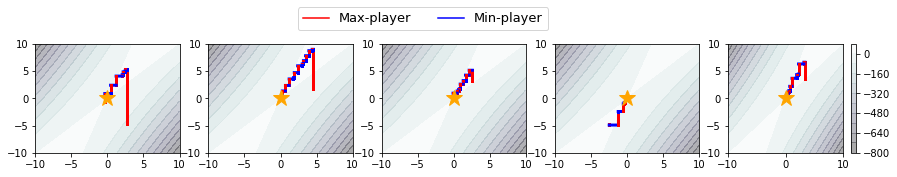}
    \caption{Different runs of our algorithm over function $F_1$ for random starting points.}
    \label{fig:f1_random_start}
\end{figure}

\begin{figure}[t]
    \centering
    \includegraphics[width=\linewidth]{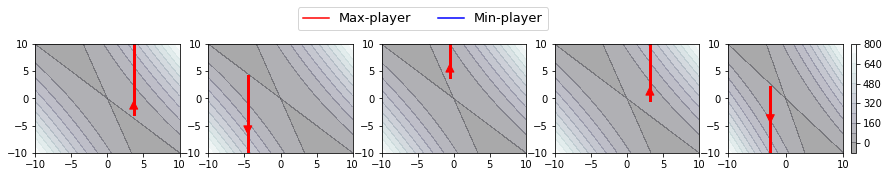}
    \caption{Different runs of our algorithm over function $F_2$ for random starting points.}
    \label{fig:f2_random_start}
\end{figure}

\begin{figure}[t]
    \centering
    \includegraphics[width=\linewidth]{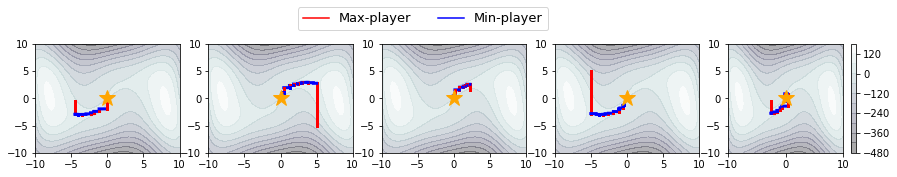}
    \caption{Different runs of our algorithm over function $F_3$ for random starting points.}
    \label{fig:f3_random_start}
\end{figure}

\section{Additional Empirical Details and Results for Test Functions and Gaussian Mixture Dataset}

\subsection{Simulation Setup for Low-dimensional Test Functions} \label{toy_simulation_setup}

In this section we describe the setup for the simulations on the low-dimensional test functions presented in Figures \ref{fig_intro} and \ref{figure_toy_function}.
For our algorithm, we use a learning rate of 
$\eta = 0.05$ for the max-player, and a proposal distribution of
$Q_{x,y} \sim N(0, 0.25)$ for the min-player.
For GDA and OMD we use a learning rate of $0.05$ for both the min-player and the max-player.
When generating Figures \ref{fig_intro} and \ref{figure_toy_function} we used the initial point $(x_0, y_0) = (5.5, 5.5)$ for all three algorithms.

\subsection{Additional Simulation Results for Low-dimensional Test Functions} \label{toy_simulation_random_start}

{We also run our algorithm for toy functions $F_1, F_2, F_3$ (defined in Section~\ref{sec:experiments}) on random initial points.
The results are present in Figures~\ref{fig:f1_random_start}, \ref{fig:f2_random_start}, \ref{fig:f3_random_start} for functions $F_1, F_2, F_3$, respectively.
For all starting points, our algorithm converges to global min-max point $(0,0)$ for functions $F_1, F_3$, and diverges to $\infty$ for function $F_2$.
}

\subsection{Simulation Setup for Gaussian Mixture Dataset} \label{appendix_hyperparameters}

In this section we discuss the neural network architectures, choice of hyperparameters, and hardware used for the Gaussian mixture dataset

\noindent \paragraph{Hyperparameters for Gaussian mixture simulations.}
 For the simulations on Gaussian mixture data, we have used the code provided by the authors of \cite{Metz2017unrolled} (\url{github.com/poolio/unrolled_gan}), which uses a batch size 512, Adam learning rates of $10^{-3}$ for the generator and $10^{-4}$ for the discriminator, and Adam parameter $\beta_1 = 0.5$ for both the generator and discriminator.\footnote{Note that the authors also mention using slightly different ADAM parameters and neural network architecture in their paper than in their code; we have used the Adam parameters and neural network architecture  provided in their code.}
      We use the same neural networks that were used in the code from \cite{Metz2017unrolled}: The generator uses a fully connected neural network with 2 hidden layers of size 128 and RELU activation, followed by a linear projection to two dimensions. 
       The discriminator uses a fully connected neural network with 2 hidden layers of size 128 and RELU activation, followed by a linear projection to 1 dimension (which is  fed as input to the cross entropy loss function).  
      As in the paper \cite{Metz2017unrolled}, we initialize all the neural network weights to be orthogonal with scaling 0.8.  
      
For OMD, we once again use Wasserstein loss and clip parameter 0.01 (\url{github.com/vsyrgkanis/optimistic_GAN_training/}).

\noindent \paragraph{Setting hyperparameters.}
In our simulations, our goal was to be able to use the smallest number of discriminator or unrolled steps while still learning the distribution in a short amount of time, and we therefore decided to compare all algorithms using the same hyperparameter $k$. To choose this single value of $k$, we started by running each algorithm with $k=1$ and increased the number of discriminator steps until one of the algorithms was able to learn the distribution consistently in the first 1500 iterations.  

The experiments were performed on four 3.0 GHz Intel Scalable CPU Processors, provided by AWS.

\subsection{Additional Simulation Results for Gaussian Mixture Dataset}
\label{sec:More_simulation_results}

\label{sec:More_simulations_GaussianMixture}

In this section we show the results of all the runs of the simulation mentioned in Figure \ref{fig:4_Gaussians}, where all the algorithms were trained on a 4-Gaussian mixture dataset for 1500 iterations.  For each run, we plot points from the generated distribution at iteration 1,500.  Figure \ref{fig:4_Gaussians_AllRuns_VanillaOneDisc} gives the results for GDA with $k=1$ discriminator step. Figure \ref{fig:4_Gaussians_AllRuns_VanillaSixDisc} gives the results for GDA with $k=6$ discriminator steps.
Figure \ref{fig:4_Gaussians_AllRuns_Unrolled} gives the results for the Unrolled GANs algorithm. 
 Figure \ref{fig:omd_gaussian_all} gives the results for the OMD algorithm.
Figure \ref{fig:4_Gaussians_AllRuns_OurAlgorithm} gives the results for our algorithm.

\begin{figure*}
    \centering
    \begin{center}
\textbf{\small GDA with 1 discriminator step} \\
\end{center}
    \includegraphics[width=\linewidth]{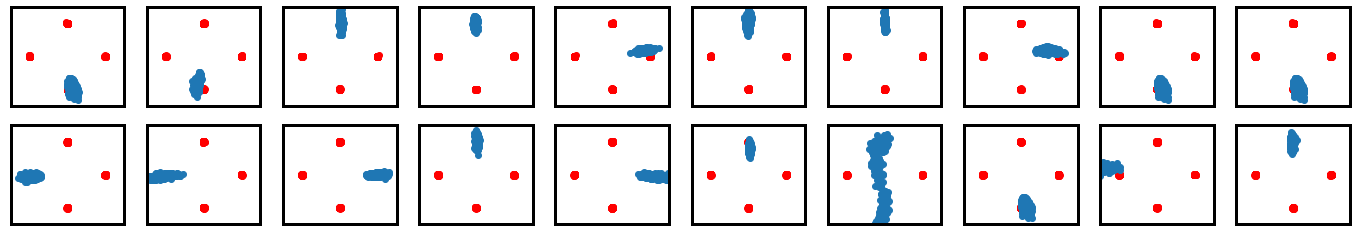}
    \caption{The generated points at the 1500'th iteration for all runs of GDA with $k=1$ discriminator steps.}
    \label{fig:4_Gaussians_AllRuns_VanillaOneDisc}
\end{figure*}

\begin{figure*}
\begin{center}
\textbf{\small GDA with 6 discriminator steps} \\
\end{center}
\noindent \fbox{\includegraphics[trim={2cm 2cm 2cm 2cm},clip, scale=0.12]{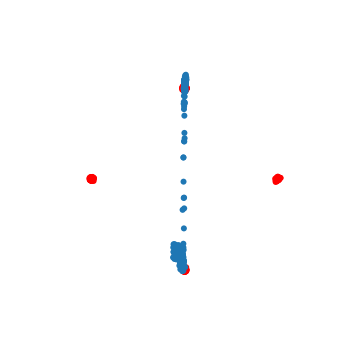}}
\noindent \fbox{\includegraphics[trim={2cm 2cm 2cm 2cm},clip, scale=0.12]{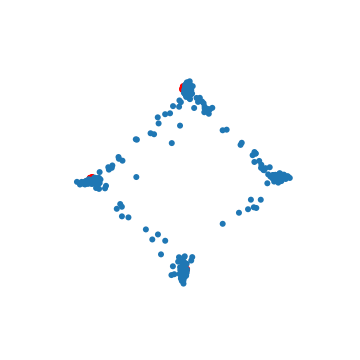}}
\noindent \fbox{\includegraphics[trim={2cm 2cm 2cm 2cm},clip, scale=0.12]{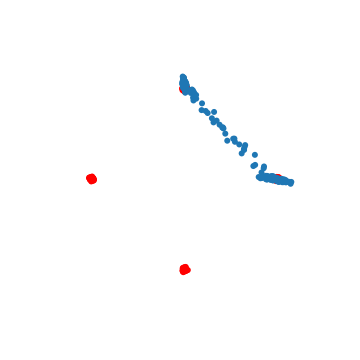}}
\noindent \fbox{\includegraphics[trim={2cm 2cm 2cm 2cm},clip, scale=0.12]{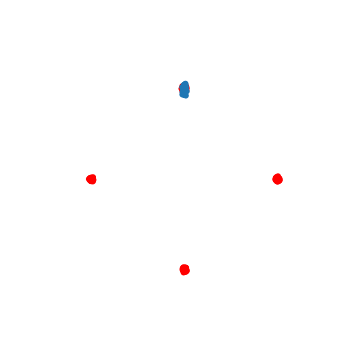}}
\noindent \fbox{\includegraphics[trim={2cm 2cm 2cm 2cm},clip, scale=0.12]{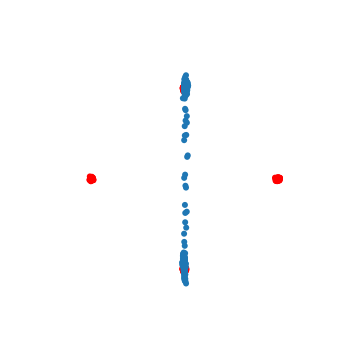}}
\noindent \fbox{\includegraphics[trim={2cm 2cm 2cm 2cm},clip, scale=0.12]{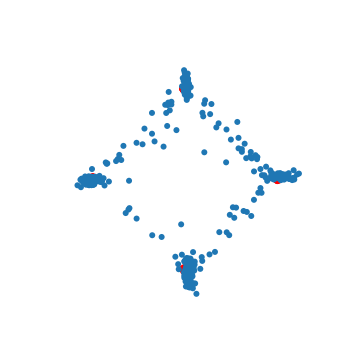}}
\noindent \fbox{\includegraphics[trim={2cm 2cm 2cm 2cm},clip, scale=0.12]{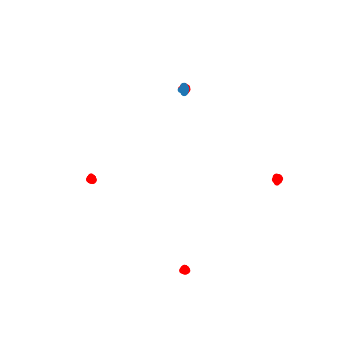}}
\noindent \fbox{\includegraphics[trim={2cm 2cm 2cm 2cm},clip, scale=0.12]{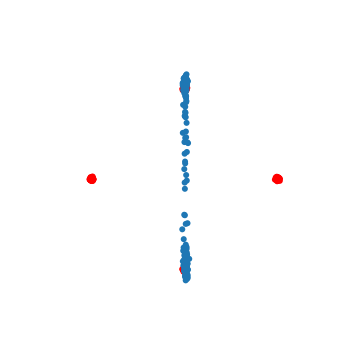}}
\noindent \fbox{\includegraphics[trim={2cm 2cm 2cm 2cm},clip, scale=0.12]{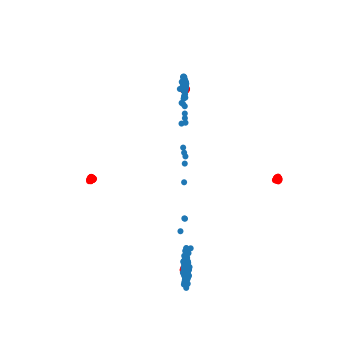}}
\noindent \fbox{\includegraphics[trim={2cm 2cm 2cm 2cm},clip, scale=0.12]{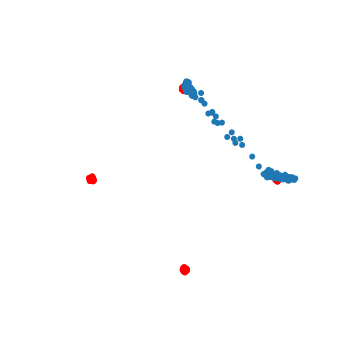}}
\noindent \fbox{\includegraphics[trim={2cm 2cm 2cm 2cm},clip, scale=0.12]{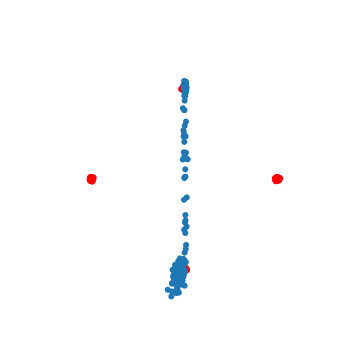}}
\noindent \fbox{\includegraphics[trim={2cm 2cm 2cm 2cm},clip, scale=0.12]{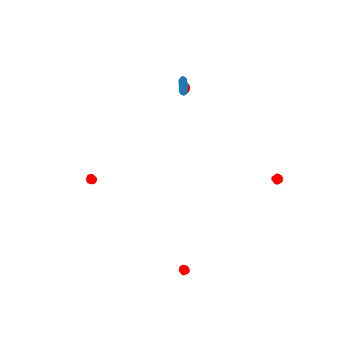}}
\noindent \fbox{\includegraphics[trim={2cm 2cm 2cm 2cm},clip, scale=0.12]{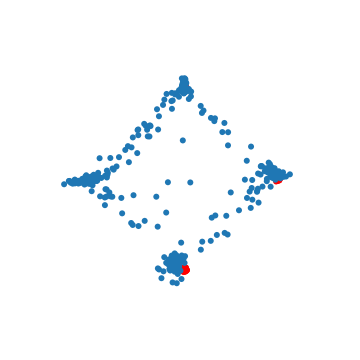}}
\noindent \fbox{\includegraphics[trim={2cm 2cm 2cm 2cm},clip, scale=0.12]{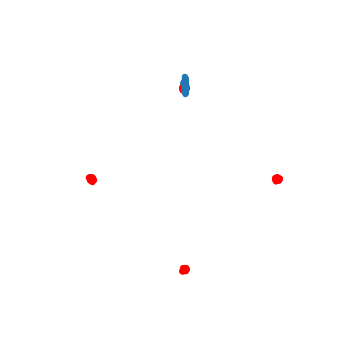}}
\noindent \fbox{\includegraphics[trim={2cm 2cm 2cm 2cm},clip, scale=0.12]{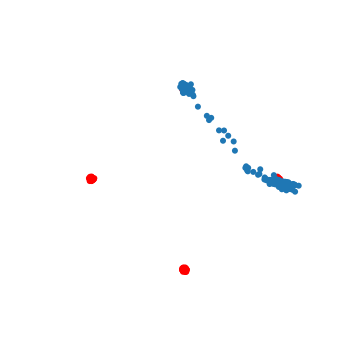}}
\noindent \fbox{\includegraphics[trim={2cm 2cm 2cm 2cm},clip, scale=0.12]{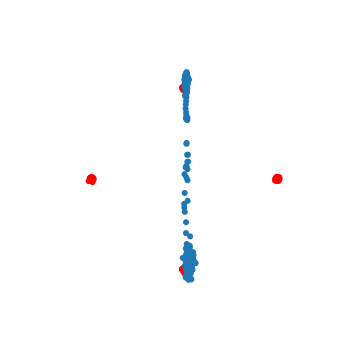}}
\noindent \fbox{\includegraphics[trim={2cm 2cm 2cm 2cm},clip, scale=0.12]{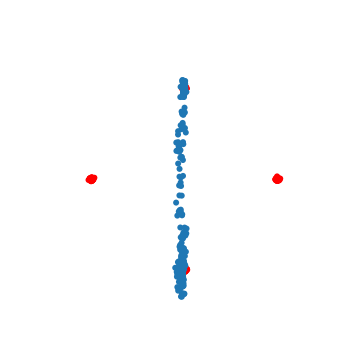}}
\noindent \fbox{\includegraphics[trim={2cm 2cm 2cm 2cm},clip, scale=0.12]{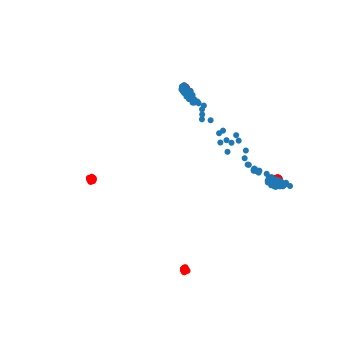}}
\noindent \fbox{\includegraphics[trim={2cm 2cm 2cm 2cm},clip, scale=0.12]{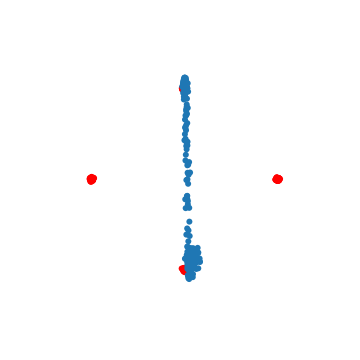}}
\noindent \fbox{\includegraphics[trim={2cm 2cm 2cm 2cm},clip, scale=0.12]{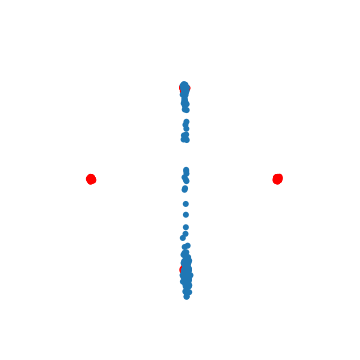}}
\caption{\small The generated points at the 1500'th iteration for all runs of the GDA algorithm, with $k=6$ discriminator steps, for the simulation mentioned in Figure \ref{fig:4_Gaussians}.  At the 1500'th iteration, GDA had learned two modes 65\% of the runs, one mode 20\% of the runs, and four modes  15 \% of the runs.}\label{fig:4_Gaussians_AllRuns_VanillaSixDisc}
\end{figure*}

\begin{figure*}
\begin{center}
\textbf{\small Unrolled GANs with 6 unrolling steps} \\
\end{center}
\noindent \fbox{\includegraphics[trim={2cm 2cm 2cm 2cm},clip, scale=0.12]{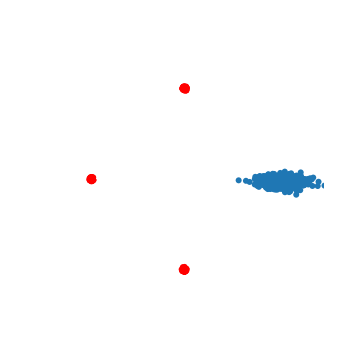}}
\noindent \fbox{\includegraphics[trim={2cm 2cm 2cm 2cm},clip, scale=0.12]{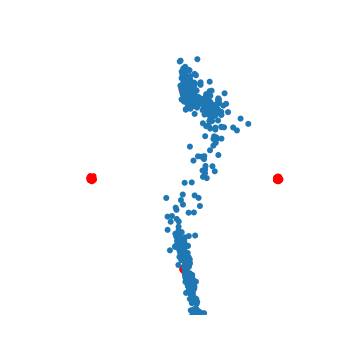}}
\noindent \fbox{\includegraphics[trim={2cm 2cm 2cm 2cm},clip, scale=0.12]{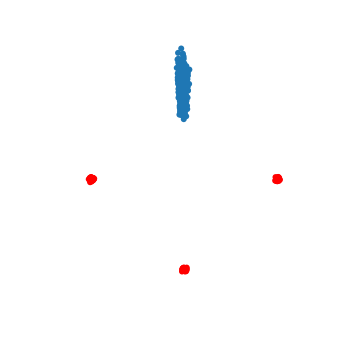}}
\noindent \fbox{\includegraphics[trim={2cm 2cm 2cm 2cm},clip, scale=0.12]{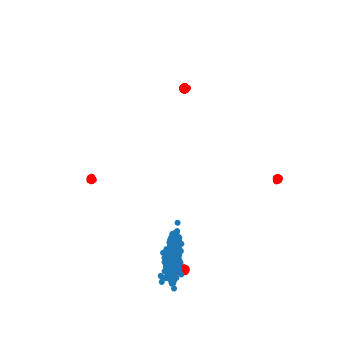}}
\noindent \fbox{\includegraphics[trim={2cm 2cm 2cm 2cm},clip, scale=0.12]{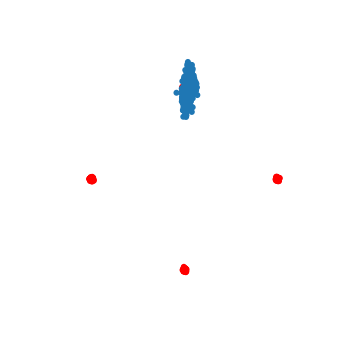}}
\noindent \fbox{\includegraphics[trim={2cm 2cm 2cm 2cm},clip, scale=0.12]{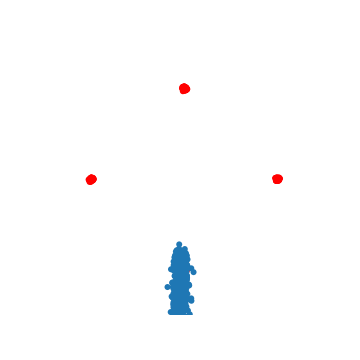}}
\noindent \fbox{\includegraphics[trim={2cm 2cm 2cm 2cm},clip, scale=0.12]{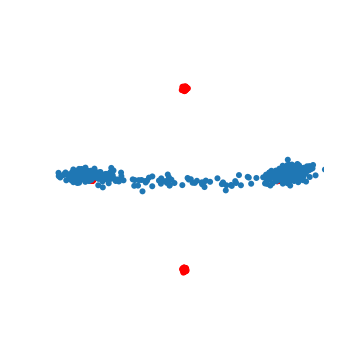}}
\noindent \fbox{\includegraphics[trim={2cm 2cm 2cm 2cm},clip, scale=0.12]{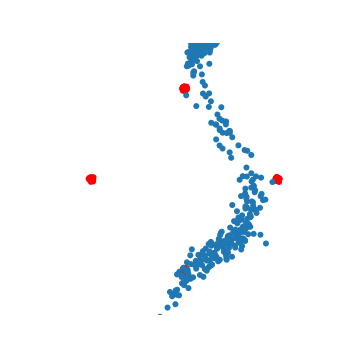}}
\noindent \fbox{\includegraphics[trim={2cm 2cm 2cm 2cm},clip, scale=0.12]{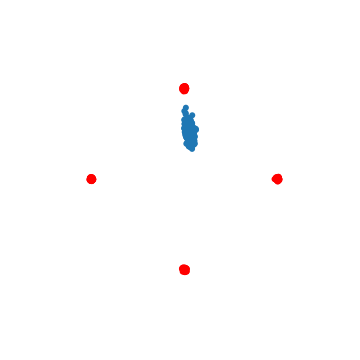}}
\noindent \fbox{\includegraphics[trim={2cm 2cm 2cm 2cm},clip, scale=0.12]{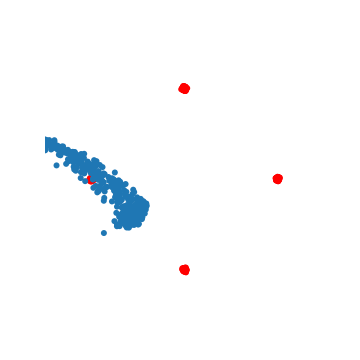}}
\noindent \fbox{\includegraphics[trim={2cm 2cm 2cm 2cm},clip, scale=0.12]{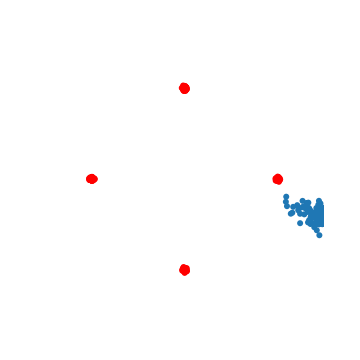}}
\noindent \fbox{\includegraphics[trim={2cm 2cm 2cm 2cm},clip, scale=0.12]{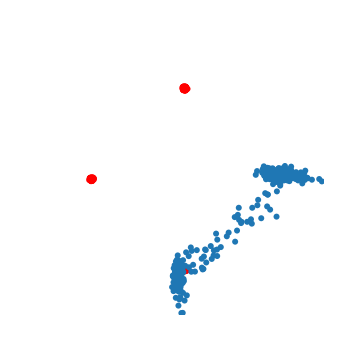}}
\noindent \fbox{\includegraphics[trim={2cm 2cm 2cm 2cm},clip, scale=0.12]{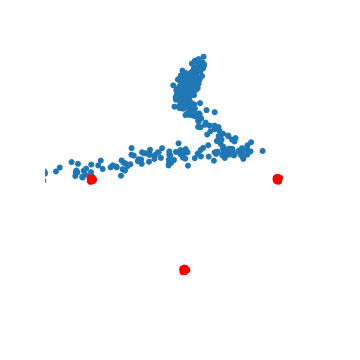}}
\noindent \fbox{\includegraphics[trim={2cm 2cm 2cm 2cm},clip, scale=0.12]{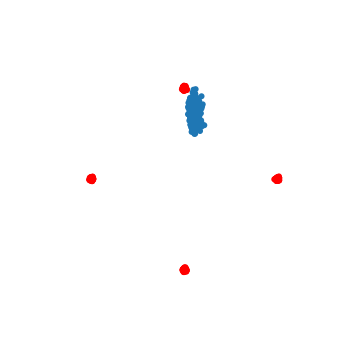}}
\noindent \fbox{\includegraphics[trim={2cm 2cm 2cm 2cm},clip, scale=0.12]{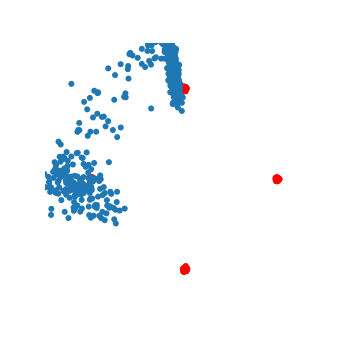}}
\noindent \fbox{\includegraphics[trim={2cm 2cm 2cm 2cm},clip, scale=0.12]{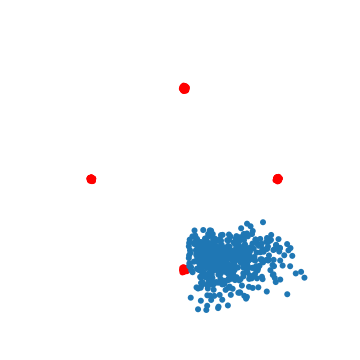}}
\noindent \fbox{\includegraphics[trim={2cm 2cm 2cm 2cm},clip, scale=0.12]{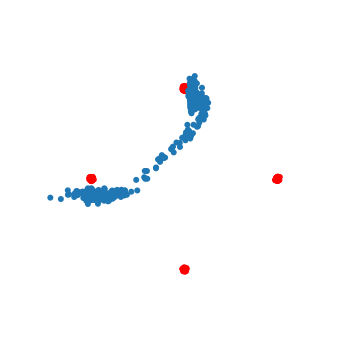}}
\noindent \fbox{\includegraphics[trim={2cm 2cm 2cm 2cm},clip, scale=0.12]{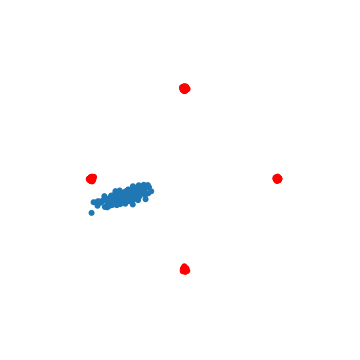}}
\noindent \fbox{\includegraphics[trim={2cm 2cm 2cm 2cm},clip, scale=0.12]{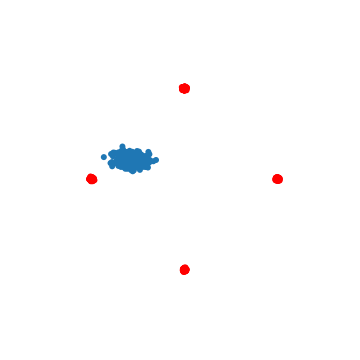}}
\noindent \fbox{\includegraphics[trim={2cm 2cm 2cm 2cm},clip, scale=0.12]{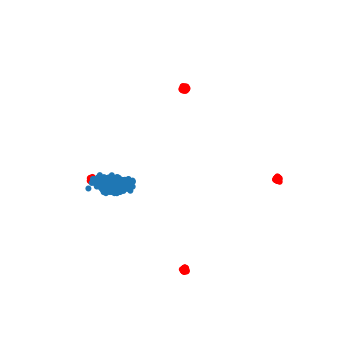}}
%
%
\caption{\small The generated points at the 1500'th iteration for all runs of the Unrolled GAN algorithm for the example in Figure \ref{fig:4_Gaussians}, with $k=6$ unrolling steps.
}\label{fig:4_Gaussians_AllRuns_Unrolled}
\end{figure*}

\begin{figure*}
    \centering
    \begin{center}
\textbf{\small OMD} \\
\end{center}
    \includegraphics[width=\linewidth]{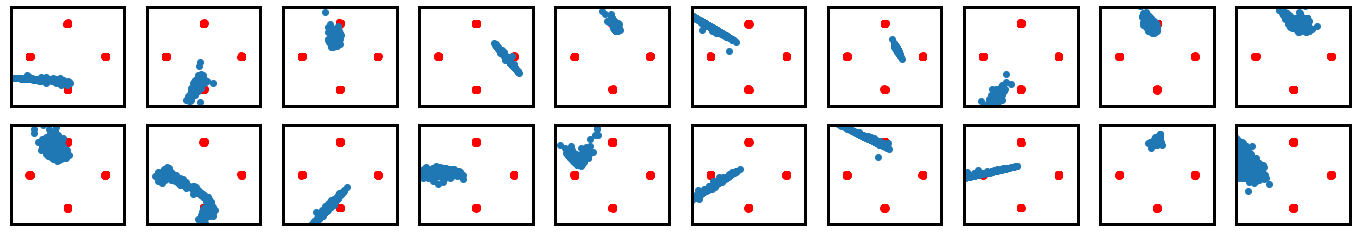}
    \caption{The generated points at the 1500'th iteration for all  runs of OMD algorithm.}
    \label{fig:omd_gaussian_all}
\end{figure*}

\begin{figure*}
\begin{center}
\textbf{\small Our algorithm} \\
\end{center}
\noindent \fbox{\includegraphics[trim={2cm 2cm 2cm 2cm}, clip, scale=0.12]{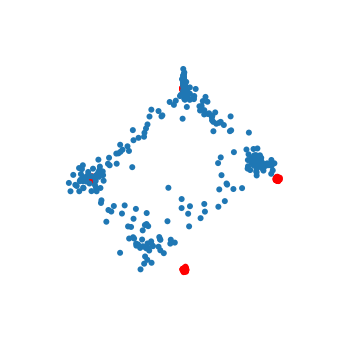}}
\noindent \fbox{\includegraphics[trim={2cm 2cm 2cm 2cm},clip, scale=0.12]{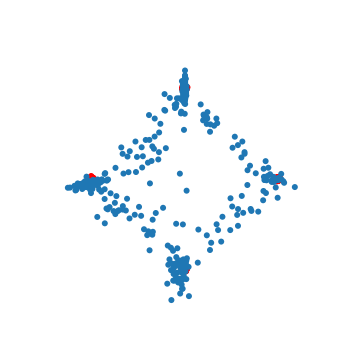}}
\noindent \fbox{\includegraphics[trim={2cm 2cm 2cm 2cm},clip, scale=0.12]{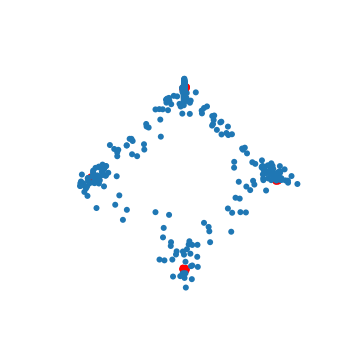}}
\noindent \fbox{\includegraphics[trim={2cm 2cm 2cm 2cm},clip, scale=0.12]{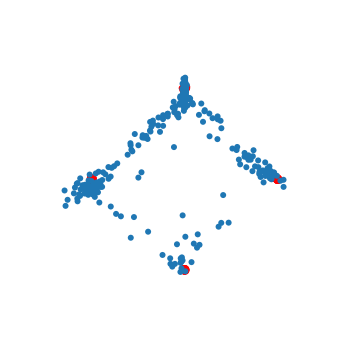}}
\noindent \fbox{\includegraphics[trim={2cm 2cm 2cm 2cm},clip, scale=0.12]{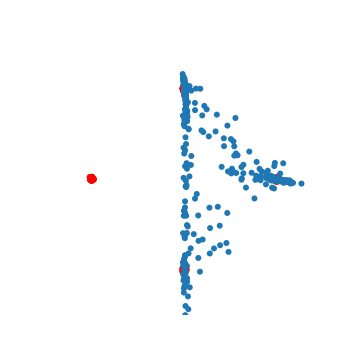}}
\noindent \fbox{\includegraphics[trim={2cm 2cm 2cm 2cm},clip, scale=0.12]{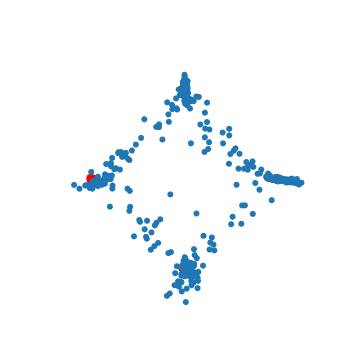}}
\noindent \fbox{\includegraphics[trim={2cm 2cm 2cm 2cm},clip, scale=0.12]{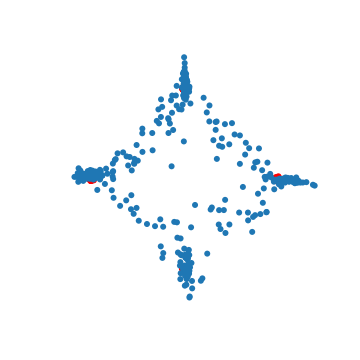}}
\noindent \fbox{\includegraphics[trim={2cm 2cm 2cm 2cm},clip, scale=0.12]{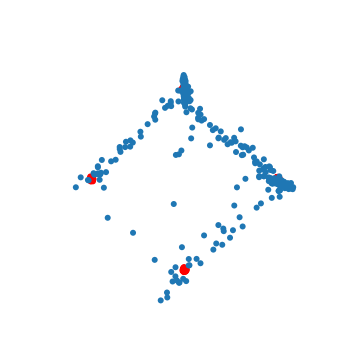}}
\noindent \fbox{\includegraphics[trim={2cm 2cm 2cm 2cm},clip, scale=0.12]{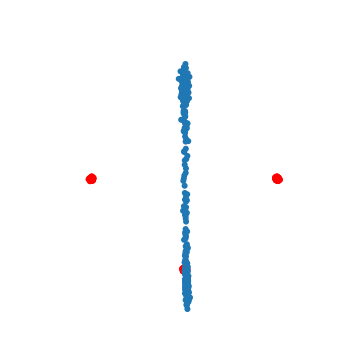}}
\noindent \fbox{\includegraphics[trim={2cm 2cm 2cm 2cm},clip, scale=0.12]{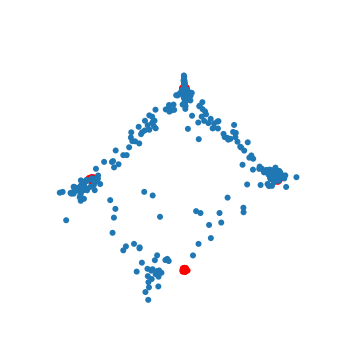}}
\noindent \fbox{\includegraphics[trim={2cm 2cm 2cm 2cm},clip, scale=0.12]{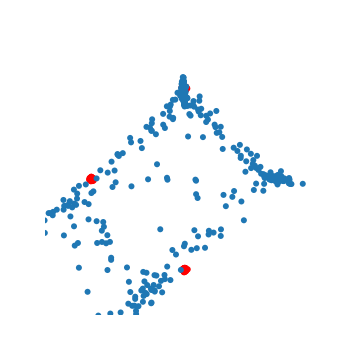}}
\noindent \fbox{\includegraphics[trim={2cm 2cm 2cm 2cm},clip, scale=0.12]{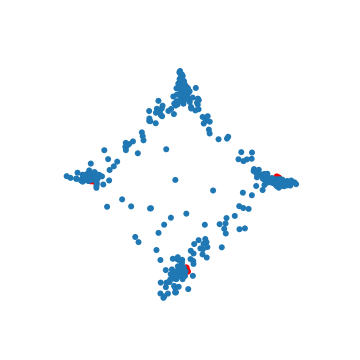}}
\noindent \fbox{\includegraphics[trim={2cm 2cm 2cm 2cm},clip, scale=0.12]{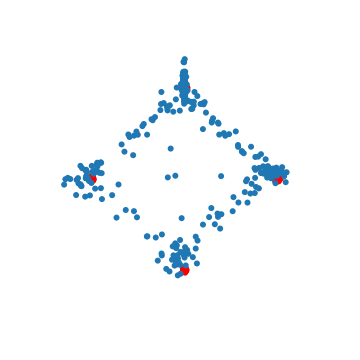}}
\noindent \fbox{\includegraphics[trim={2cm 2cm 2cm 2cm},clip, scale=0.12]{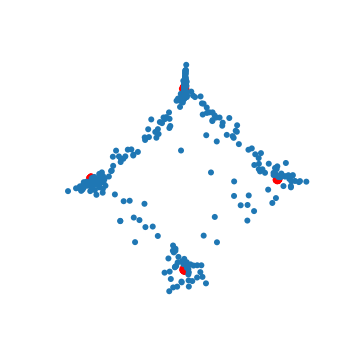}}
\noindent \fbox{\includegraphics[trim={2cm 2cm 2cm 2cm},clip, scale=0.12]{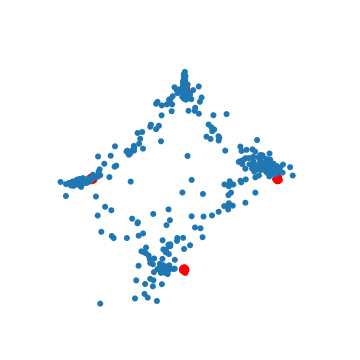}}
 \noindent \fbox{\includegraphics[trim={2cm 2cm 2cm 2cm},clip, scale=0.12]{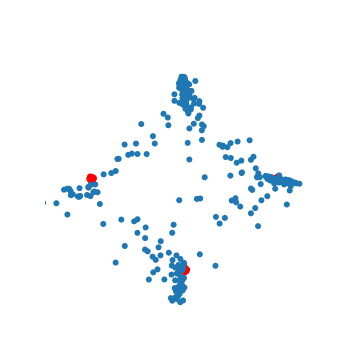}}
 \noindent \fbox{\includegraphics[trim={2cm 2cm 2cm 2cm},clip, scale=0.12]{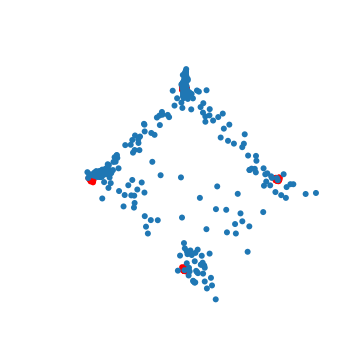}}
 \noindent \fbox{\includegraphics[trim={2cm 2cm 2cm 2cm},clip, scale=0.12]{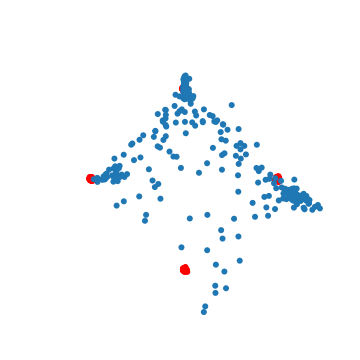}}
 \noindent \fbox{\includegraphics[trim={2cm 2cm 2cm 2cm},clip, scale=0.12]{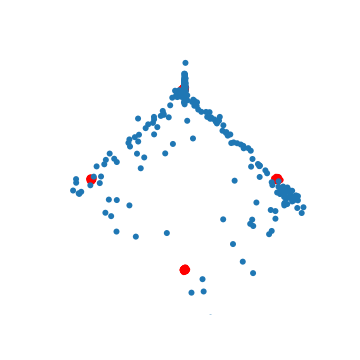}}
\caption{\small The generated points at the 1500'th iteration for all runs of our algorithm, for the simulation mentioned in Figure \ref{fig:4_Gaussians}.  Our algorithm used $k=6$ discriminator steps and an acceptance rate hyperparameter of $\frac{1}{\tau} = \frac{1}{4}$.  By the 1500'th iteration, our algorithm seems to have learned all four modes 70\% of the runs, three modes 15\% of the runs, and two modes 15\% of the runs.}\label{fig:4_Gaussians_AllRuns_OurAlgorithm}
\end{figure*}

\begin{figure}
    \centering
    \includegraphics[width=0.6\linewidth]{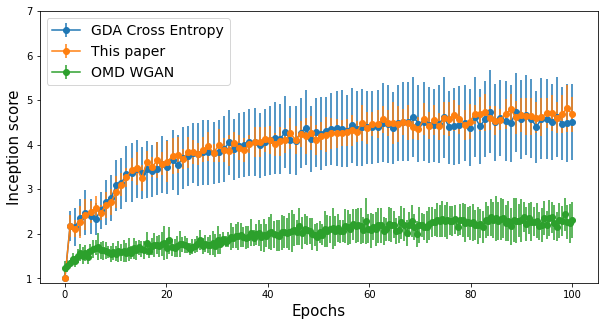}
    \caption{Inception score average (and standard deviation in errorbars) of all methods across iterations. Note that mean inception score of our algorithm is higher than the mean inception score of OMD, while the standard deviation of inception score of our algorithm is lower than the standard deviation of inception score of GDA.}
    \label{fig:inception_vs_iterations}
\end{figure}

\begin{figure*}
\includegraphics[width=\linewidth]{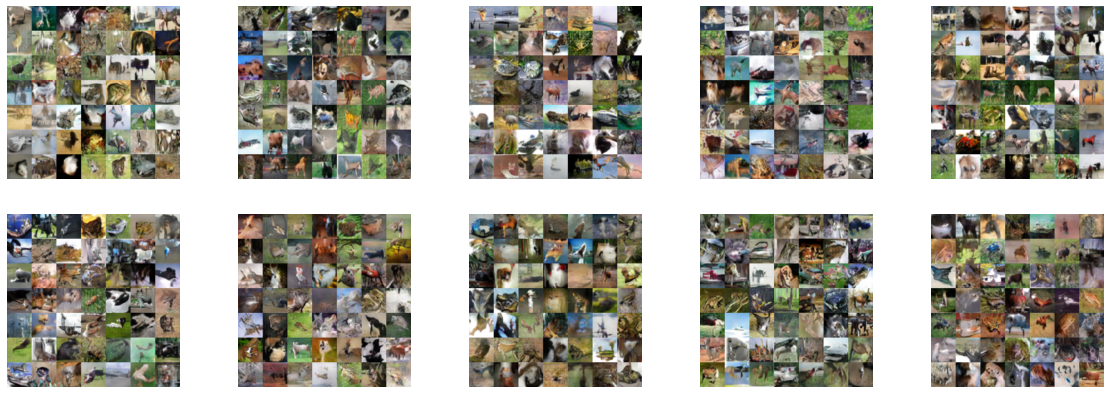}
\caption{\small GAN trained using our algorithm (with $k=1$
  discriminator steps and acceptance rate
  $e^{-\frac{1}{\tau}} = \frac{1}{2}$).  We repeated
  this simulation multiple times; here we display images generated from some of the resulting generators
  for our algorithm.}\label{fig:mx_cifar_samples}
\end{figure*}

\begin{figure*}
\includegraphics[width=\linewidth]{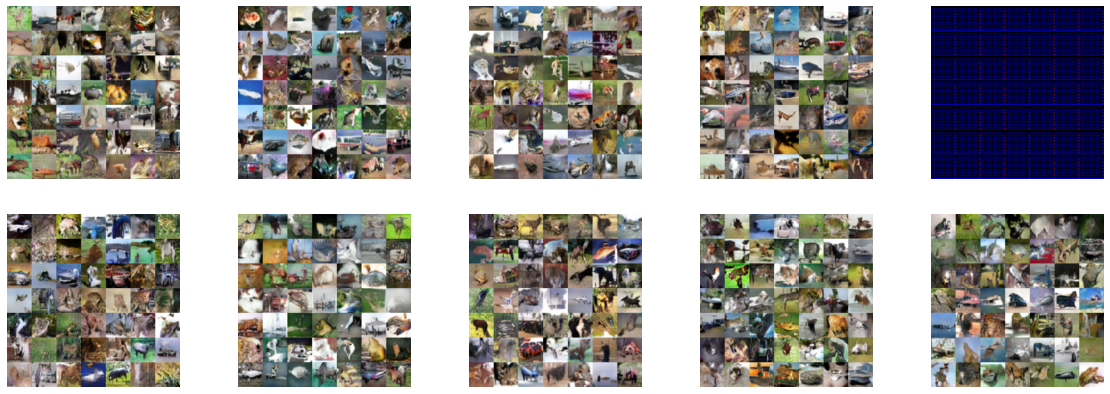}
\caption{\small GAN trained using GDA (with $k=1$
  discriminator steps).  We repeated
  this simulation multiple times; here we display images generated from some of the resulting generators
  for GDA.}\label{fig:gda_cifar_samples}
\end{figure*}

\begin{figure*}
\includegraphics[width=\linewidth]{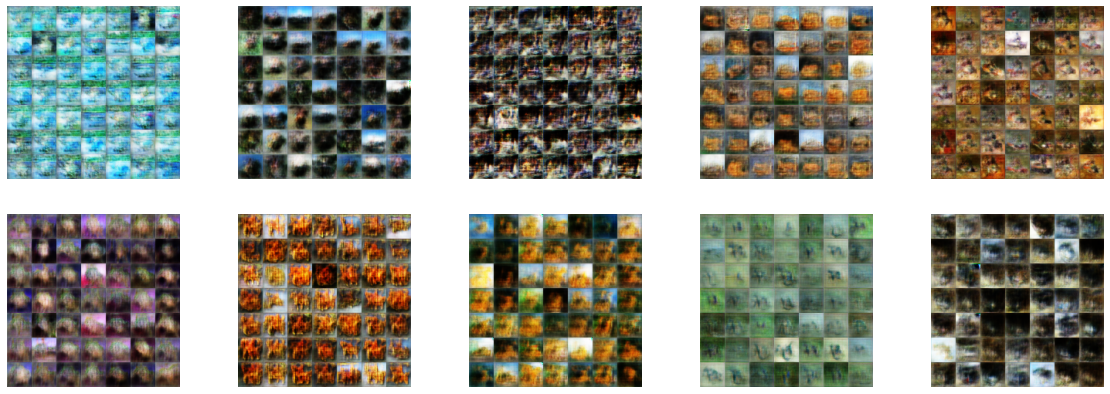}
\caption{\small GAN trained using OMD.  We repeated
  this simulation multiple times; here we display images generated from some of the resulting generators
  for OMD.}\label{fig:omd_cifar_samples}
\end{figure*}

\section{Empirical Results for CIFAR-10 Dataset} \label{CIFAR_results_appendix}

\begin{table}
\begin{center}
\caption{CIFAR-10 dataset: The mean (and standard error) of Inception Scores of models from different training algorithms.
Note that, GDA and our algorithm return generators with similar mean performance; however, the standard error of the Inception Score in case of GDA is relatively larger.
}
\label{tbl:cifar_results}
\small
\begin{tabular}{lcccc}
\toprule
& \multicolumn{4}{c}{Iteration} \\
Method & 5000 & 10000 & 25000 & 50000 \\
\midrule
Ours & 2.71 (0.28) & 3.57 (0.26) & 4.10 (0.35)  & \textbf{4.68} (0.39) \\
GDA & 2.80 (0.52) & 3.56 (0.64) & 4.28 (0.77) & 4.51 (0.86) \\
OMD & 1.60 (0.18) & 1.80 (0.37) & 1.73 (0.25) & 1.96 (0.26) \\
\bottomrule
\end{tabular}
\end{center}
\end{table}
\raggedbottom

\noindent
This real-world dataset contains 60K color images from 10 classes. 
Previous works \cite{borji2019pros, Metz2017unrolled, srivastava2017veegan} have noted that it is challenging to
detect mode collapse on CIFAR-10, visually or using standard metrics such as Inception Scores, because the modes are not well-separated. 
We use this dataset primarily to compare the scalability, quality, and stability of GANs in our framework obtained using our training algorithm. 

For CIFAR-10, in addition to providing images generated by the GANs, we also report the Inception Scores \cite{salimans2018improving} at different iterations. 
Inception Score is a standard heuristic measure for evaluating the quality of CIFAR-10 images and quantifies whether the generated images correspond to specific objects/classes, as well as, whether the GAN generates diverse images.
A higher Inception Score is better, and the lowest possible Inception Score is 1.%

\noindent \paragraph{Hyperparameters for CIFAR-10 simulations.}
      For the CIFAR-10 simulations, we use a batch size of 128, with Adam learning rate of $0.0002$ and hyperparameter $\beta_1=0.5$ for both the generator and discriminator gradients. 
 Our code for the CIFAR-10 simulations is based on the code of Jason Brownlee~\cite{Brownlee}, which originally used gradient descent ascent and ADAM gradients for training.

For the generator we use a neural network with input of size 100 and 4 hidden layers. The first hidden layer consists of a dense layer with $4,096$ parameters, followed by a leaky RELU layer, whose activations are reshaped into $246$ $4 \times 4$ feature maps. The feature maps are then upscaled to an output shape of 32 x 32 via three hidden layers of size 128 each consisting of a convolutional {\em Conv2DTranspose} layer followed by a leaky RELU layer, until the output layer where three filter maps (channels) are created.  Each leaky RELU layer has ``alpha" parameter $0.2$.

For the discriminator, we use a neural network with input of size $32 \times 32 \times 3$ followed by 5 hidden layers.  The first four hidden layers each consist of a convolutional {\em Conv2DTranspose} layer followed by a leaky RELU layer with ``alpha" parameter $0.2$.  The first layer has size 64, the next two layers each have size 128, and the fourth layer has size 256.  The output layer consists of a projection to 1 dimension with dropout regularization of 0.4 and sigmoid activation function.

\noindent \paragraph{Hardware.}
Our simulations on the CIFAR-10 dataset were performed on the above, and using one GPU with High frequency Intel Xeon E5-2686 v4 (Broadwell) processors, provided by AWS.

\noindent \paragraph{Results for CIFAR-10.}
We ran our algorithm (with $k=1$ discriminator steps and acceptance
rate $e^{-\frac{1}{\tau}} = \frac{1}{2}$) on CIFAR-10 for 20 repetitions and 50,000 iterations per repetition. We compare with GDA with $k=1$ discriminator
steps and OMD.
For all algorithms, we compute the Inception Score every 500 iterations; Table~\ref{tbl:cifar_results} reports the Inception Scores at iteration 5000, 10000, 25000, and 50000, while Figure~\ref{fig:inception_vs_iterations} provides the complete plot for Inception Score vs. training iterations.
Sample images from all three algorithms are also provided in Figures \ref{fig:mx_cifar_samples}, \ref{fig:gda_cifar_samples}, \ref{fig:omd_cifar_samples}.

The average Inception Score of GANs from both GDA and our algorithm are fairly close to each other, with the final mean Inception Score of 4.68 for our algorithm being somewhat higher than the final mean of 4.51 for GDA.
However, the standard error of Inception Scores of GDA is much larger than of our algorithm.
The relatively larger standard deviation of GDA is because GDA, in certain runs, does not learn an appropriate distribution at all (Inception Score is close to 1 throughout training in this case), leading to a larger value of standard deviation.
Visually, in these GDA runs, the GANs from GDA  do not generate recognizable images (Figure \ref{fig:gda_cifar_samples}, top-right image).
For all other trials, the images generated by GDA have similar Inception Score (and similar quality) as the images generated by our algorithm.
In other words, 
our algorithm seems to be more stable 
than GDA and returns GANs that generate high quality images in every repetition.

GANs trained using OMD  attain much lower Inception Scores than our algorithm.\footnote{We could not replicate the performance of OMD reported in \cite{Daskalakis2018optimism}, even with the  implementation provided here - \url{https://github.com/vsyrgkanis/optimistic_GAN_training}.}
Moreover, the images generated by GANs trained using OMD have visually much lower quality than the images generated by GANs trained using our algorithm (Figure \ref{fig:omd_cifar_samples}).

Evaluation on CIFAR-10 dataset shows that the GANs from our training algorithm can always generate good quality images; in comparison to OMD, the GANs trained using our algorithm generate higher quality images, while in comparison to GDA, it is relatively more stable.

\paragraph{Clock time per iteration.}
When training on CIFAR-10, our algorithm and GDA both took the same amount of time per iteration, 0.08 seconds, on the AWS GPU server.

We evaluate our algorithm on MNIST dataset as well, where it also learns to generate from multiple modes; 
the results are presented in Appendix~\ref{sec:mnist_results}.

\begin{figure*}[t]
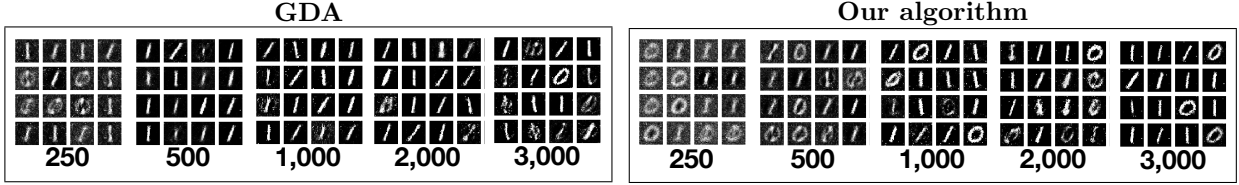

    \begin{minipage}{0.495\textwidth}
    \begin{center}
      \textbf{\small  GDA} \\
      \vspace{0.4mm}
    \fbox{\includegraphics[width=0.96\textwidth]{images/9_Vanilla_MNIST_small}}
  \end{center}
  \end{minipage}
    \begin{minipage}{0.495\textwidth}
    \begin{center}
    \textbf{\small Our algorithm} \\
    \fbox{\includegraphics[trim={0, 0, 0, 0}, clip, width=0.96\textwidth]{images/3_Picky_MNIST_small}}
  \end{center}
  \end{minipage}
  \caption{\small We trained a GAN using our algorithm on 0-1 MNIST for 30,000
    iterations (with $k=1$ discriminator steps and acceptance rate
    $e^{-\frac{1}{\tau}} = \frac{1}{5}$).  We repeated this experiment
    22 times for our algorithm and 13 times for  GDA.  Shown
    here are the images generated from one of these runs at various
    iterations for our algorithm (right) and  GDA (left).}
  \label{fig:MNIST_Picky_Vs_Vanilla}
\end{figure*}

\section{Empirical Results for MNIST Dataset} \label{sec:mnist_results}

This dataset consists of 60k images of hand-written digits \cite{lecun2010mnist}. 
We use two versions of this dataset: the full dataset and the dataset restricted to 0-1 digits.

\noindent \paragraph{Hyperparameters for MNIST simulations.}
For the MNIST simulations, we use a batch size of 128, with Adam learning rate of $0.0002$ and hyperparameter $\beta_1=0.5$ for both the generator and discriminator gradients. 
 Our code for the MNIST simulations is based on the code of Renu Khandelwal~\cite{Khandelwal} and Rowel Atienza~\cite{Atienza}, which originally used gradient descent ascent and ADAM gradients for training.

For the generator we use a neural network with input of size 256 and 3 hidden layers, with leaky RELUS each with ``alpha" parameter $0.2$ and dropout regularization of 0.2 at each layer. 
 The first layer has size 256, the second layer has size  512, and the third layer has size 1024, followed by an output layer with hyperbolic tangent (``tanh") acvtivation.

For the discriminator we use a neural network with 3 hidden layers, and leaky RELUS each with ``alpha" parameter $0.2$, and dropout regularization of 0.3 (for the first two layers) and 0.2 (for the last layer). 
 The first layer has size 1024, the second layer has size  512, the third layer has size 256, and the hidden layers are followed by a projection to 1 dimension with sigmoid activation (which is fed as input to the cross entropy loss function).

\noindent \paragraph{Results for 0-1 MNIST.}
We trained GANs using both GDA and our algorithm on the 0-1 MNIST dataset, and ran each algorithm for 3000 iterations (Figure~\ref{fig:MNIST_Picky_Vs_Vanilla}).   GDA
seems to briefly generate shapes that look like a combination of 0's
and 1's, then switches to generating only 1's, and then re-learns how
to generate 0's.  In contrast, our algorithm seems to learn how to
generate both 0's and 1's early on and
does not mode collapse to either  digit.
(See Figure \ref{fig_01MNISTGDA} for images generated by all the runs of GDA, and Figure \ref{fig_01MNISTOur} for images generated by the GAN for all the runs of our algorithm.)

\begin{figure*}
\fbox{\includegraphics[trim={0 13cm 13cm 0},clip, scale=0.125]{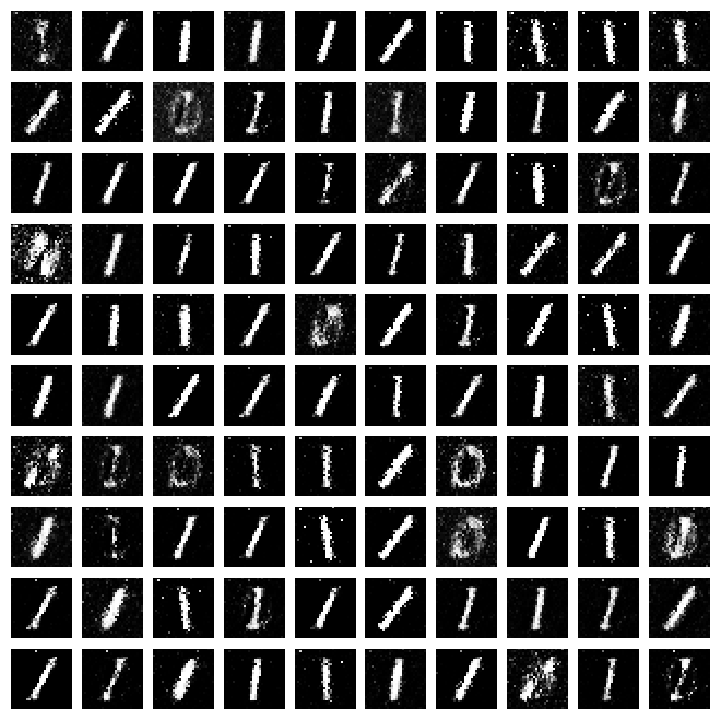}}
\fbox{\includegraphics[trim={0 13cm 13cm 0},clip, scale=0.125]{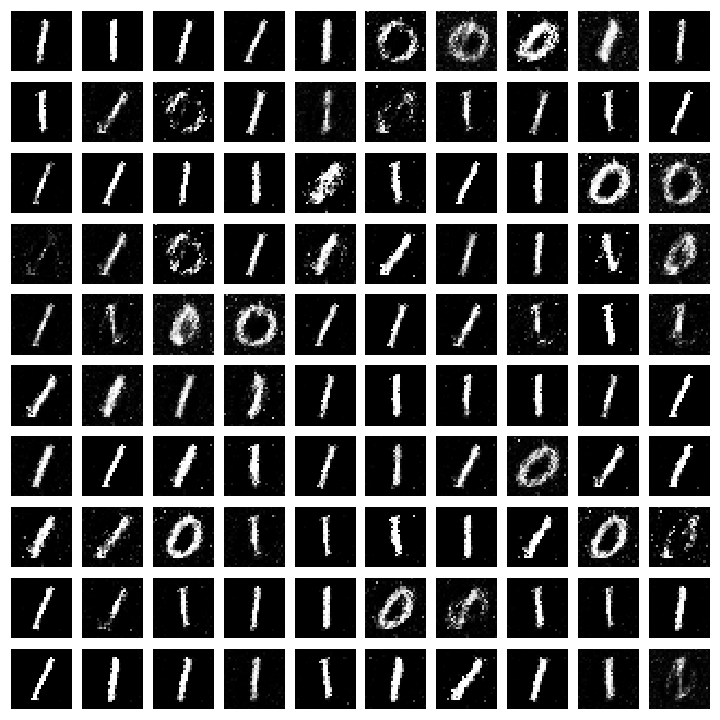}}
\fbox{\includegraphics[trim={0 13cm 13cm 0},clip, scale=0.125]{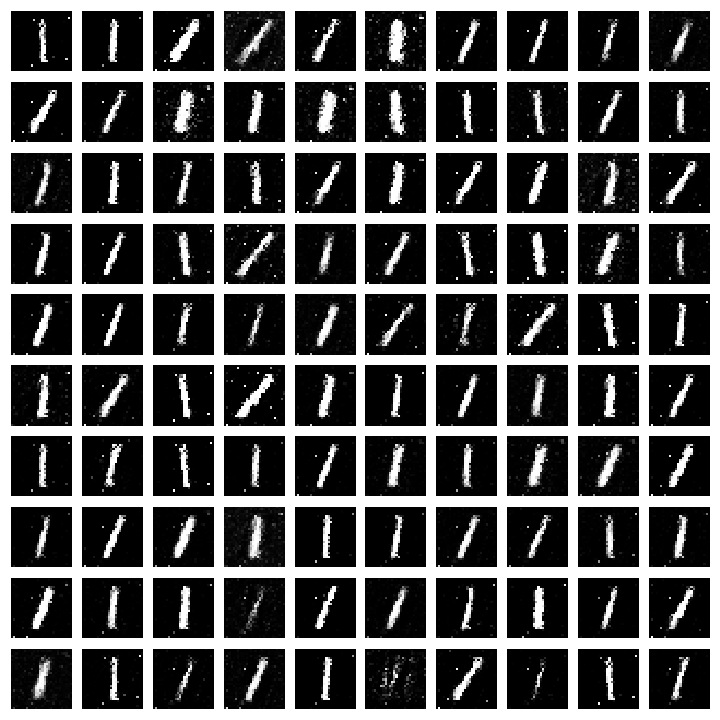}}
\fbox{\includegraphics[trim={0 13cm 13cm 0},clip, scale=0.125]{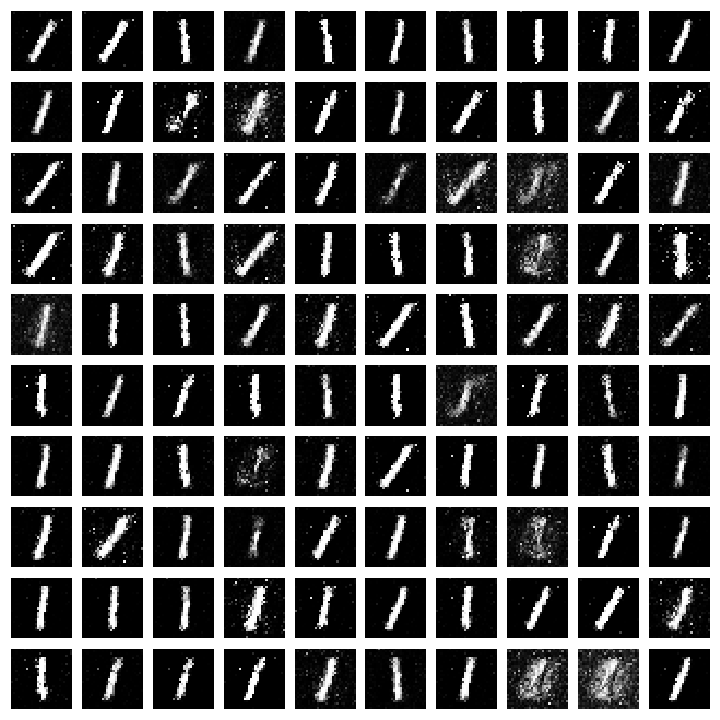}}
\fbox{\includegraphics[trim={0 13cm 13cm 0},clip, scale=0.125]{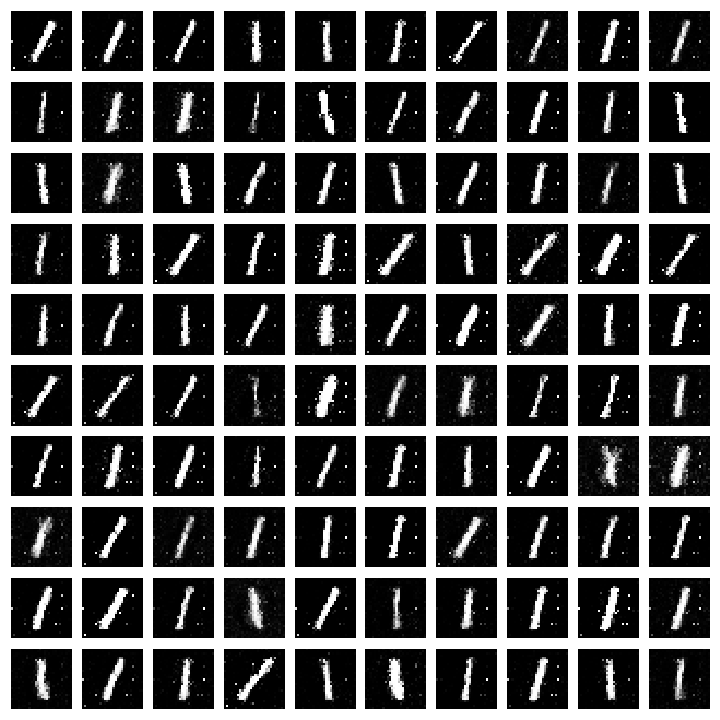}}
\fbox{\includegraphics[trim={0 13cm 13cm 0},clip, scale=0.125]{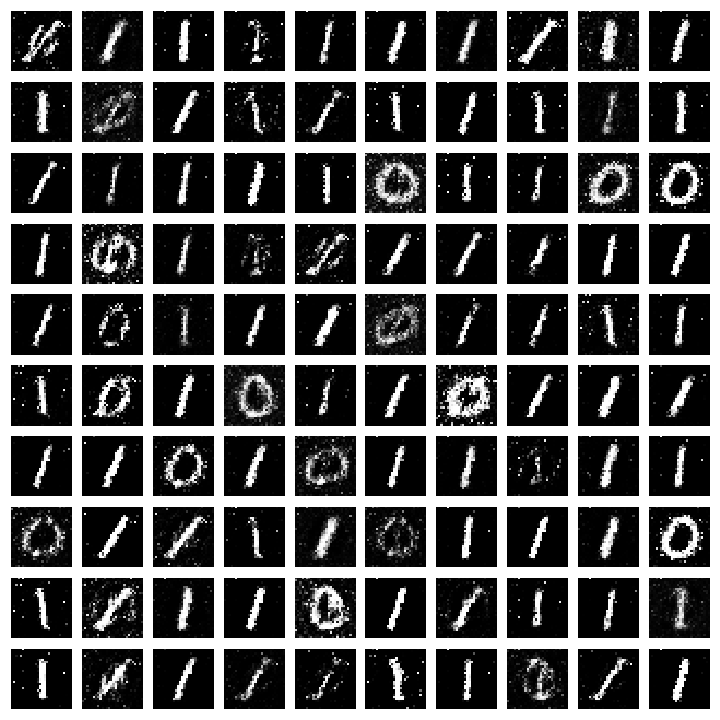}}
\fbox{\includegraphics[trim={0 13cm 13cm 0},clip, scale=0.125]{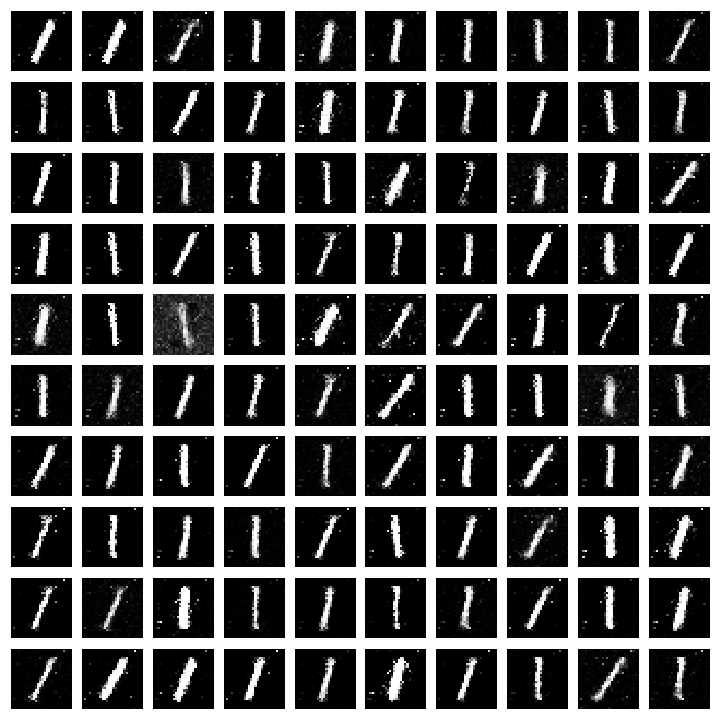}}
\noindent \fbox{\includegraphics[trim={0 13cm 13cm 0},clip, scale=0.125]{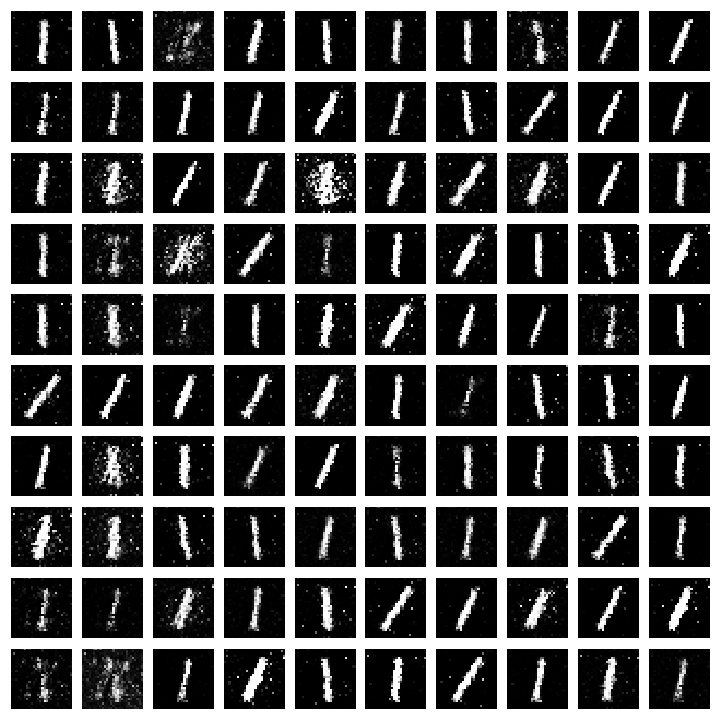}}
\fbox{\includegraphics[trim={0 13cm 13cm 0},clip, scale=0.125]{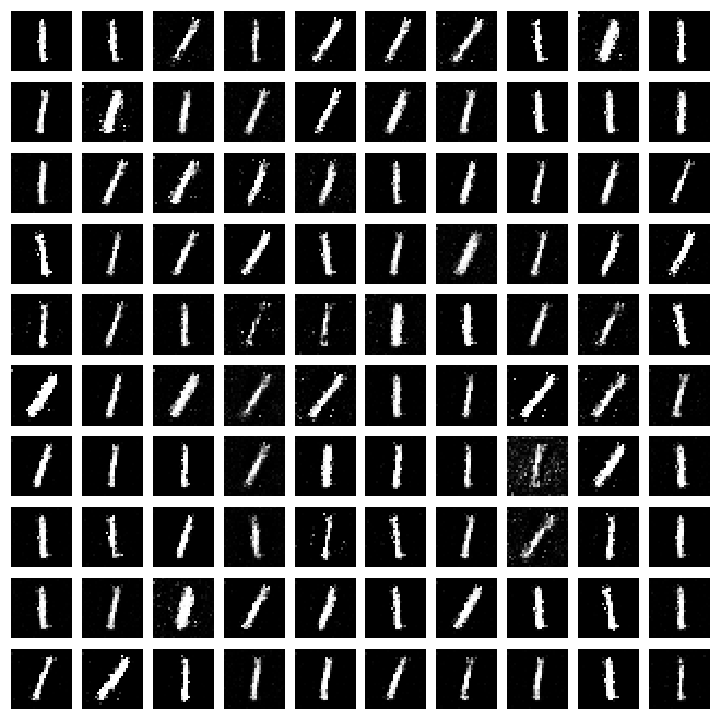}}
\fbox{\includegraphics[trim={0 13cm 13cm 0},clip, scale=0.125]{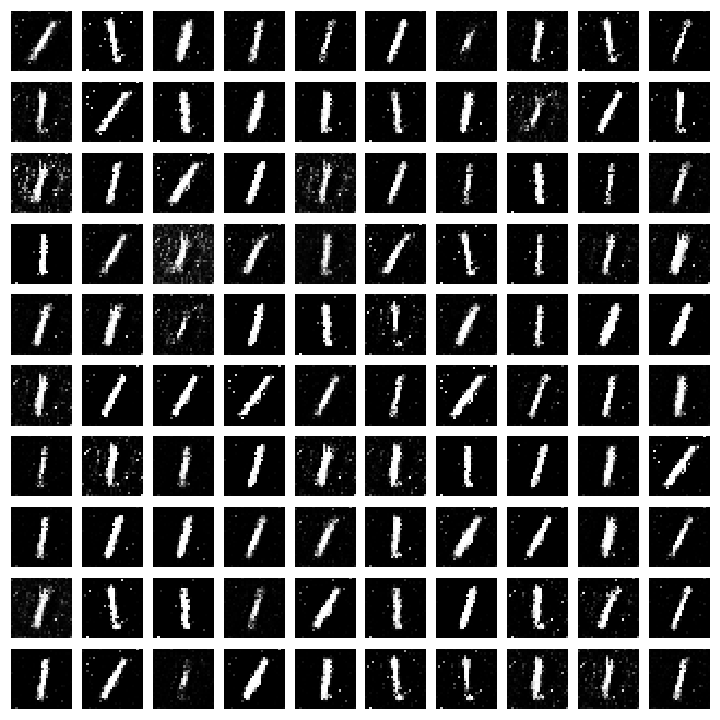}}
\fbox{\includegraphics[trim={0 13cm 13cm 0},clip, scale=0.125]{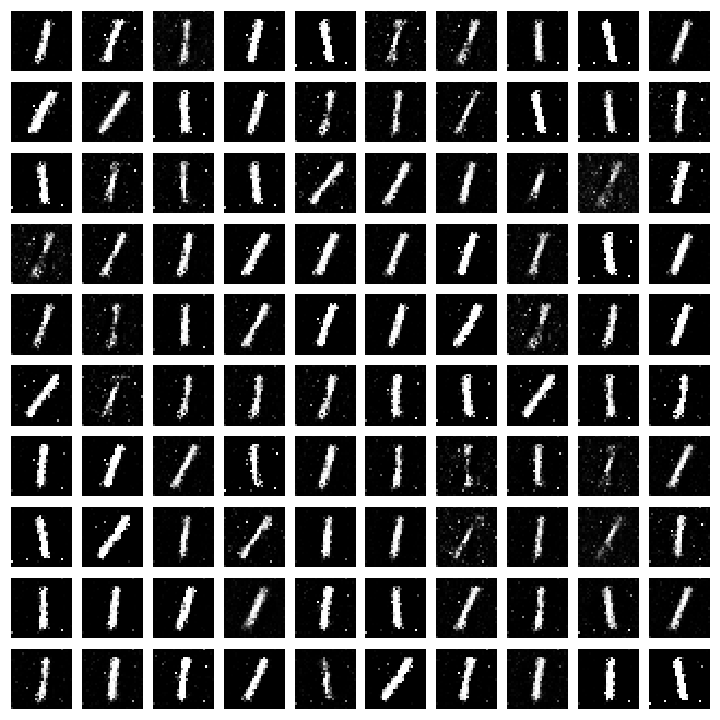}}
\fbox{\includegraphics[trim={0 13cm 13cm 0},clip, scale=0.125]{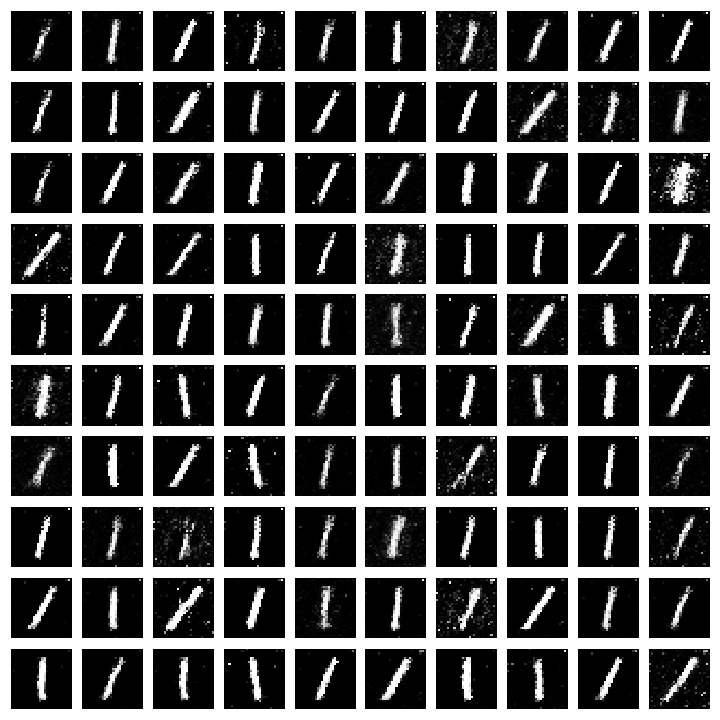}}
\fbox{\includegraphics[trim={0 13cm 13cm 0},clip, scale=0.125]{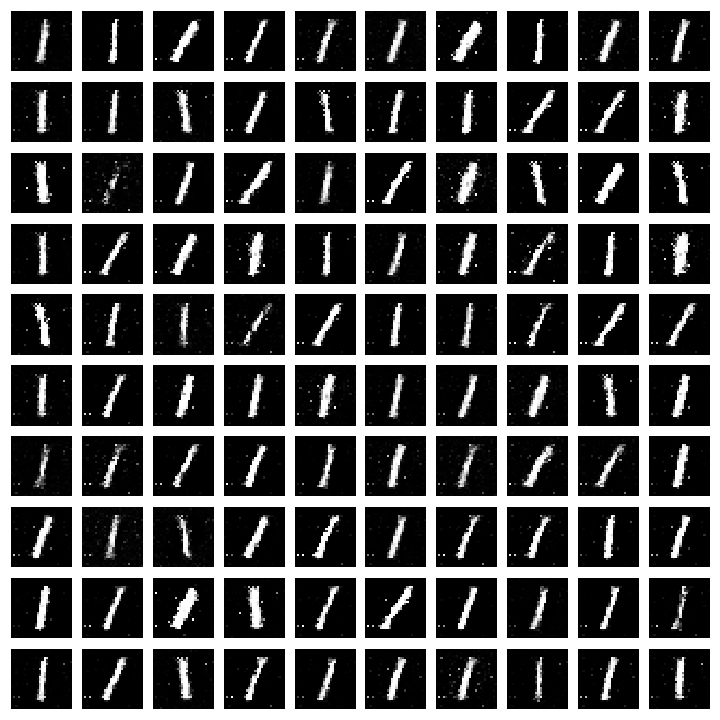}}
\caption{\small Images generated at the 1000'th iteration of the 13 runs of the GDA simulation mentioned in Figure \ref{fig:MNIST_Picky_Vs_Vanilla}.  In 77\% of the runs the generator seems to be generating only 1's at the 1000'th iteration.} \label{fig_01MNISTGDA}
\end{figure*}

\begin{figure*}
\fbox{\includegraphics[trim={0 13cm 13cm 0},clip, scale=0.125]{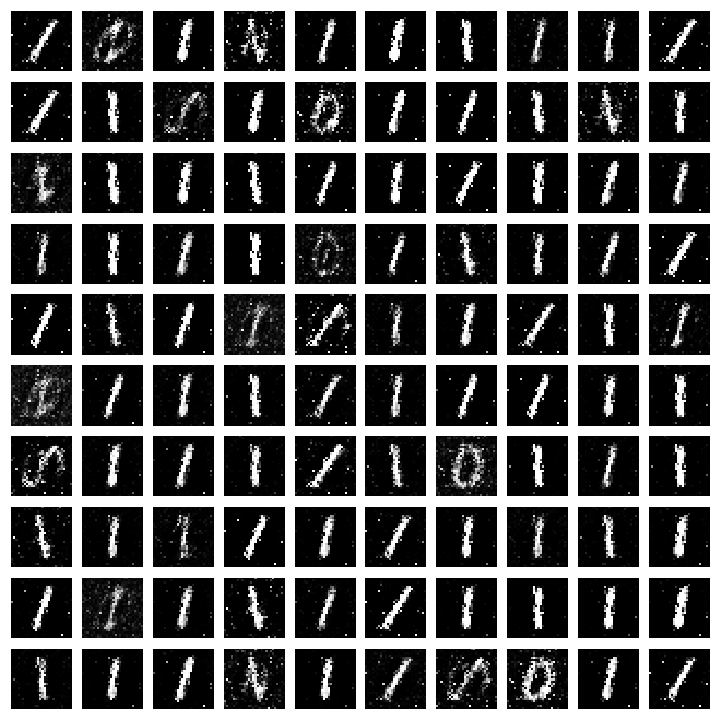}}
\fbox{\includegraphics[trim={0 13cm 13cm 0},clip, scale=0.125]{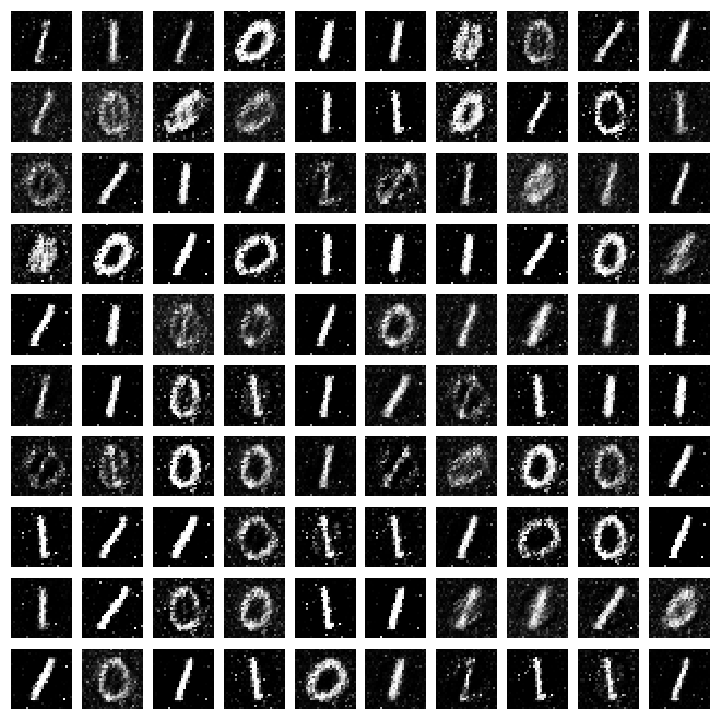}}
\fbox{\includegraphics[trim={0 13cm 13cm 0},clip, scale=0.125]{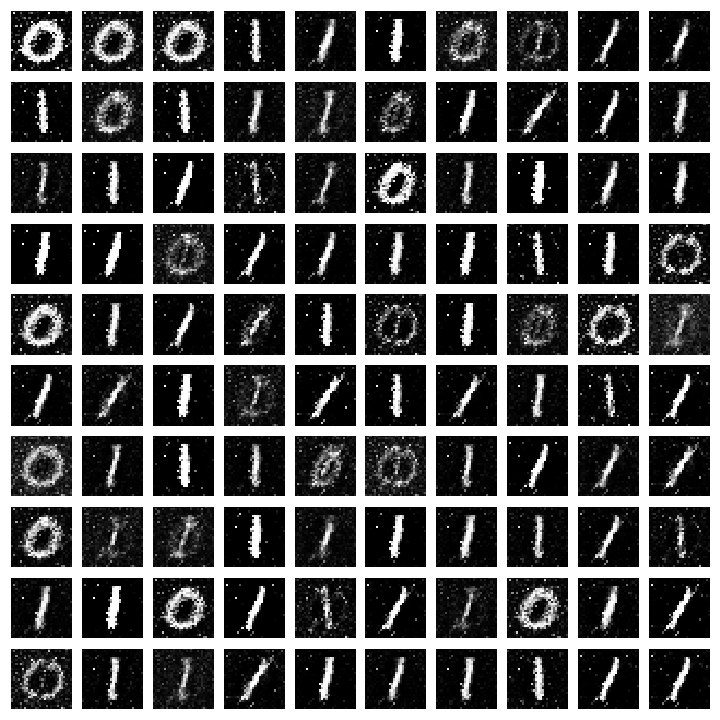}}
\fbox{\includegraphics[trim={0 13cm 13cm 0},clip, scale=0.125]{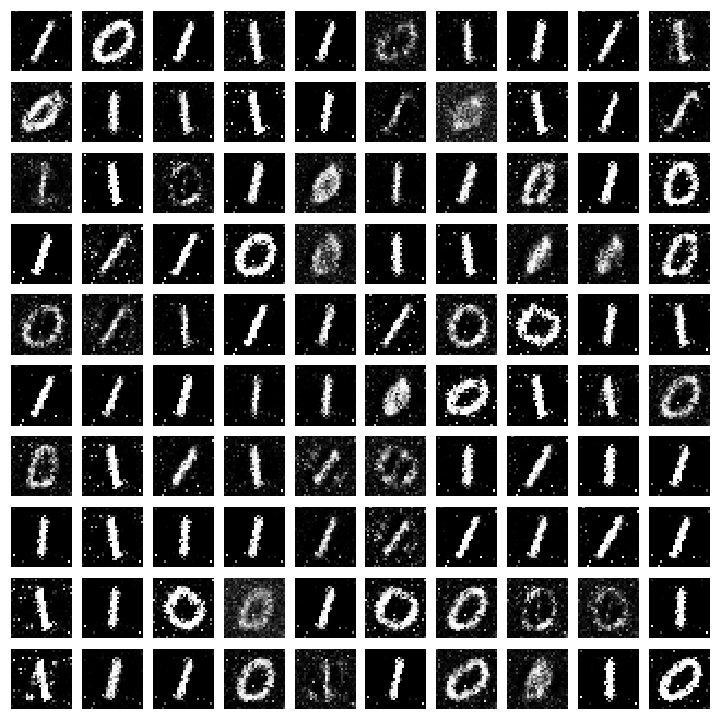}}
\fbox{\includegraphics[trim={0 13cm 13cm 0},clip, scale=0.125]{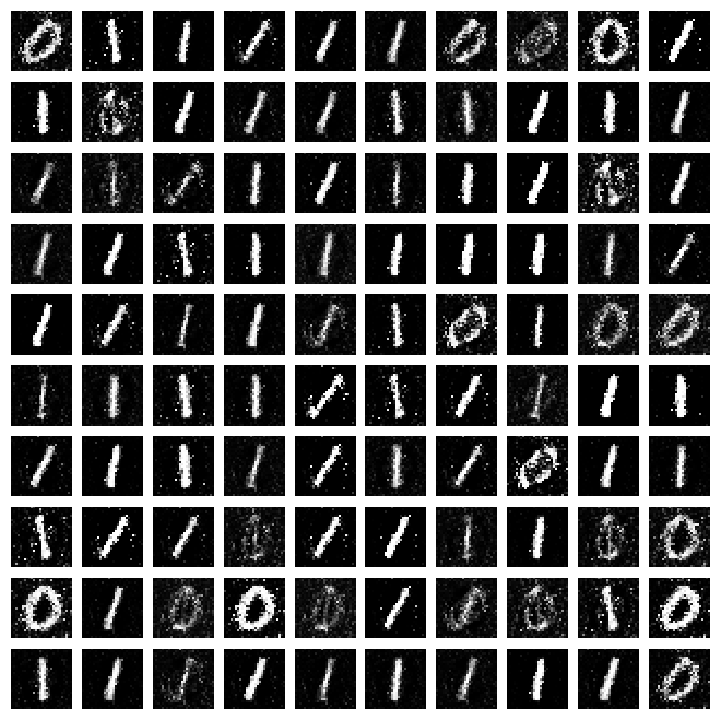}}
\fbox{\includegraphics[trim={0 13cm 13cm 0},clip, scale=0.125]{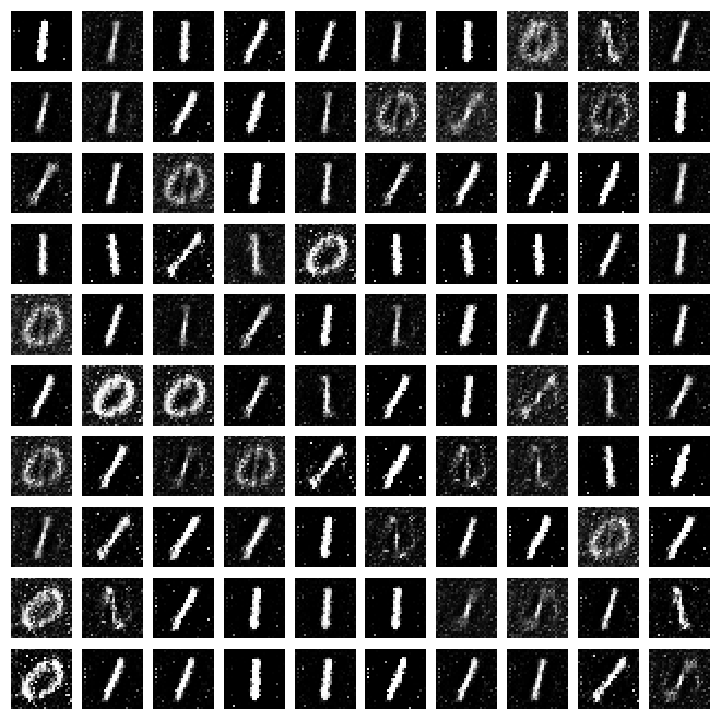}}
\fbox{\includegraphics[trim={0 13cm 13cm 0},clip, scale=0.125]{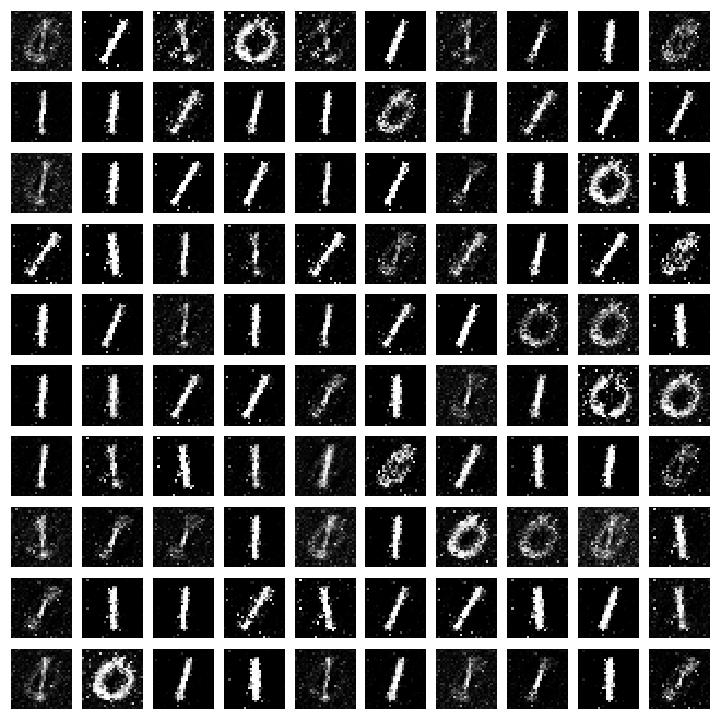}}
\fbox{\includegraphics[trim={0 13cm 13cm 0},clip, scale=0.125]{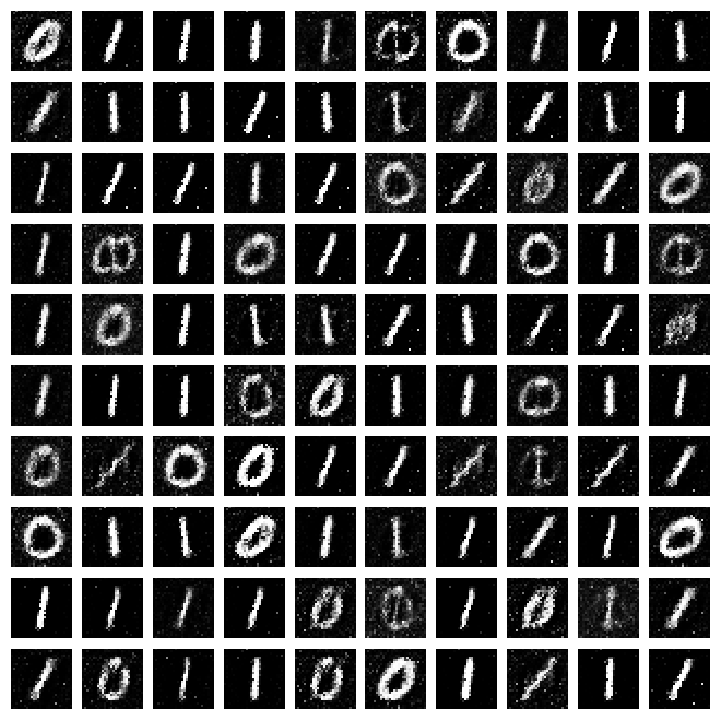}}
\fbox{\includegraphics[trim={0 13cm 13cm 0},clip, scale=0.125]{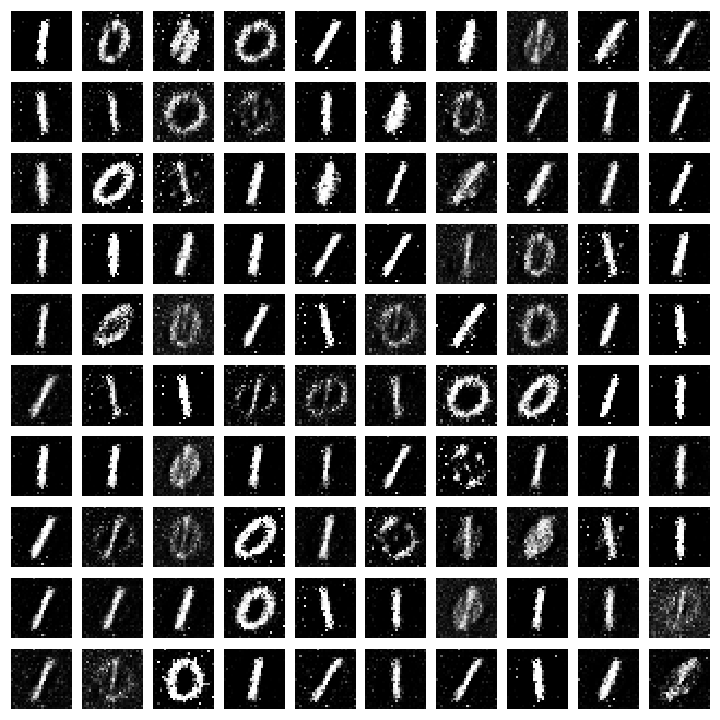}}
\fbox{\includegraphics[trim={0 13cm 13cm 0},clip, scale=0.125]{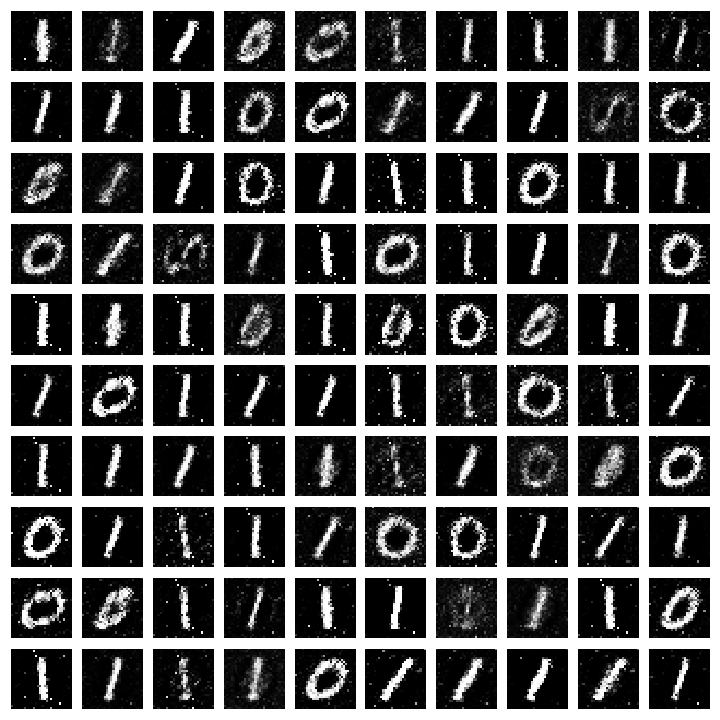}}
\fbox{\includegraphics[trim={0 13cm 13cm 0},clip, scale=0.125]{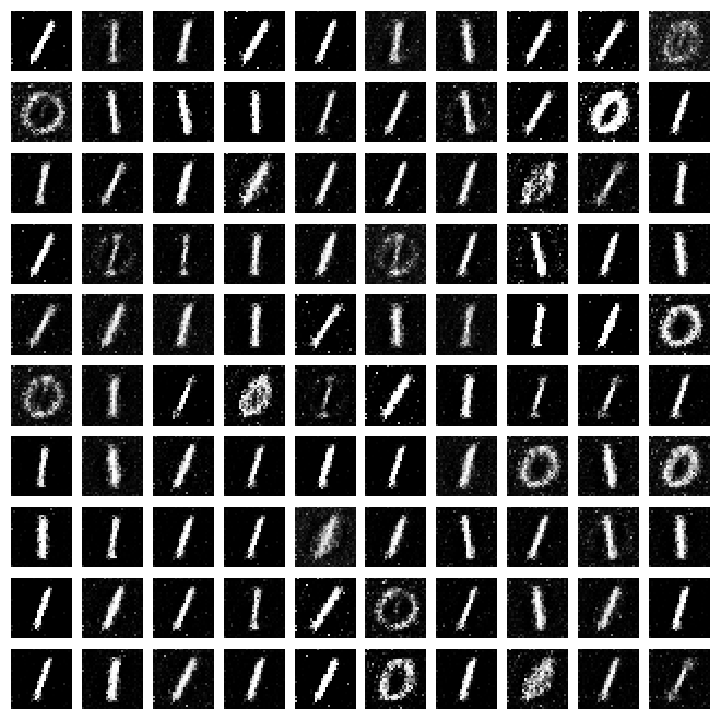}}
\fbox{\includegraphics[trim={0 13cm 13cm 0},clip, scale=0.125]{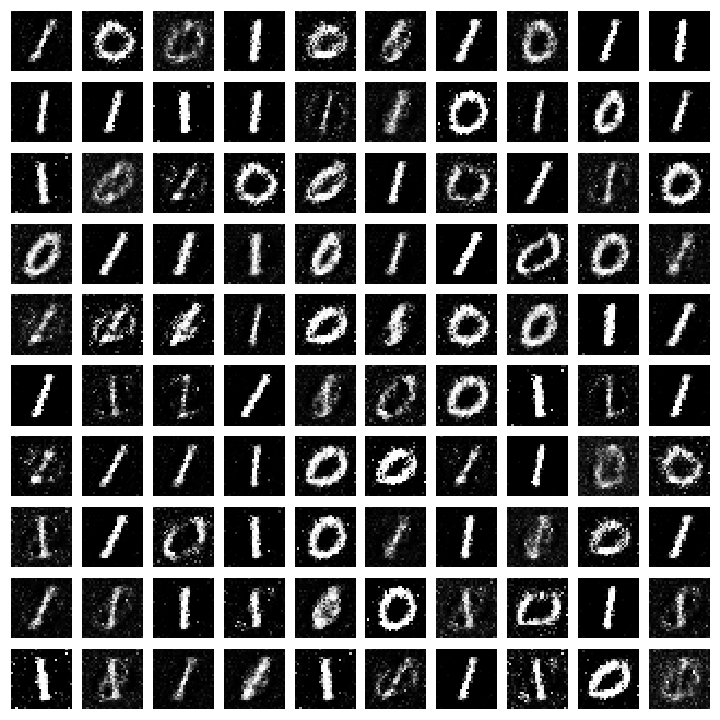}}
\fbox{\includegraphics[trim={0 13cm 13cm 0},clip, scale=0.125]{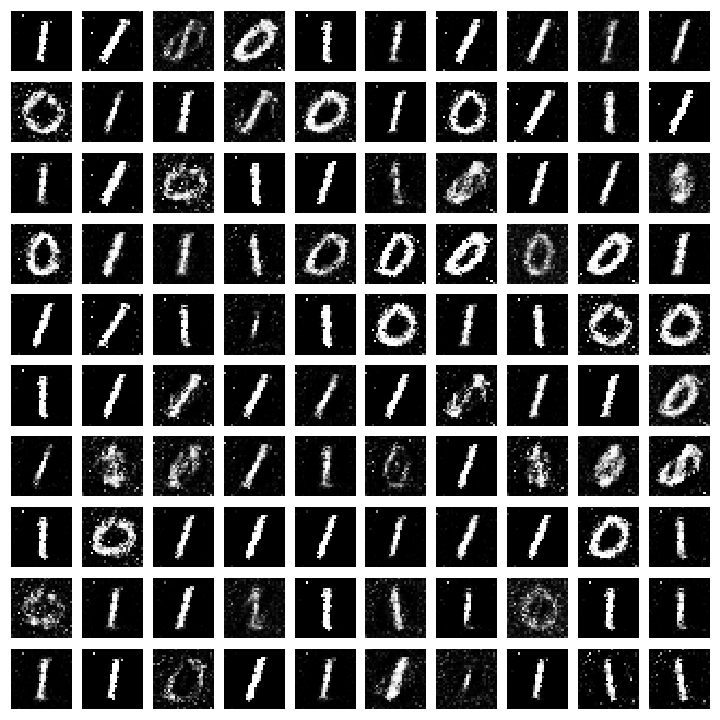}}
\fbox{\includegraphics[trim={0 13cm 13cm 0},clip, scale=0.125]{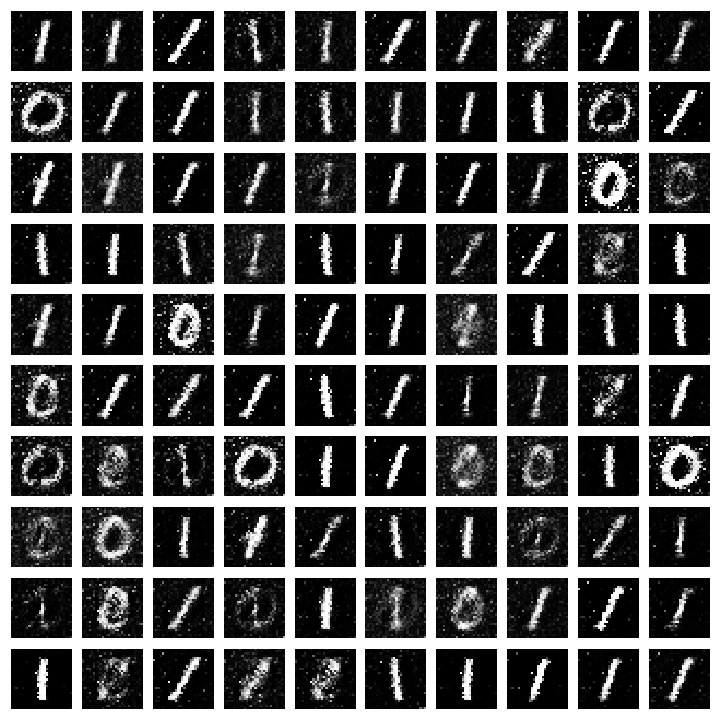}}
\fbox{\includegraphics[trim={0 13cm 13cm 0},clip, scale=0.125]{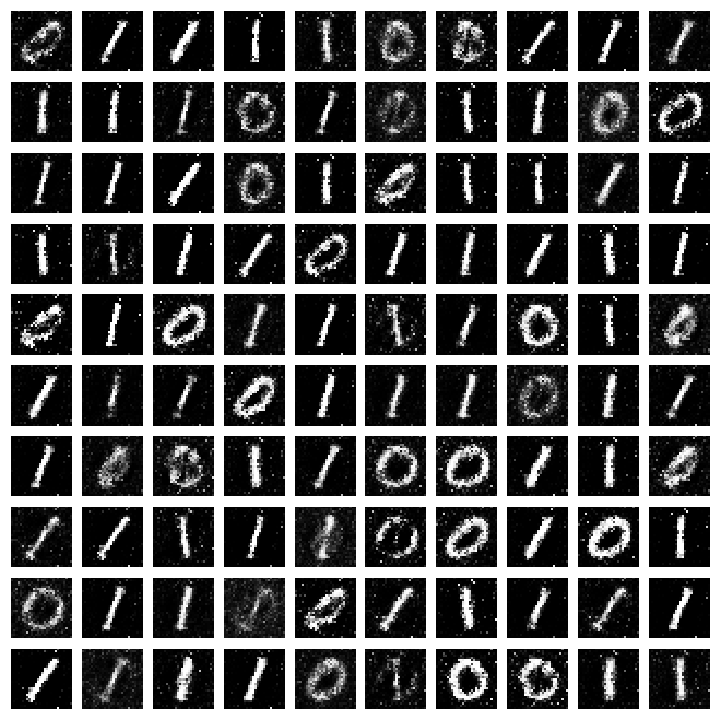}}
\fbox{\includegraphics[trim={0 13cm 13cm 0},clip, scale=0.125]{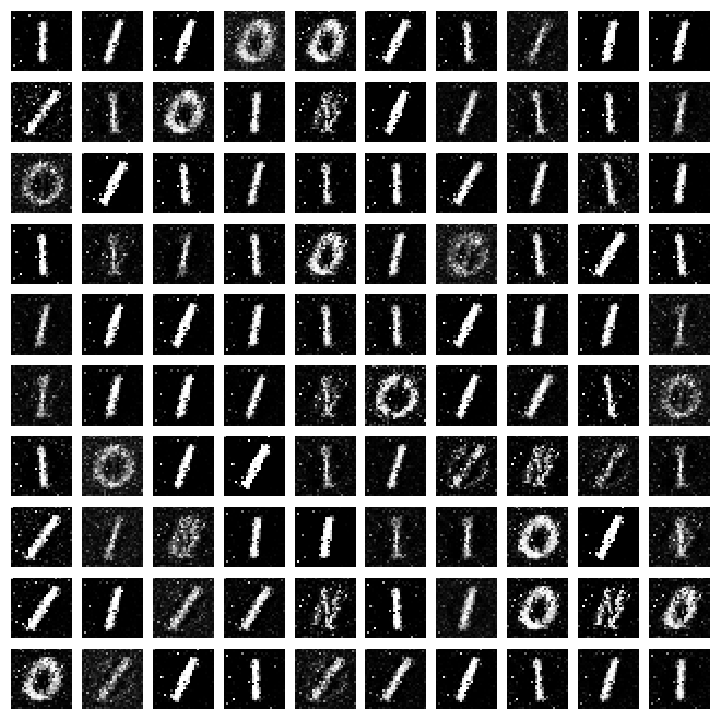}}
\fbox{\includegraphics[trim={0 13cm 13cm 0},clip, scale=0.125]{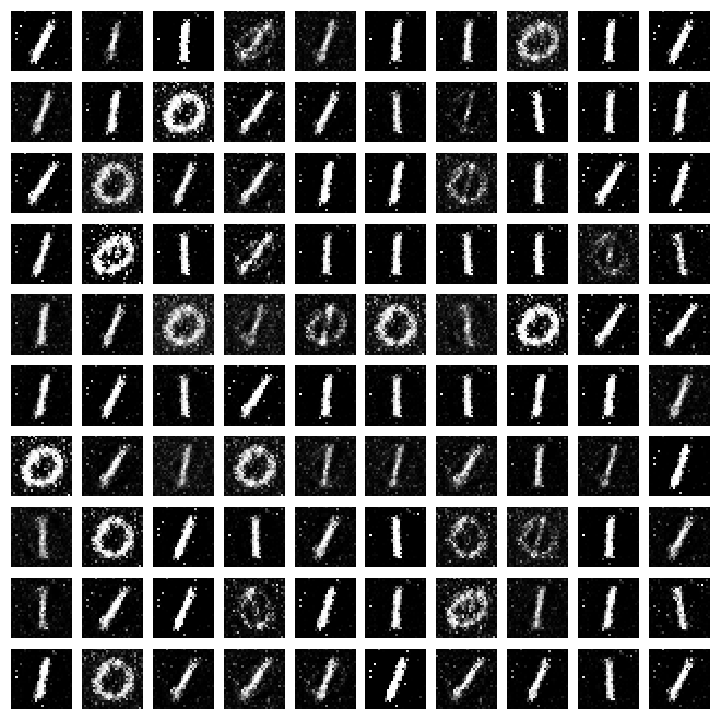}}
\fbox{\includegraphics[trim={0 13cm 13cm 0},clip, scale=0.125]{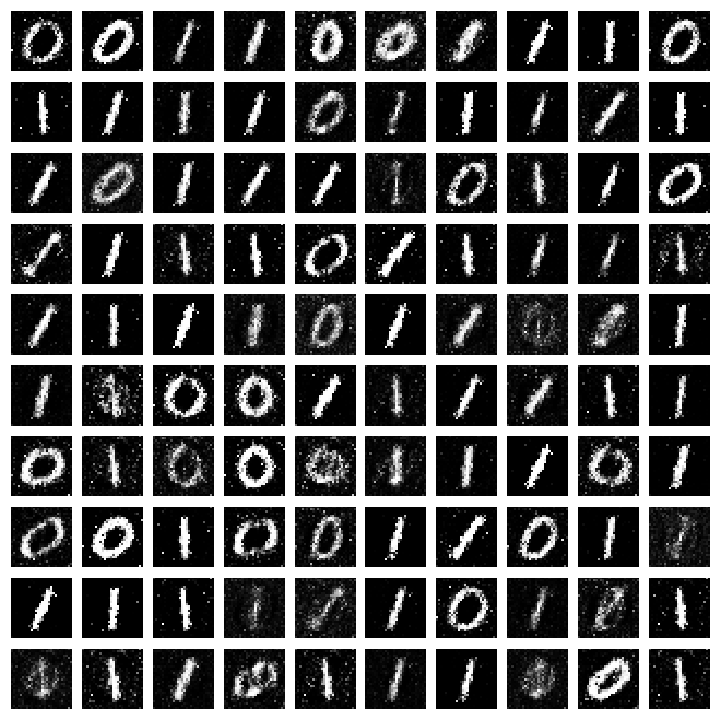}}
\fbox{\includegraphics[trim={0 13cm 13cm 0},clip, scale=0.125]{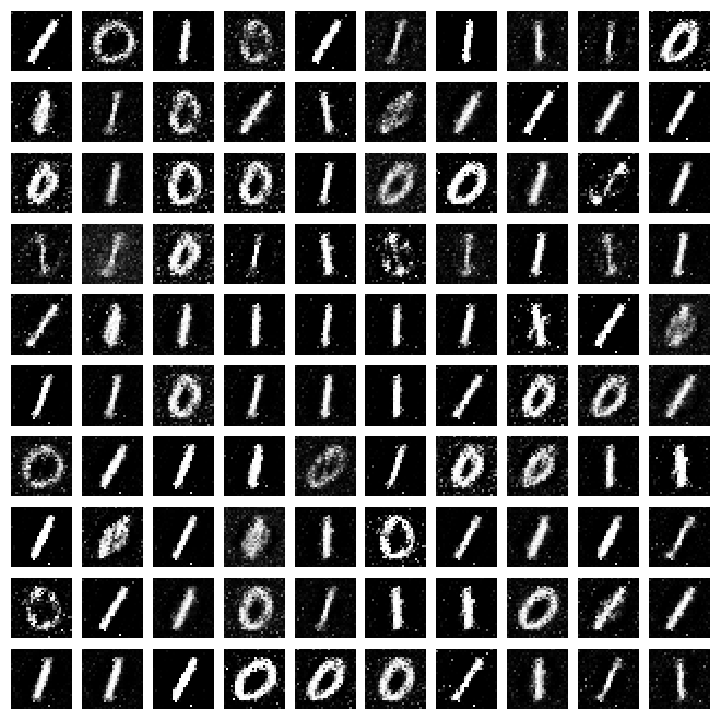}}
\fbox{\includegraphics[trim={0 13cm 13cm 0},clip, scale=0.125]{images/17_Picky_MNIST_1000.png}}
\fbox{\includegraphics[trim={0 13cm 13cm 0},clip, scale=0.125]{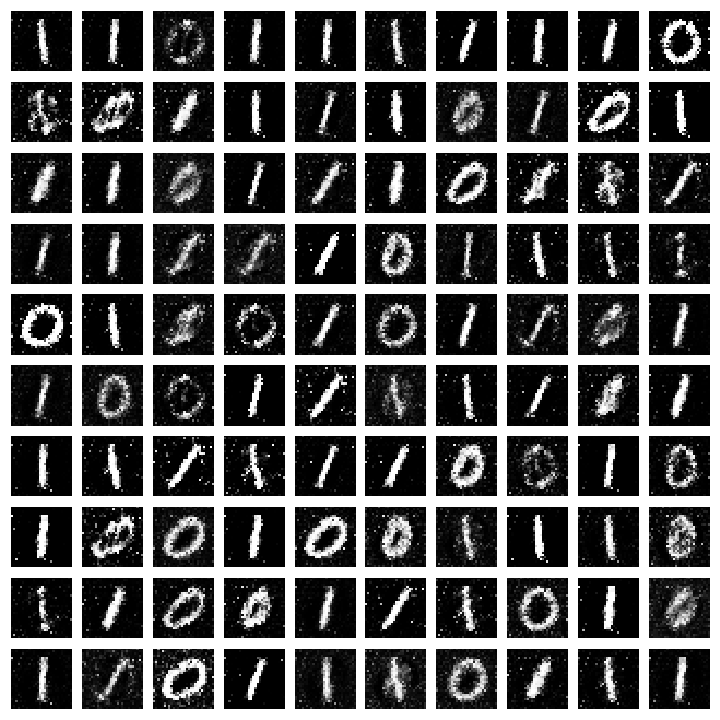}}
\fbox{\includegraphics[trim={0 13cm 13cm 0},clip, scale=0.125]{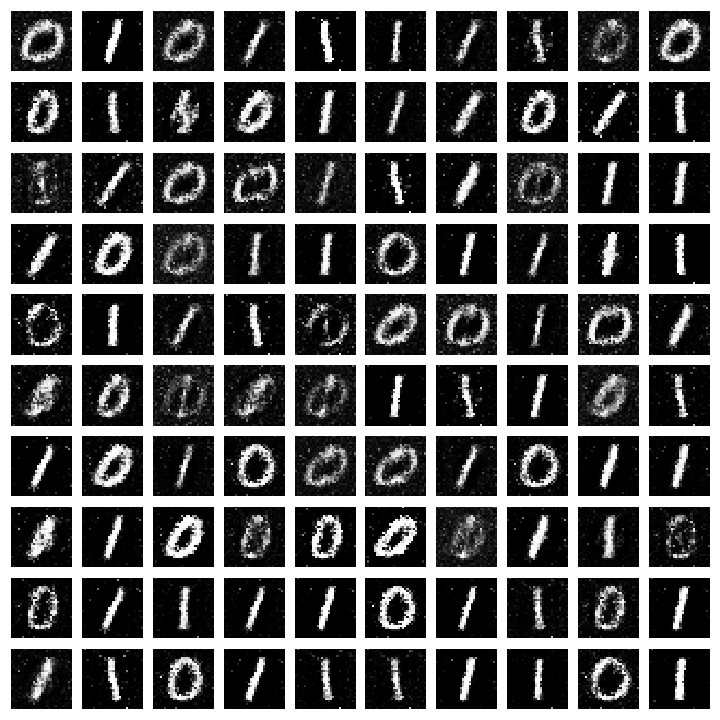}}
\caption{\small Images generated at the 1000'th iteration of each of the 22 runs of our algorithm for the simulation mentioned in Figure \ref{fig:MNIST_Picky_Vs_Vanilla}.}\label{fig_01MNISTOur}
\end{figure*}

\noindent \paragraph{Full MNIST.} Next we evaluate the utility of our algorithm on the full MNIST dataset.
We trained a GAN on the full MNIST dataset using our algorithm
for 39,000 iterations (with $k=1$ discriminator steps and acceptance
rate $e^{-\frac{1}{\tau}} = \frac{1}{5}$).  We ran this simulation five times; each time the GAN learned to generate all ten
digits (see Fig. \ref{fig:MNIST_Picky_tenDigits_AllRuns} for generated
images).

\begin{figure*}
\fbox{\includegraphics[scale=0.12]{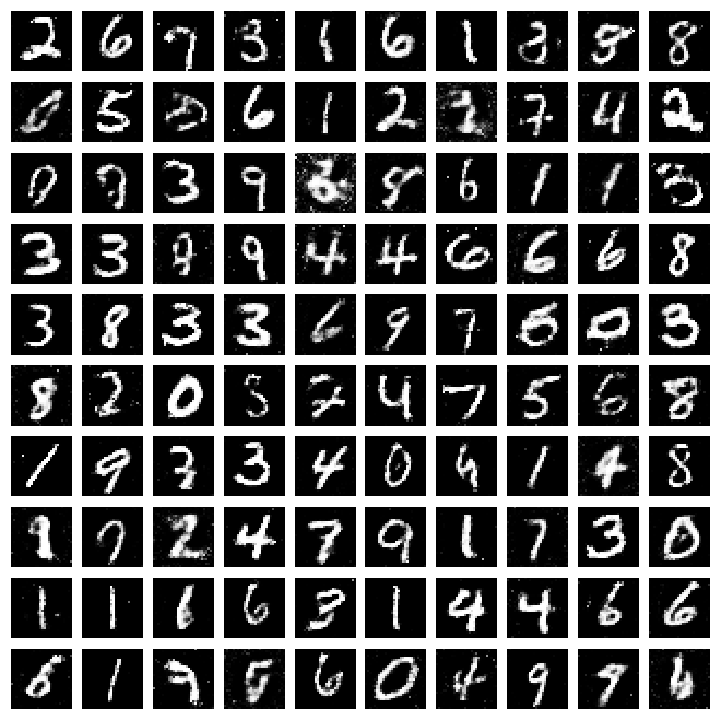}}
\fbox{\includegraphics[scale=0.12]{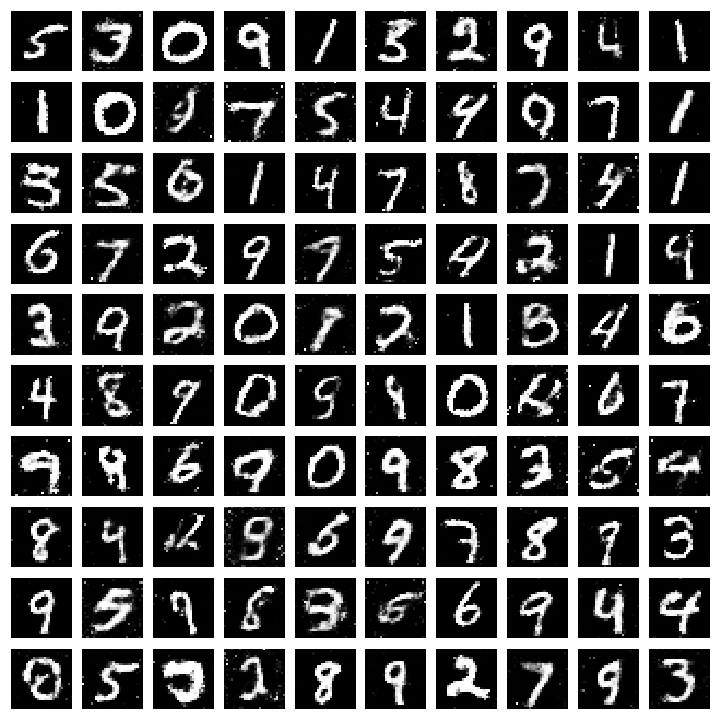}}
\fbox{\includegraphics[scale=0.12]{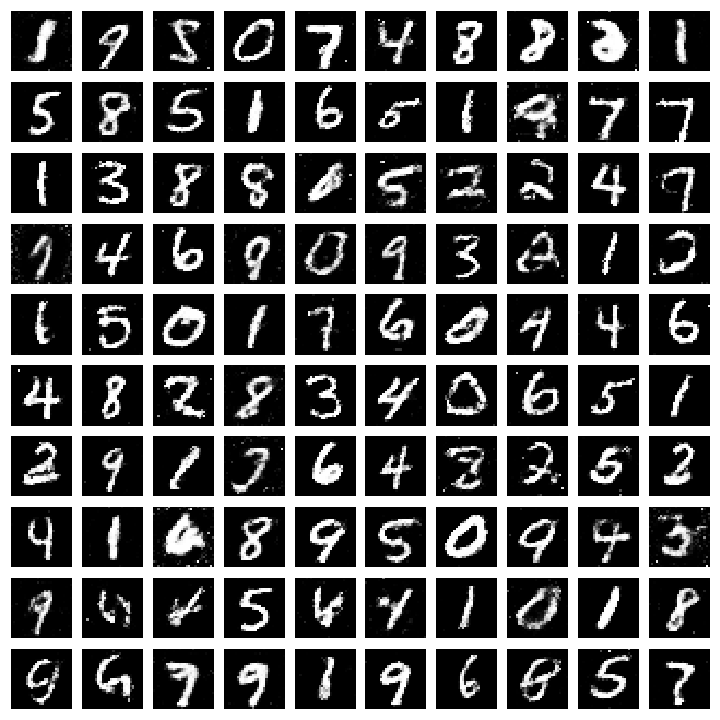}}
\fbox{\includegraphics[scale=0.12]{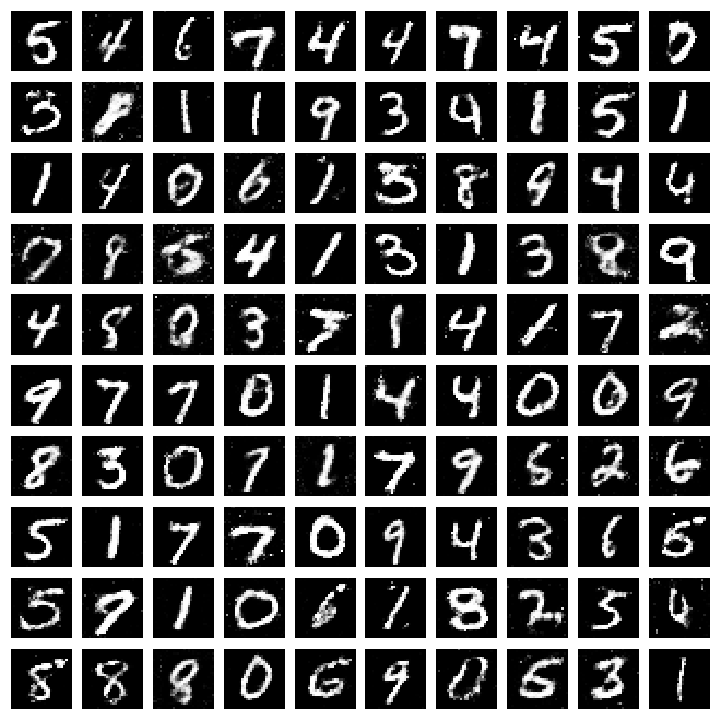}}
\fbox{\includegraphics[scale=0.12]{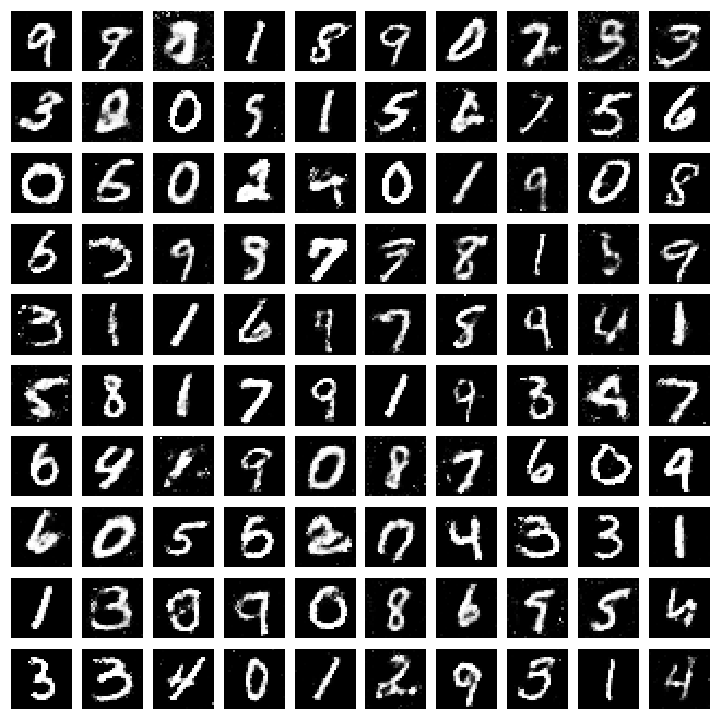}}
\caption{\small We ran our algorithm (with $k=1$ discriminator steps and acceptance rate $e^{-\frac{1}{\tau}} = \frac{1}{5}$) on the full MNIST dataset for 39,000 iterations, and then plotted images generated from the resulting generator.  We repeated this simulation five times; the generated images from each of the five runs are shown here.} \label{fig:MNIST_Picky_tenDigits_AllRuns}
\end{figure*}

\section{Randomized Acceptance Rule with Decreasing Temperature}
\label{sec:More_simulations_annealing}
In this section we give the simulations mentioned in the paragraph  towards the beginning of Section \ref{sec:experiments}, which discusses simplifications to our algorithm.  We included these simulations to verify that our algorithm also works well when it is implemented using a randomized acceptance rule with a decreasing temperature schedule (Figure \ref{fig:MNIST_annealing}).

\begin{figure*}
\centering
\noindent \includegraphics[scale=0.38]{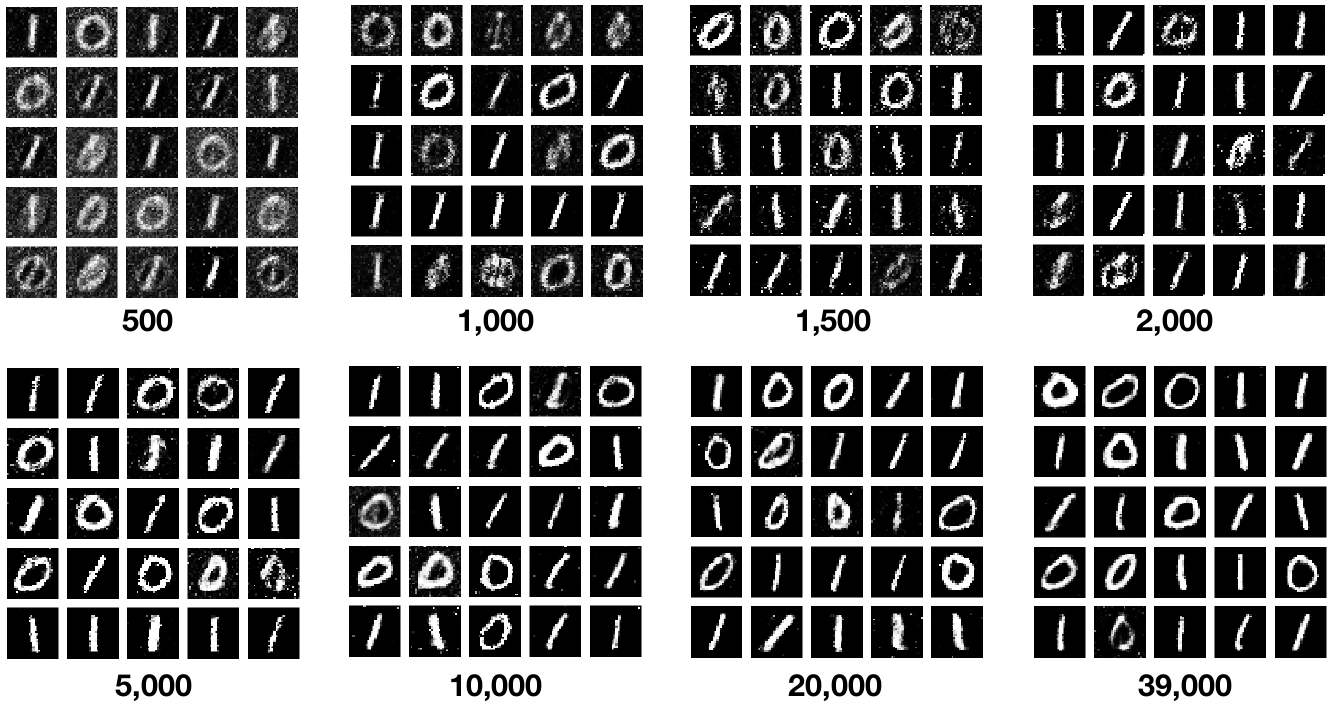}
\caption{\small 
In this simulation we used a randomized accept/reject rule, with a decreasing temperature schedule.  The algorithm was run for 39,000 iterations, with a temperature schedule of $e^{-\frac{1}{\tau_{i}}} = \frac{1}{4 + e^{(i/20000)^2}}$.  Proposed steps which decreased the computed value of the loss function were accepted with probability $1$, and proposed steps which increased the computed value of the loss function were rejected with probability $\max(0,1-e^{-\frac{i}{\tau_1}})$ at each iteration $i$.  We ran the simulation 5 times, and obtained similar results each time, with the generator learning both modes.  In this figure, we plotted the generated images from one of the runs at various iterations, with the iteration number specified at the bottom of each figure (see also Figure \ref{fig:MNIST_annealing_AllRuns} for results from the other four runs)
}\label{fig:MNIST_annealing}
\end{figure*}

\begin{figure*}
\centering
\fbox{\includegraphics[trim={0 13cm 13cm 0},clip, scale=0.17]{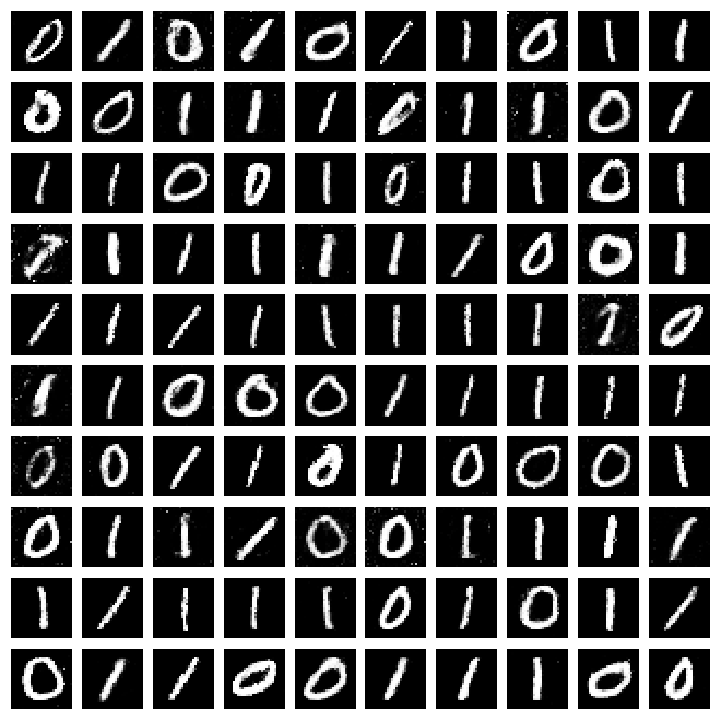}}
\fbox{\includegraphics[trim={0 13cm 13cm 0},clip, scale=0.17]{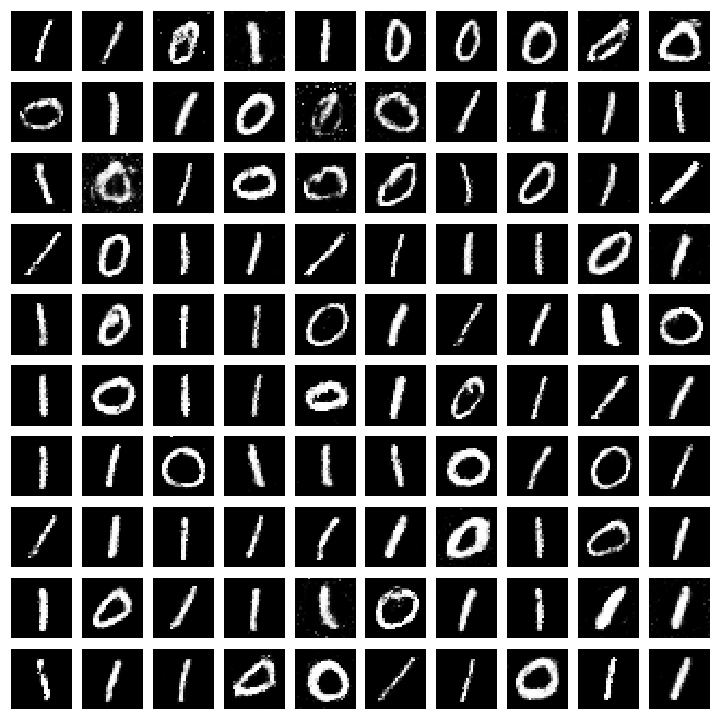}}
\fbox{\includegraphics[trim={0 13cm 13cm 0},clip, scale=0.17]{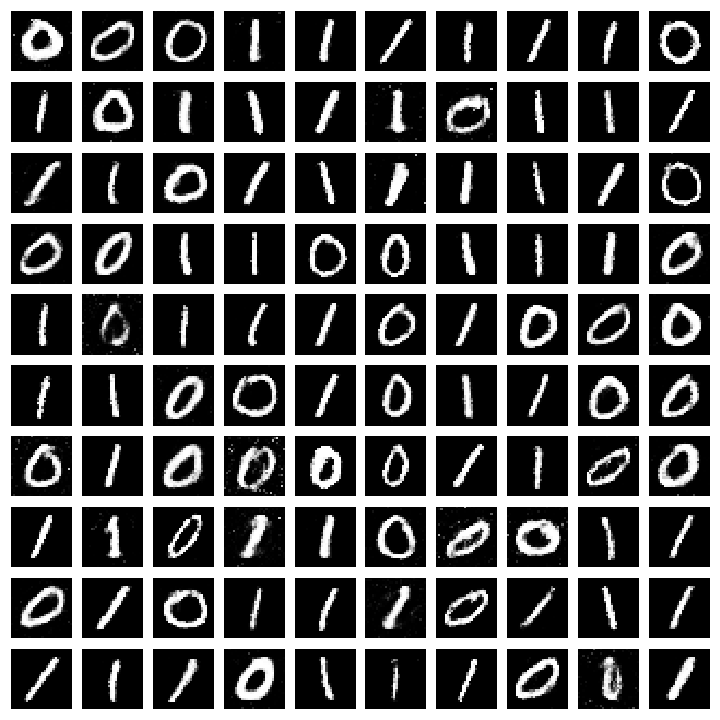}}
\fbox{\includegraphics[trim={0 13cm 13cm 0},clip, scale=0.17]{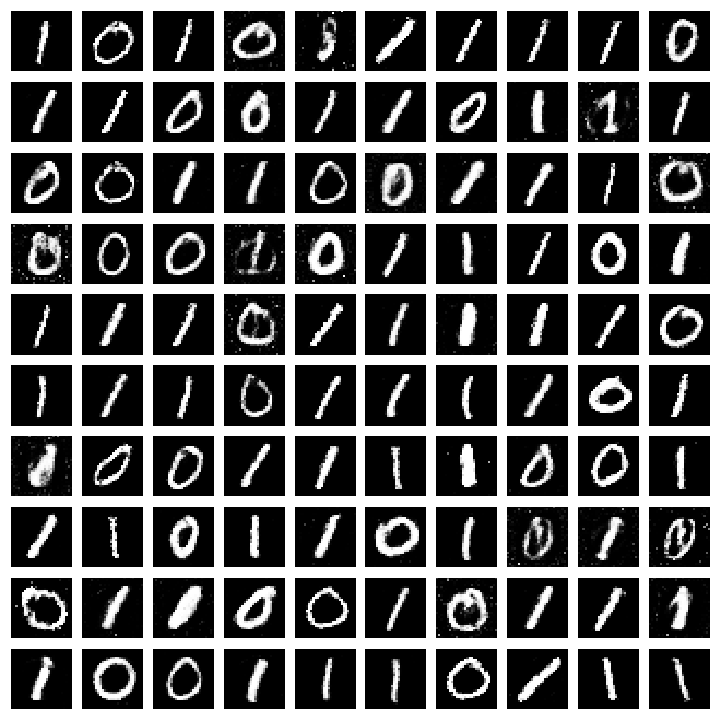}}
\fbox{\includegraphics[trim={0 13cm 13cm 0},clip, scale=0.17]{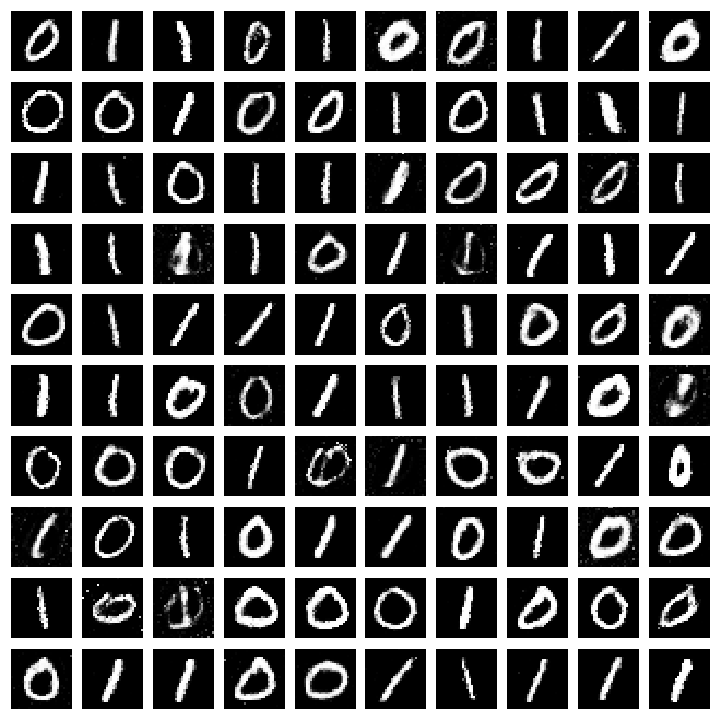}}
\caption{\small Images generated at the 39,000'th iteration of each of the 5 runs of our algorithm for the simulation mentioned in Figure \ref{fig:MNIST_annealing} with a randomized acceptance rule with a temperature schedule of $e^{-\frac{1}{\tau_{i}}} = \frac{1}{4 + e^{(i/20000)^2}}$.} \label{fig:MNIST_annealing_AllRuns}
\end{figure*}

\end{document}